\documentclass{article}
\usepackage[numbers]{natbib}
\pdfoutput=1
\usepackage{amsmath,amssymb,amsthm}
\usepackage{algorithm,algorithmic}
\usepackage[colorlinks=true,linkcolor=blue]{hyperref}
\usepackage[usenames,dvipsnames]{xcolor}
\usepackage{fullpage}
\usepackage{MnSymbol}
\usepackage{graphicx}
\usepackage{amsmath,amsfonts,amssymb,amsthm}
\usepackage{graphicx}
\usepackage{color}

\newtheorem{theorem}{Theorem}
\newtheorem{lemma}[theorem]{Lemma}

\newtheorem{corollary}[theorem]{Corollary}
\newtheorem{assumption}{Assumption}
\newtheorem{proposition}[theorem]{Proposition}

\newtheorem{remark}{Remark}

\theoremstyle{definition}

\begin{document}

\setlength{\parskip}{2mm}
\setlength{\parindent}{0pt}


\newcommand{\mbb}[1]{\mathbb{#1}}
\newcommand{\mbf}[1]{\mathbf{#1}}
\newcommand{\mc}[1]{\mathcal{#1}}
\newcommand{\mrm}[1]{\mathrm{#1}}
\newcommand{\trm}[1]{\textrm{#1}}

\newcommand{\norm}[1]{\left\|#1\right\|}
\newcommand{\sign}{\mrm{sign}}
\newcommand{\argmin}[1]{\underset{#1}{\mrm{argmin}} \ }
\newcommand{\argmax}[1]{\underset{#1}{\mrm{argmax}} \ }
\newcommand{\reals}{\mathbb{R}}
\newcommand{\E}[1]{\mathbb{E}\left[ #1 \right]} 
\newcommand{\Ebr}[1]{\mathbb{E}\left\{ #1 \right\}} 
\newcommand{\En}{\mathbb{E}}  
\newcommand{\Eu}[1]{\underset{#1}{\mathbb{E}}}  
\newcommand{\Ebar}{\Hat{\Hat{\mathbb{E}}}}  
\newcommand{\Esbar}[2]{\Hat{\Hat{\mathbb{E}}}_{#1}\left[ #2 \right]} 
\newcommand{\Es}[2]{\mathbb{E}_{#1}\left[ #2 \right]} 
\newcommand{\Ps}[2]{\mathbb{P}_{#1}\left[ #2 \right]}
\newcommand{\Prob}{\mathbb{P}}
\newcommand{\conv}{\operatorname{conv}}
\newcommand{\inner}[1]{\left\langle #1 \right\rangle}
\newcommand{\ip}[2]{\left<#1,#2\right>}
\newcommand{\lv}{\left\|}
\newcommand{\rv}{\right\|}
\newcommand{\Phifunc}[1]{\Phi\left(#1\right)}
\newcommand{\ind}[1]{{\bf 1}\left\{#1\right\}}
\newcommand{\tr}{\ensuremath{{\scriptscriptstyle\mathsf{T}}}}
\newcommand{\eqdist}{\stackrel{\text{d}}{=}}
\newcommand{\alphT}{\widehat{\alpha}(T)}
\newcommand{\PD}{\mathcal P}
\newcommand{\QD}{\mathcal Q}
\newcommand{\jp}{\ensuremath{\mathbf{p}}}
\newcommand{\rh}{\boldsymbol{\rho}}
\newcommand{\proj}{\text{Proj}}
\newcommand{\Eunderone}[1]{\underset{#1}{\En}}
\newcommand{\Eunder}[2]{\underset{\underset{#1}{#2}}{\En}}
\newcommand{\bphi}{\boldsymbol\phi}
\newcommand\s{\mathbf{s}}
\newcommand\w{\mathbf{w}}
\newcommand\x{\mathbf{x}}
\newcommand\y{\mathbf{y}}
\newcommand\z{\mathbf{z}}
\newcommand\f{\mathbf{f}}

\renewcommand\v{\mathbf{v}}

\newcommand\cB{\mathcal{B}}
\newcommand\cC{\mathcal{C}}
\newcommand\cD{\mathcal{D}}
\newcommand\cL{\mathcal{L}}
\newcommand\cN{\mathcal{N}}
\newcommand\X{\mathcal{X}}
\newcommand\Y{\mathcal{Y}}
\newcommand\Z{\mathcal{Z}}
\newcommand\F{\mathcal{F}}
\newcommand\G{\mathcal{G}}
\newcommand\cH{\mathcal{H}}
\newcommand\N{\mathcal{N}}
\newcommand\M{\mathcal{M}}
\newcommand\W{\mathcal{W}}
\newcommand\Nhat{\mathcal{\widehat{N}}}
\newcommand\Diff{\mathcal{G}}
\newcommand\Compare{\boldsymbol{B}}
\newcommand\RH{\eta} 
\newcommand\metricent{\N_{\mathrm{metric}}} 

\newcommand{\karthik}[1]{{\color{red} Karthik: #1}}
\newcommand{\sasha}[1]{{\color{blue} Sasha: #1}}
\newcommand{\ohad}[1]{{\color{green} Ohad: #1}}

\newcommand\ldim{\mathrm{Ldim}}
\newcommand\fat{\mathrm{fat}}
\newcommand\Img{\mbox{Img}}
\newcommand\sparam{\sigma} 
\newcommand\Psimax{\ensuremath{\Psi_{\mathrm{max}}}}

\newcommand\Rad{\mathfrak{R}}
\newcommand\Val{\mathcal{V}}
\newcommand\Valdet{\mathcal{V}^{\mathrm{det}}}
\newcommand\Dudley{\mathfrak{D}}
\newcommand\Reg{\mbf{Reg}}
\newcommand\D{\mbf{D}}
\renewcommand\P{\mbf{P}}

\newcommand\Xcvx{\X_\mathrm{cvx}}
\newcommand\Xlin{\X_\mathrm{lin}}
\newcommand\loss{\mathrm{loss}}
\newcommand{\Relax}[3]{\mbf{Rel}_{#1}\left(#2 \middle| #3 \right)}
\newcommand{\Rel}[2]{\mbf{Rel}_{#1}\left(#2 \right)}

\newcommand{\multiminimax}[1]{\ensuremath{\left\llangle #1\right\rrangle}}

\def\deq{\triangleq}


\title{Relax and Localize: From Value to Algorithms}

\author{
Alexander Rakhlin \\ University of Pennsylvania
\and 
Ohad Shamir\\
Microsoft Research
\and
Karthik Sridharan \\
University of Pennsylvania
}

\maketitle

\begin{abstract}
We show a principled way of deriving online learning algorithms from a minimax analysis. Various upper bounds on the minimax value, previously thought to be non-constructive, are shown to yield algorithms. This allows us to seamlessly recover known methods and to derive new ones. Our framework also captures such ``unorthodox'' methods as Follow the Perturbed Leader and the $R^2$ forecaster. We emphasize that understanding the inherent complexity of the learning problem leads to the development of algorithms.

We define \emph{local} sequential Rademacher complexities and associated algorithms that allow us to obtain faster rates in online learning, similarly to statistical learning theory. Based on these localized complexities we build a general adaptive method that can take advantage of the suboptimality of the observed sequence. 
	
We present a number of new algorithms, including a family of  randomized methods that use the idea of a ``random playout''. Several new versions of the Follow-the-Perturbed-Leader algorithms are presented, as well as methods based on the Littlestone's dimension, efficient methods for matrix completion with trace norm, and algorithms for the problems of transductive learning and prediction with static experts.
\end{abstract}


\section{Introduction}

This paper studies the online learning framework, where the goal of the player is to incur small regret while  observing a sequence of data on which we place no distributional assumptions. Within this framework, many algorithms have been developed over the past two decades, and we refer to the book of Cesa-Bianchi and Lugosi \citep{PLG} for a comprehensive treatment of the subject. More recently, a non-algorithmic minimax approach has been developed to study the \emph{inherent complexities} of sequential problems  \citep{AbeBarRakTew08colt, AbeAgrBarRak09colt, RakSriTew10nips, sridharan2010convex}. In particular, it was shown that a theory in parallel to Statistical Learning can be developed, with random averages, combinatorial parameters, covering numbers, and other measures of complexity. Just as the classical learning theory is concerned with the study of the supremum of empirical or Rademacher process, online learning is concerned with the study of the supremum of a martingale or a certain dyadic process. Even though complexity tools introduced in \cite{RakSriTew10nips,RakSriTew11nips,RakSriTew11colt} provide ways of studying the minimax value, no algorithms have been exhibited to achieve these non-constructive bounds in general. 

In this paper, we show that algorithms can, in fact, be extracted from the minimax analysis. This observation leads to a unifying view of many of the methods known in the literature, and also gives a general recipe for developing new algorithms. We show that the potential method, which has been studied in various forms, naturally arises from the study of the minimax value as a certain \emph{relaxation}. We further show that the sequential complexity tools introduced in \citep{RakSriTew10nips} are, in fact, relaxations and can be used for constructing algorithms that enjoy the corresponding bounds. By choosing appropriate relaxations, we recover many known methods, improved variants of some known methods, and new algorithms. One can view our framework as one for converting a non-constructive proof of an upper bound on the value of the game into an algorithm. Surprisingly, this allows us to also study such ``unorthodox'' methods as Follow the Perturbed Leader \citep{KalVem05}, and the recent method of \citep{CBShamir11efficient} under the same umbrella with others. We show that the idea of a random playout has a solid theoretical basis, and that Follow the Perturbed Leader algorithm is an example of such a method. It turns out that whenever the sequential Rademacher complexity is of the same order as its i.i.d. cousin, there is a family of randomized methods that avoid certain computational hurdles. Based on these developments, we exhibit an efficient method for the trace norm matrix completion problem, novel Follow the Perturbed Leader algorithms, and efficient methods for the problems of transductive learning and prediction with static experts. 

The framework of this paper gives a recipe for developing algorithms. Throughout the paper, we stress that the notion of a relaxation, introduced below, is not appearing out of thin air but rather as an upper bound on the sequential Rademacher complexity. The understanding of {\em inherent complexity} thus leads to the \emph{development of algorithms}.

One unsatisfying aspect of the minimax developments so far has been the lack of a \emph{localized} analysis. Local Rademacher averages have been shown to play a key role in Statistical Learning for obtaining fast rates. It is also  well-known that fast rates are possible in online learning, on the case-by-case basis, such as for online optimization of strongly convex functions. We show that, in fact, a localized analysis can be performed at an abstract level, and it goes hand-in-hand with the idea of relaxations. Using such localized analysis, we arrive at \emph{local sequential Rademacher} and other local complexities. These complexities upper-bound the value of the online learning game and can lead to fast rates. What is equally important, we provide an associated generic algorithm to achieve the localized bounds. We further develop the ideas of localization, presenting a general adaptive (data-dependent) procedure that takes advantage of the actual moves of the adversary that might have been suboptimal. We illustrate the procedure on a few examples. Our study of localized complexities and adaptive methods follows from a general agenda of developing universal methods that can adapt to the actual sequence of data played by Nature, thus automatically interpolating between benign and minimax optimal sequences.

This paper is organized as follows. In Section~\ref{sec:value} we formulate the value of the online learning problem and present the (possibly computationally inefficient) minimax algorithm. In Section~\ref{sec:rel} we develop the idea of relaxations and the meta algorithm based on relaxations, and present a few examples. Section~\ref{sec:local} is devoted to a new formalism of localized complexities, and we present a basic localized meta algorithm. We show, in particular, that for strongly convex objectives, the regret is easily bounded through localization. Next, in Section~\ref{sec:adaptive}, we present a fully adaptive method that constantly checks whether the sequence being played by the adversary is in fact minimax optimal. We show that, in particular, we recover some of the known adaptive results. We also demonstrate how local data-dependent norms arise as a natural adaptive method. The remaining sections present a number of new algorithms, often with superior computational properties and regret guarantees than what is known in the literature. 

\paragraph{Notation:} A set $\{x_1,\ldots,x_t\}$ is often denoted by $x_{1:t}$. A $t$-fold product of $\X$ is denoted by $\X^t$. Expectation with respect to a random variable $Z$ with distribution $p$ is denoted by $\En_Z$ or $\En_{Z\sim p}$. The set $\{1,\ldots,T\}$ is denoted by $[T]$, and the set of all distributions on some set ${\mathcal A}$ by $\Delta({\mathcal A})$. The inner product between two vectors is written as $\inner{a,b}$ or as $a^\tr b$. The set of all functions from $\X$ to $\Y$ is denoted by $\Y^\X$. Unless specified otherwise, $\epsilon$ denotes a vector $(\epsilon_1,\ldots,\epsilon_T)$ of i.i.d. Rademacher random variables. An $\X$-valued tree $\x$ of depth $d$ is defined as a sequence $(\x_1,\ldots,\x_d)$ of mappings $\x_t:\{\pm 1\}^{t-1}\mapsto \X$ (see \cite{RakSriTew10nips}). We often write $\x_t(\epsilon)$ instead of $\x_t(\epsilon_{1:t-1})$.

\section{Value and The Minimax Algorithm}
\label{sec:value}

Let $\F$ be the set of learner's moves and $\X$ the set of moves of Nature. The online protocol dictates that on every round $t=1,\ldots,T$ the learner and Nature simultaneously choose $f_t\in\F$, $x_t\in\X$, and observe each other's actions. The learner aims to minimize regret
$$\Reg_T \deq \sum_{t=1}^T \ell(f_t,x_t) - \inf_{f \in \F} \sum_{t=1}^T \ell(f,x_t)$$
where $\ell:\F\times\X\to \reals$ is a known loss function. Our aim is to study this online learning problem at an abstract level without assuming convexity or other properties of the loss function and the sets $\F$ and $\X$. We do assume, however, that $\ell$, $\F$, and $\X$ are such that the minimax theorem in the space of distributions over $\F$ and $\X$ holds. By studying the abstract setting, we are able to develop general algorithmic and non-algorithmic ideas that are common across various application areas. 

The starting point of our development is the minimax value of the associated online learning game:
\begin{align}
	\label{eq:val}
	\Val_T(\F) = \inf_{q_1 \in \Delta(\F)} \sup_{x_1 \in \X} \Eunderone{f_1 \sim q_1} \ldots \inf_{q_T \in \Delta(\F)} \sup_{x_T \in \X} \Eunderone{f_T \sim q_T}\left[\sum_{t=1}^T \ell(f_t,x_t) - \inf_{f \in \F} \sum_{t=1}^T \ell(f,x_t) \right]
\end{align}
where $\Delta(\F)$ is the set of distributions on $\F$. The minimax formulation immediately gives rise to the optimal algorithm that solves the minimax expression at every round $t$. That is, after witnessing $x_1,\ldots,x_{t-1}$ and $f_1,\ldots,f_{t-1}$, the algorithm returns
\begin{align}
	\label{eq:minimaxopt-algorithm}
	&\argmin{q \in \Delta(\F)} \left\{ \sup_{x_t} \Eunderone{f_t \sim q} \inf_{q_{t+1}} \sup_{x_{t+1}} \Eunderone{f_{t+1}}\ldots \inf_{q_T} \sup_{x_T} \Eunderone{f_T}\left[\sum_{i=t}^T \ell(f_i,x_i) - \inf_{f \in \F} \sum_{i=1}^T \ell(f,x_i) \right] \right\} \\
	&=\argmin{q \in \Delta(\F)} \left\{ \sup_{x_t} \Eunderone{f_t \sim q} \left[ \ell(f_t,x_t) + \inf_{q_{t+1}} \sup_{x_{t+1}} \Eunderone{f_{t+1}}\ldots \inf_{q_T} \sup_{x_T} \Eunderone{f_T}\left[\sum_{i=t+1}^T \ell(f_i,x_i) - \inf_{f \in \F} \sum_{i=1}^T \ell(f,x_i) \right] \right]\right\} \notag
\end{align}
Henceforth, if the quantification in $\inf$ and $\sup$ is omitted, it will be understood that $x_t$, $f_t$, $p_t$, $q_t$ range over $\X$, $\F$, $\Delta(\X)$, $\Delta(\F)$, respectively. Moreover, $\En_{x_t}$ is with respect to $p_t$ while $\En_{f_t}$ is with respect to $q_t$. The first sum in \eqref{eq:minimaxopt-algorithm} starts at $i=t$ since the partial loss $\sum_{i=1}^{t-1} \ell(f_i,x_i)$ has been fixed. We now notice a recursive form for defining the value of the game. Define for any $t \in [T-1]$ and any given prefix $x_{1},\ldots,x_t \in \X$ the {\em conditional value} 
\begin{align*}
\Val_T\left(\F \middle| x_1,\ldots,x_t \right) \deq \inf_{q \in \Delta(\F)} \sup_{x \in \X} \left\{ \underset{f \sim q}{\En}\left[\ell(f,x) \right] + \Val_T(\F| x_{1},\ldots,x_{t},x)\right\}
\end{align*}
where 
$$\Val_T\left(\F \middle| x_1,\ldots,x_T \right) \deq - \inf_{f \in \F} \sum_{t=1}^T \ell(f,x_t) ~~~~\mbox{and}~~~~ \Val_T(\F) = \Val_T(\F| \{\}).$$
The minimax optimal algorithm specifying the mixed strategy of the player can be written succinctly 
\begin{align}
	\label{eq:minimax_algo}
	q_t = \argmin{q \in \Delta(\F)} \sup_{x \in \X} \Big\{ \Es{f \sim q}{\ell(f,x)}  + \Val_T(\F| x_1,\ldots,x_{t-1},x)\Big\} \ .
\end{align}
This recursive formulation has appeared in the literature, but now we have tools to study the conditional value of the game. We will show that various upper bounds on $\Val_T(\F| x_1,\ldots,x_{t-1},x)$ yield an array of algorithms, some with better computational properties than others. In this way, the non-constructive approach of \cite{RakSriTew10nips,RakSriTew11colt,RakSriTew11nips} to upper bound the value of the game directly translates into algorithms. 

The minimax algorithm in \eqref{eq:minimax_algo} can be interpreted as choosing the best decision that takes into account the present loss and the worst-case future. We then realize that the conditional value of the game serves as a ``regularizer'', and thus well-known online learning algorithms such as Exponential Weights, Mirror Descent and Follow-the-Regularized-Leader arise as relaxations rather than a ``method that just works''. 

The first step is to appeal to the minimax theorem and perform the same manipulation as in \citep{AbeAgrBarRak09colt,RakSriTew10nips}, but only on the value from $t+1$ onwards:
\begin{align*}
\Val_T\left(\F \middle| x_1,\ldots,x_t \right) =  \sup_{p_{t+1}} \Eunderone{x_{t+1}} \ldots \sup_{p_T} \Eunderone{x_T}\left[\sum_{i=t+1}^T \inf_{f_{i} \in \F} \Eunderone{x_i \sim p_i} \ell(f_i,x_i) - \inf_{f \in \F} \sum_{i=1}^T \ell(f,x_i) \right]  
\end{align*}
This expression is still unwieldy, and the idea is now to come up with more manageable, yet tight, upper bounds of the conditional value.

\section{Relaxations and the Basic Meta-Algorithm}\label{sec:rel}

A \emph{relaxation} $\Rel{}{}$ is a sequence of functions $\Relax{T}{\F}{x_{1},\ldots,x_{t}}$ for each $t\in[T]$. We shall use the notation $\Rel{T}{\F}$ for $\Relax{T}{\F}{\{\}}$. A relaxation will be called {\em admissible} if for any $x_1,\ldots,x_T \in \X$, 
\begin{align}
	\label{eq:relax_admissibility}
	\Relax{T}{\F}{x_{1},\ldots,x_{t}} \ge \inf_{q \in \Delta(\F)} \sup_{x \in \X} \left\{ \underset{f \sim q}{\En}\left[\ell(f,x)\right] + \Relax{T}{\F}{x_{1},\ldots,x_{t},x}\right\}
\end{align}
for all $t\in[T-1]$, and 
$$ \Relax{T}{\F}{x_{1},\ldots,x_{T}} \ge - \inf_{f \in \F} \sum_{t=1}^T \ell(f,x_t) .$$

A strategy $q$ that minimizes the expression in \eqref{eq:relax_admissibility} defines an optimal algorithm for the relaxation $\Rel{}{}$. This algorithm is given below under the name ``Meta-Algorithm''. However, minimization need not be exact: any $q$ that satisfies the admissibility condition \eqref{eq:relax_admissibility} is a valid method, and we will say that such an algorithm is {\em admissible with respect to the relaxation $\Rel{}{}$}.

\begin{algorithm}[h]
\caption{Meta-Algorithm $\mbf{MetAlgo}$}
\begin{algorithmic}
\STATE Parameters: Admissible relaxation $\mbf{Rel}$ 
\FOR{$t = 1$ to $T$}
    \STATE $q_t = \arg\min_{q \in \Delta(\F)} \sup_{x \in \X} \left\{ \Es{f \sim q}{\ell(f,x)}  + \Relax{T}{\F}{ x_1,\ldots,x_{t-1},x}\right\}$
    \STATE Play $f_t \sim q_t$ and receive $x_t$ from adversary
\ENDFOR
\end{algorithmic}
\label{alg:met}
\end{algorithm}

\begin{proposition}\label{prop:main}
Let $\Rel{}{}$ be an admissible relaxation. For any admissible algorithm with respect to $\Rel{}{}$, including the \emph{Meta-Algorithm}, irrespective of the strategy of the adversary,
\begin{align}
	\label{eq:sum_cond_exp_bdd_by_relax}
	\sum_{t=1}^T \En_{f_t\sim q_t}\ell(f_t,x_t) - \inf_{f\in\F} \sum_{t=1}^T \ell(f,x_t) \leq \Rel{T}{\F} \ , 
\end{align}
and therefore,
\begin{align*}
	\En[\Reg_T]\leq \Rel{T}{\F} \ .
\end{align*}
We also have that 
\begin{align*}
\Val_T(\F) \le \Rel{T}{\F} \ .
\end{align*}
If $a\leq \ell(f,x) \leq b$ for all $f\in\F,x\in\X$, the Hoeffding-Azuma inequality yields, with probability at least $1-\delta$,
\begin{align*}
	\Reg_T = \sum_{t=1}^T \ell(f_t,x_t) - \inf_{f\in\F} \sum_{t=1}^T \ell(f,x_t) \leq \Rel{T}{\F} + (b-a)\sqrt{T/2\cdot \log (2/\delta)} \ .
\end{align*}
Further, if for all $t\in[T]$, the admissible strategies $q_t$ are deterministic, 
\begin{align*}
	\Reg_T \leq \Rel{T}{\F} \ .
\end{align*}
\end{proposition}

The reader might recognize $\mbf{Rel}$ as a potential function. It is known that one can derive regret bounds by coming up with a potential such that the current loss of the player is related to the difference in the potentials at successive steps, and that the loss of the best decision in hindsight can be extracted from the final potential. The origin of ``good'' potential functions has always been a mystery (at least to the authors). One of the conceptual contributions of this paper is to show that they naturally arise as relaxations on the conditional value. The conditional value itself can be characterized as the tightest possible relaxation.

In particular, for many problems a tight relaxation (sometimes within a factor of $2$) is achieved through symmetrization. Define the \emph{conditional Sequential Rademacher complexity} 
	\begin{align}
		\label{eq:def_cond_seq_rad}
		\Rad_T (\F | x_1,\ldots,x_t) = \sup_{\x} \En_{\epsilon_{t+1:T}} \sup_{f\in\F} \left[ 2\sum_{s=t+1}^T \epsilon_s\ell(f,\x_{s-t}(\epsilon_{t+1:s-1})) - \sum_{s=1}^t \ell(f,x_s) \right] \ .
	\end{align}
Here the supremum is over all $\X$-valued binary trees of depth $T-t$. One may view this complexity as a partially symmetrized version of the sequential Rademacher complexity
\begin{align}
	\Rad_T (\F ) \deq \Rad_T (\F ~|~ \{\}) = \sup_{\x} \En_{\epsilon_{1:T}} \sup_{f\in\F} \left[ 2\sum_{s=1}^T \epsilon_s\ell(f,\x_{s}(\epsilon_{1:s-1})) \right]
\end{align}
defined in \cite{RakSriTew10nips}. We shall refer to the term involving the tree $\x$ as the ``future'' and the term being subtracted off -- as the ``past''. This indeed corresponds to the fact that the quantity is conditioned on the already observed $x_1,\ldots,x_t$, while for the future we have the worst possible binary tree.\footnote{It is somewhat cumbersome to write out the indices on $\x_{s-t}(\epsilon_{t+1:s-1})$ in \eqref{eq:def_cond_seq_rad}, so we will instead use $\x_s(\epsilon)$ for $s=1,\ldots,T-t$, whenever this does not cause confusion.}

\begin{proposition}
	\label{prop:rad_admissible}
	The conditional Sequential Rademacher complexity is admissible.
\end{proposition}
The proof of this proposition is given in the Appendix and it corresponds to one step of the sequential symmetrization proof in \cite{RakSriTew10nips}. We note that the factor $2$ appearing in \eqref{eq:def_cond_seq_rad} is not necessary in certain cases (e.g. binary prediction with absolute loss). 

We now show that several well-known methods arise as further relaxations on the conditional sequential Rademacher complexity $\Rad_T$. 

\paragraph{Exponential Weights}

Suppose $\F$ is a finite class and $|\ell(f,x)|\leq 1$. In this case, a (tight) upper bound on sequential Rademacher complexity leads to the following relaxation:
\begin{align}
	\label{rel:expweights}
	\Relax{T}{\F}{x_{1},\ldots, x_t} = \inf_{\lambda>0}\left\{ \frac{1}{\lambda}\log\left( \sum_{f \in \F}  \exp\left(   - \lambda \sum_{i=1}^t \ell(f,x_i)   \right) \right)  +  2 \lambda (T-t) \right\}
\end{align}

\begin{proposition}
	\label{prop:exp_weights_relax}
	The relaxation \eqref{rel:expweights} is admissible and
	$$\Rad_T (\F | x_1,\ldots,x_t)\leq \Relax{T}{\F}{x_{1},\ldots, x_t}.$$
	Furthermore, it leads to a parameter-free version of the Exponential Weights algorithm, defined on round $t+1$ by the mixed strategy
	$$q_{t+1}(f) \propto \exp\left(-\lambda_t^* \sum_{s=1}^{t}\ell(f,x_s)\right)$$  
	with $\lambda_t^*$ the optimal value in \eqref{rel:expweights}.
	The algorithm's regret is bounded by $$\Rel{T}{\F}\leq 2\sqrt{2T\log |\F|} \ . $$
\end{proposition}
The Chernoff-Cram\`er inequality tells us that \eqref{rel:expweights} is the tightest possible relaxation. The proof of Proposition~\ref{prop:exp_weights_relax} reveals that the only inequality is the softmax which is also present in the proof of the maximal inequality for a finite collection of random variables. In this way, exponential weights is an algorithmic realization of a maximal inequality for a finite collection of random variables. The connection between probabilistic (or concentration) inequalities and algorithms runs much deeper.

We point out that the exponential-weights algorithm arising from the relaxation \eqref{rel:expweights} is a {\em parameter-free} algorithm. The learning rate $\lambda^*$ can be optimized (via one-dimensional line search) at each iteration with almost no cost. This can lead to improved performance as compared to the classical methods that set a particular schedule for the learning rate.

\paragraph{Mirror Descent}

In the setting of online linear optimization, the loss is $\ell(f,x)=\ip{f}{x}$. Suppose $\F$ is a unit ball in some Banach space and $\X$ is the dual. Let $\|\cdot\|$ be some $(2,C)$-smooth norm on $\X$ (in the Euclidean case, $C=2$). Using the notation $\tilde{x}_{t-1}=\sum_{s=1}^{t-1}x_s$, a straightforward upper bound on sequential Rademacher complexity is the following relaxation:
\begin{align}
	\label{rel:mirror}
	\Relax{T}{\F}{x_{1},\ldots, x_t}= \sqrt{ \norm{ \tilde{x}_{t-1}}^2 + \ip{\nabla \norm{\tilde{x}_{t-1}}^2}{x_t} + C (T - t + 1) }
\end{align}
\begin{proposition}
	\label{prop:mirror_relax}
	The relaxation \eqref{rel:mirror} is admissible and
	$$\Rad_T (\F | x_1,\ldots,x_t)\leq \Relax{T}{\F}{x_{1},\ldots, x_t} \ .$$
	Furthermore, it leads to the Mirror Descent algorithm with regret at most $\Rel{T}{\F}\leq \sqrt{2C T}$. 
\end{proposition}

An important feature of the algorithms we just proposed is the absence of any parameters, as the step size is tuned automatically. We had chosen Exponential Weights and Mirror Descent for illustration because these methods are well-known. Our aim at this point was to show that the associated relaxations arise naturally (typically with a few steps of algebra) from the sequential Rademacher complexity. More examples are included later in the paper. It should now be clear that upper bounds, such as the Dudley Entropy integral, can be turned into a relaxation, provided that admissibility is proved. Our ideas have semblance of those in Statistics, where an information-theoretic complexity can be used for defining penalization methods.

\section{Localized Complexities and the Localized-Meta Algorithm}
\label{sec:local}

The localized analysis plays an important role in Statistical Learning Theory. The basic idea is that better rates can be proved for empirical risk minimization when one considers the empirical process in the vicinity of the target hypothesis \cite{koltchinskii2002empirical,bartlett2005local}. Through this, localization gives \emph{extra information} by shrinking the size of the set which needs to be analyzed. What does it mean to localize in online learning? As we obtain more data, we can rule out parts of $\F$ as  those that are unlikely to become the leaders. This observation indeed gives rise to faster rates. Let us develop a general framework of localization and then illustrate it on examples. We emphasize that the localization ideas will be developed at an abstract level where no assumptions are placed on the loss function or the sets $\F$ and $\X$.

Given any $x_1,\ldots,x_{t} \in \X$, for any $k \ge 1$ define 
$$
\F^k(x_1,\ldots,x_{t}) = \left\{f \in \F : \exists~ x_{t+1},\ldots,x_{t+k} \in \X ~\textrm{ s.t. }~ \sum_{i=1}^{t+k} \ell(f,x_i) = \inf_{f \in \F} \sum_{i=1}^{t+k} \ell(f,x_i) \right\} \ .
$$
That is, given the instances $x_1,\ldots,x_t$, the set $\F^k(x_1,\ldots,x_{t})$ is the set of elements that could be the minimizers of cumulative loss on $t+k$ instances, the first $t$ of which are $x_1,\ldots,x_t$ and the remaining $k$ arbitrary. We shall refer to minimizers of cumulative loss as {\em empirical risk minimizers} (or, ERM). 

Importantly,
$$
\Val_T(\F|x_1,\ldots,x_t) = \Val_T\left(\F^{T-t}(x_1,\ldots,x_t)|x_1,\ldots,x_t\right) \ .
$$
Henceforth, we shall use the notation $\tilde{k}_j \deq \sum_{i=1}^j k_i$. We now consider subdividing $T$ into blocks of time $k_1,\ldots,k_m \in [T]$ such that $\tilde{k}_m = T$. With this notation, $\tilde{k}_i$ is the last time in the $i$th block. We then have regret upper bounded as
\begin{align}
	\label{eq:split-regret}
 \sum_{t=1}^T \ell(f_t,x_t) - \inf_{f \in \F} \sum_{t=1}^T \ell(f,x_t) \le  \sum_{t=1}^T \ell(f_t,x_t) - \sum_{i=1}^m \inf_{f \in \F^{k_i}\left(x_1,\ldots,x_{\tilde{k}_{i-1}}\right)} \sum_{t=\tilde{k}_{i-1}+1}^{\tilde{k}_{i}} \ell(f,x_t) \ . 
\end{align}
The short inductive proof is given in Appendix, Lemma~\ref{lem:reg_split}. We can now bound \eqref{eq:split-regret} by
\begin{align*}
&\sum_{i=1}^m \left( \sum_{t=\tilde{k}_{i-1}+1}^{\tilde{k}_{i}} \ell(f,x_t) -  \inf_{f \in \F^{k_i}\left(x_1,\ldots,x_{\tilde{k}_{i-1}}\right)} ~\sum_{t=\tilde{k}_{i-1}+1}^{\tilde{k}_{i}} \ell(f,x_t) \right)\\
& ~~~~~~~\le  \sum_{i=1}^m \Reg_{k_i}(x_{\tilde{k}_{i-1}},\ldots,x_{\tilde{k}_{i}},f_{\tilde{k}_{i-1}},\ldots,f_{\tilde{k}_{i}},\F^{k_i}(x_1,\ldots,x_{\tilde{k}_{i-1}}))
\end{align*}

Hence, one can decompose the online learning game into blocks of $m$ successive games. The crucial point to notice is that at the $i^{th}$ block, we do not compete with the best hypothesis in all of $\F$ but rather only $\F^{k_i}(x_1,\ldots,x_{\tilde{k}_{i-1}})$. It is this localization based on history that could lead to possibly faster rates. While the ``blocking'' idea often appears in the literature (for instance, in the form of a doubling trick, as described below), the process is usually ``restarted'' from scratch by considering all of $\F$. Notice further that one need not choose all $k_1,\ldots,k_m$ in advance. The player can choose $k_i$ based on history $x_1,\ldots,x_{\tilde{k}_{i-1}}$ and then use, for instance, the Meta-Algorithm introduced in previous section to play the game within the block $k_i$ using the localized class $\F^{k_i}(x_1,\ldots,x_{\tilde{k}_{i-1}})$. Such adaptive procedures will be considered in Section~\ref{sec:adaptive}, but presently we assume that the block sizes $k_1,\ldots,k_m$ are fixed.

While the successive localizations using subsets $\F^{k_i}(x_1,\ldots,x_{\tilde{k}_{i-1}})$ can provide an  algorithm with possibly better performance, specifying and analyzing the localized subset $\F^{k_i}(x_1,\ldots,x_{\tilde{k}_{i-1}})$ exactly might not be possible. In such a case, one can instead use
$$
\F_r(x_1,\ldots,x_{\tilde{k}_{i-1}}) = \left\{ f \in \F : P\left(f ~|~ x_1,\ldots,x_{\tilde{k}_{i-1}}\right) \le r \right\}
$$
where $P$ is some ``property'' of $f$ given data. This definition echoes the definition of the set of $r$-minimizers of empirical or expected risk in Statistical Learning. Further, for a given $k$ define 
$$
r(k;x_1,\ldots,x_t) =  \inf\{r :   \F^{k}(x_1,\ldots,x_t) \subset \F_r(x_1,\ldots,x_{t})\}
$$
the smallest ``radius'' such that $\F_r$ includes the set of potential minimizers over the next $k$ time steps. Of course, if the property $P$ does not enforce localization, the bounds are not going to exhibit any improvement, so $P$ needs to be chosen carefully for a particular problem of interest. 

We have the following algorithm:
\begin{algorithm}[h]
\caption{Localized Meta-Algorithm}
\begin{algorithmic}
\STATE Parameters :  Relaxation $\mbf{Rel}$ 
\STATE Initialize $t=0$ and blocks $k_1,\ldots,k_m$ s.t. $\sum_{i=1}^m k_i = T$
\FOR{$i=1$ to $m$}
    \STATE Play $k_i$ rounds  using $\mbf{MetAlgo}\left(\F_{r(k_i;x_1,\ldots,x_t)}\right)$ and set $t = t + k_i$
\ENDFOR
\end{algorithmic}
\label{alg:locmet}
\end{algorithm}

\begin{lemma}\label{lem:locseq}
The regret of the Localized Meta-Algorithm is bounded as
\begin{align*}
\Reg_T(x_1,\ldots,x_T) \le \sum_{i=1}^{m} \Rel{k_i}{\F_{r\left(k_i;x_1,\ldots,x_{\tilde{k}_{i-1}}\right)}}
\end{align*}
\end{lemma}

Note that the above lemma points to local sequential complexities for online learning problems that can lead to possibly fast rates. In particular, if sequential Rademacher complexity is used as the relaxation in the Localized Meta-Algorithm, we get a bound in terms of \emph{local sequential Rademacher complexities}. 

\subsection{Local Sequential Complexities}
The following corollary is a direct consequence of Lemma \ref{lem:locseq}. 

\begin{corollary}[Local Sequential Rademacher Complexity]
For any property $P$ and any $k_1,\ldots,k_m \in \mathbb{N}$ such that $\sum_{i=1}^m k_i = T$, we have that :
$$
\Val_T(\F) \le \sup_{x_1,\ldots,x_T} \sum_{i=1}^m \Rad_{k_i}\left(\F_{r\left(k_i;x_1,\ldots,x_{\tilde{k}_{i-1}}\right)}\right)
$$
\end{corollary}
Clearly, the sequential Rademacher complexities in the above bound can be replaced with other sequential complexity measures of the localized classes that are upper bounds on the sequential Rademacher complexities. For instance, one can replace each Rademacher complexity $\Rad_{k_i}$ by covering number based bounds of the local classes, such as the analogues of the Dudley Entropy Integral bounds developed in the sequential setting in \cite{RakSriTew10nips}. Once can also use, for instance, fat-shattering dimension based complexity measures for these local classes.

\subsection{Examples}
\subsubsection{Example : Doubling trick} \label{subsec:double}
The doubling trick can be seen as a particular blocking strategy with $k_i = 2^{i-1}$ so that
\begin{align*}
\Reg_T(x_1,\ldots,x_T) & \le \sum_{i=1}^{\lceil \log_2 T \rceil + 1} \Rel{2^{i-1}}{\F_{r(2^{i-1};x_1,\ldots,x_{\sum_{j=1}^{i-1} 2^{j-1}})}}
& \le \sum_{i=1}^{\lceil \log_2 T\rceil + 1} \Rel{2^{i-1}}{\F}
\end{align*}
for $\F_r$ defined with respect to some property $P$. The latter inequality is potentially loose, as the algorithm is ``restarted'' after the previous block is completed. Now if $\mbf{Rel}$ is such that for any $t$, $\Rel{t}{\F} \le t^p$ for some $p$ then the regret is upper bounded by $\frac{T^p - 2^{-p}}{1 - 2^{-p} }$.
The main advantage of the doubling trick is of course that we do not need to know $T$ in advance.

\subsubsection{Example : Strongly Convex Loss}
To illustrate the idea of localization, consider online convex optimization with $\lambda$-strongly convex functions $x_t:\F\mapsto\reals$ (that is, $\ell(f,x) = x(f)$). Define
\begin{align*}
\Relax{T}{\F}{ x_{1},\ldots,x_{t}} =  - \inf_{f \in \F} \sum_{i=1}^t x_i(f) + (T-t)  \inf_{f \in\F} \sup_{f' \in \F} \|f - f'\|
\end{align*}
An easy Lemma~\ref{lem:strconv} in the Appendix shows that this relaxation is admissible. Notice that this relaxation grows linearly with block size and is by itself quite bad. However, with blocking and localization, the relaxation gives an optimal bound for strongly convex objectives. To see this note that for $k = 1$, any minimizer of $\sum_{i=1}^{t+1} x_i(f)$ has to be close to the minimizer $\hat{f}_t$ of $\sum_{i=1}^{t} x_i(f)$, due to strong convexity of the functions. In other words, the property $$P(f|x_1,\ldots,x_t) = \|f-\hat{f}_t\|$$ 
with $r=1/(\lambda t)$ entails
$$
\F^1(x_1,\ldots,x_t) \subseteq  \left\{f \in \F : \|f - \hat{f}_t\| \le 1/(\lambda t)\right\} = \F_r(x_1,\ldots,x_t) . 
$$

The relaxation for the block of size $k=1$ is
$$
\Rel{1}{\F_r(x_1,\ldots,x_t)} \le \inf_{f \in\F_r(x_1,\ldots,x_t)} \sup_{f' \in \F_r(x_1,\ldots,x_t)} \|f - f'\|,
$$
the radius of the smallest ball containing the localized set $\F_r(x_1,\ldots,x_t)$, and we immediately get
$$
\Reg_T(x_1,\ldots,x_T) \le   \sum_{t=1}^T 1/(\lambda t) \le (1 + \log(T))/\lambda \ .
$$
We remark that this proof is different in spirit from the usual proofs of fast rates for strongly convex functions, and it demonstrates the power of localization.

\section{Adaptive Procedures}
\label{sec:adaptive}

There is a strong interest in developing methods that enjoy worst-case regret guarantees but also take advantage of the suboptimality of the sequence being played by Nature. An algorithm that is able to do so without knowing in advance that the sequence will have a certain property will be called \emph{adaptive}. Imagine, for instance, running an experts algorithm, and one of the experts has gained such a lead that she is clearly the winner (that is, the empirical risk minimizer) at the end of the game. In this case, since we are to be compared with the leader at the end, we need not focus on anyone else, and regret for the remainder of the game is zero. 

There has been previous work on exploiting particular ways in which sequences can be suboptimal. Examples include the Adaptive Gradient Descent of \cite{BarHazRak07} and Adaptive Hedge of \cite{adaptivehedge11}. We now give a generic method which incorporates the idea of localization in order to adaptively (and constantly) check whether the sequence being played is of optimal or suboptimal nature. Notice that, as before, we present the algorithm at the abstract level of the online game with some decision sets $\F$, $\X$, and some loss function $\ell$. 

The adaptive procedure below uses a subroutine $\mbf{Block}(\{x_1,\ldots,x_t\}, \tau)$ which, given the history $\{x_1,\ldots,x_t\}$, returns a subdivision of the next $\tau$ rounds into sub-blocks. The choice of the blocking strategy has to be made for the particular problem at hand, but, as we show in examples, one can often use very simple strategies.

Let us describe the adaptive procedure. First, for simplicity of exposition, we start with the doubling-size blocks. Here is what happens within each of these blocks. During each round the learner decides whether to stay in the same sub-block or to start a new one, as given by the blocking procedure $\mbf{Block}$. If started, the new sub-block uses the localized subset given history of adversary's moves up until last round. Choosing to start a new sub-block corresponds to the realization of the learner that the sequence being presented so far is in fact suboptimal. The learner then incorporates this suboptimality into the localized procedure.

\newcommand{\nbl}{{\tt nbl}}
\newcommand{\curr}{{\tt curr}}

\begin{algorithm}[h]
\caption{Adaptive Localized Meta-Algorithm}
\begin{algorithmic}
\STATE Parameters :  Relaxation $\mbf{Rel}$ and block size calculator $\mbf{Block}$.
\STATE Initialize $t=1$ and $\nbl = 1$, and suppose $T=2^{c}-1$ for some $c\geq 2$.
\FOR{$i = 1$ to $c$} 
\STATE $G = \Rel{2^i}{\F_r(2^i;x_1,\ldots,x_{t-1})}$ \hspace{1.5in} {\scriptsize\tt\%  guaranteed value of  relaxation }
\STATE $m = 1, \curr =1$ and $K_1 = 2^i$
\WHILE{$\curr \le 2^i$ and $t \le T$ }
 \STATE $(\kappa_1,\ldots,\kappa_{m'}) = \mbf{Block}\left(\{x_1,\ldots,x_t\},2^i-\curr\right)$ \hspace{0.6in} {\scriptsize\tt\% blocking for remainder of $2^i$ }
 	\IF{$G > \sup_{x_{t+1},\ldots,x_{2^{i+1}-1}}\sum_{j=1}^{m'} \Rel{\kappa_j}{\F_{r(\kappa_i;x_1,\ldots,x_{t+\tilde{\kappa}_{j-1}})}}$ } 
 		\STATE $k^*_\nbl = \kappa_1$, $K = (\kappa_2,\ldots,\kappa_{m'})$, $m = m'-1$ \hspace{0.3in} {\scriptsize\tt\%  if better value, accept new blocking } 
	\ELSE
		\STATE $k^*_\nbl = K_1$, $K = (K_2,\ldots,K_m)$, $m = m-1$ \hspace{0.3in} {\scriptsize\tt\%  else continue with current blocking }
	\ENDIF
    \STATE Play $k^*_\nbl$ rounds  using $\mbf{MetAlgo}(\F_{r(k^*_\nbl;x_1,\ldots,x_t)})$ \\[2mm]
    \STATE $\nbl = \nbl+1$, $t = t+k^*_\nbl$, $\curr = \curr+k^*_\nbl$ 
	\STATE Let $$G = \sup_{x_{t+1},\ldots,x_{2^{i+1}-1}}\sum_{j=1}^{m} \Rel{K_j}{\F_{r(K_j;x_1,\ldots,x_{t+\sum_{i=1}^{j-1} K_i})}}$$		
    \ENDWHILE
\ENDFOR
\end{algorithmic}
\label{alg:adaplocmet}
\end{algorithm}

\begin{lemma}\label{lem:adaptive}
Given some admissible relaxation $\mbf{Rel}$, the regret of the adaptive localized meta-algorithm (Algorithm~\ref{alg:adaplocmet}) is bounded as
\begin{align*}
\Reg_T \le  \sum_{i=1}^{\nbl} \Rel{k^*_i}{\F_{r\left(k^*_i;x_1,\ldots,x_{\tilde{k}^*_{i-1}}\right)}} 
\end{align*}
where $\nbl$ is the number of blocks actually played and $k^*_i$'s are adaptive block lengths defined within the algorithm. Further, irrespective of the blocking strategy $\mbf{Block} $ used, if the relaxation $\mbf{Rel}$ is such that for any $t$, $\Rel{t}{\F} \le t^p$ for some $p \in (0,1]$, then the worst case regret is always bounded as
$$
\Reg_T \le  (T^p - 2^{-p})/(1 - 2^{-p}) \ .
$$
\end{lemma}

We now demonstrate that the adaptive algorithm in fact takes advantage of sub-optimality in several situations that have been previously studied in the literature. On the conceptual level, adaptive localization allows us to view several fast rate results under the same umbrella.

\paragraph{Example: Adaptive Gradient Descent}

Consider the online convex optimization scenario. Following the setup of \cite{BarHazRak07}, suppose the learner encounters a sequence of convex functions $x_t$ with the strong convexity parameter $\sigma_t$, potentially zero, with respect to a $(2,C)$-smooth norm $\|\cdot\|$. The goal is to adapt to the actual sequence of functions presented by the adversary. Let us invoke the Adaptive Localized Meta-Algorithm with a rather simple blocking strategy 
$$
\mbf{Block}\left(\{x_1,\ldots,x_t\},k\right) = \left\{\begin{array}{cc}
(k) & \textrm{if }\sqrt{k} >  \sigma_{1:t}\\
(1,1,\ldots,1) & \textrm{otherwise } 
\end{array}\right.
$$ 
This blocking strategy either says ``use all of the next $k$ rounds as one block'', or ``make each of the next $k$ time step into separate blocks''. Let $\hat{f}_t$ be the empirical minimizer at the start of the block (that is after $t$ rounds), and let $y_t = \nabla x_t(f_t)$. Then we can use the localization 
\begin{align*}
\F_{r(k;x_1,\ldots,x_t)} &= \left\{f \in \F : \|f - \hat{f}_t\| \le 2 \min\left\{1, k/\sigma_{1:t}\right\}\right\} 
\end{align*}
and relaxation
\begin{align*}
\Relax{k}{\F_{r(k;x_1,\ldots,x_t)}}{y_1,\ldots,y_i} &=  - \ip{\hat{f}_{t}}{\tilde{y}_i} + 2 \min\left\{1, k/\sigma_{1:t}\right\} \left( \norm{\tilde{y}_{i-1}}^2 +  \ip{\nabla \norm{\tilde{y}_{i-1}}^2}{y_i}+ C (k-i+1) \right)^{1/2}
\end{align*}
where $\tilde{y}_{i-1} = \sum_{j=1}^{i-1} y_j$. For the above relaxation we can show that the corresponding update at round $t+i$ is given by 
\begin{align*}
f_{t+i} = \hat{f}_t - \max\left\{ 1 ,\frac{k}{ \sigma_{1:t} }\right\}  \frac{- \nabla\norm{\bar{x}_{i-1}}^2}{\sqrt{ \norm{\bar{x}_{i-1}}^2 + C (k-i+1)  }}
\end{align*}
where $k$ is the length of the current block. The next lemma shows that the proposed adaptive gradient descent recovers the results of \cite{BarHazRak07}. The method is a mixture of Follow the Leader -style algorithm and a Gradient Descent -style algorithm.

\begin{lemma}
	\label{lem:adaptive_gd}
The relaxation specified above is admissible. Suppose the adversary plays $1$-Lipchitz convex functions $x_1,\ldots,x_T$ such that for any $t \in [T]$, $\sum_{i=1}^t x_i$ is $\sigma_{1:t}$-strongly convex, and further suppose that for some $B \le 1$, we have that $\sigma_{1:t} = B t^{\alpha}$. Then, for the blocking strategy specified above,
\begin{enumerate}
\item If $\alpha \le 1/2$ then $\Reg_T \le O\left(\sqrt{T}\right)$ 
\item If $1 > \alpha > 1/2$ then $\Reg_T \le O(\frac{T^{1 - \alpha}}{B})$ 
\item If $\alpha = 1$ then $\Reg_T \le O\left(\frac{\log T}{B}\right)$ 
\end{enumerate}

\end{lemma}

\paragraph{Example: Adaptive Experts}

We now turn to the setting of Adaptive Hedge or Exponential Weights algorithm similar to the one studied in \cite{adaptivehedge11}. Consider the following situation: for all time steps after some $\tau$, there is an element (or, expert) $f$ that is the best by a margin $k$ over the next-best choice in $\F$ in terms of the (unnormalized) cumulative loss, and it remains to be the winner until the end. Let us use the localization 
\begin{align*}
\F_{r(k;x_1,\ldots,x_t)} = \left\{f \in \F ~:~ \sum_{i=1}^t \ell(f,x_i) - \min_{f \in \F} \sum_{i=1}^t \ell(f,x_i) \le k\right\} \ ,
\end{align*}
the set of functions closer than the margin to the ERM. Let 
$$\hat{\F}_t = \left\{f \in \F ~:~ \sum_{i=1}^t \ell(f,x_i) = \min_{f \in \F} \sum_{i=1}^t \ell(f , x_i)\right\}$$ 
be the set of empirical minimizers at time $t$. We use the blocking strategy 
\begin{align}
	\label{eq:hedge_blocking}
\mbf{Block}(\{x_{1},\ldots,x_t\},k) = (j, k-j)  ~~~\text{where}~~~ j =  \left\lfloor \min_{ f \notin \hat{\F}_t } \sum_{i=1}^t \ell(f,x_i) - \min_{f \in \hat{\F}_t}\sum_{i=1}^t \ell(f,x_i) \right\rfloor 
\end{align}
which says that the size of the next block is given by the gap between empirical minimizer(s) and non-minimizers. The idea behind the proof and the blocking strategy is simple. If it happens at the start a new block that there is a large gap between the current leader and the next expert, then for the number of rounds approximately equal to this gap we can play a new block and not suffer any extra regret.

Consider the relaxation \eqref{rel:expweights} used for the Exponential Weights algorithm. 

\begin{lemma}
	\label{lem:adaptivehedge}
Suppose that there exists a single best expert 
$$\hat{f}_T = \arg\min_{f \in \F} \sum_{t=1}^T \ell(f,x_t),$$ and that  
for some $k\ge1$ there exists $\tau \in [T]$ such that for all $t > \tau$ and all $f \ne \hat{f}_T$ the partial cumulative loss
$$
\sum_{i=1}^t \ell(f,x_i) -  \sum_{i=1}^t \ell(\hat{f}_T,x_i) \ge k \ .
$$
Then the regret of Algorithm~\ref{alg:adaplocmet} with the Exponential Weights relaxation, the blocking strategy \eqref{eq:hedge_blocking} and the localization mentioned above is bounded as 
$$
\Reg_T \le 4 \min\left\{\tau, \sqrt{\tau \log(|\F|)}\right\} 
$$
\end{lemma}
While we demonstrated a very simple example, the algorithm is adaptive more generally. 
Lemma~\ref{lem:adaptivehedge} considers the assumption that a single expert becomes a clear winner after $\tau$ rounds, with margin of $k$. Even when there is no clear winner throughout the game, we can still achieve low regret. For instance, this happens if only a few elements of $\F$ have low cumulative loss throughout the game and the rest of $\F$ suffers heavy loss. Then the algorithm adapts to the suboptimality and gives regret bound with the dominating term depending logarithmically only on the cardinality of the ``good'' choices in the set $\F$. Similar ideas appear in \cite{chaudhuri2009parameter}, and will be investigated in more generality in the full version of the paper.

\paragraph{Example: Adapting to the Data Norm} Recall that the set $\F^k(x_1,\ldots,x_t)$ is the subset of functions in $\F$ that are possible empirical risk minimizers when we consider $x_1,\ldots,x_{t+k}$ for some $x_{t+1},\ldots,x_{t+k}$ that can occur in the future. Now, given history $x_1,\ldots,x_{t}$ and a possible future sequence $x_{t+1},\ldots,x_{t+k}$, if $\hat{f}_{t+k}$ is an ERM for $x_{1} ,\ldots,x_{t+k}$ and $\hat{f}_{t}$ is an ERM for $x_{1} ,\ldots,x_{t}$ then 
\begin{align*}
\sum_{i=1}^t \ell(\hat{f}_{t+k},x_i) -\sum_{i=1}^t \ell(\hat{f}_{t},x_i)  & = \sum_{i=1}^{t+k} \ell(\hat{f}_{t+k},x_i) -\sum_{i=1}^{t+k} \ell(\hat{f}_{t},x_i) + \sum_{i=t+1}^{t+k} \ell(\hat{f}_{t},x_i) - \sum_{i=t+1}^{t+k} \ell(\hat{f}_{t+k},x_i)\\
& \le 0 + \sup_{x_{t+1} ,\ldots,x_{t+k}}\left\{ \sum_{i=t+1}^{t+k} \ell(\hat{f}_{t},x_i) - \sum_{i=t+1}^{t+k} \ell(\hat{f}_{t+k},x_i) \right\} \ .
\end{align*}
Hence, we see that it suffices to consider localizations
$$
\F_{r(k;x_1,\ldots,x_t)} = \left\{ f \in \F ~:~ \sum_{i=1}^t \ell(f,x_i) -\sum_{i=1}^t \ell(\hat{f}_{t},x_i) \le \sup_{x_{t+1} ,\ldots,x_{t+k}}\left\{ \sum_{i=t+1}^{t+k} \ell(\hat{f}_{t},x_i) - \sum_{i=t+1}^{t+k} \ell(f,x_i) \right\} \right\} \ .
$$
If we consider online convex Lipschitz learning problems where $\F = \{f : \norm{f} \le 1\}$ and loss is convex in $f$ and is such that $\norm{\nabla \ell(f,x)}_* \le 1$ in the dual norm $\|\cdot\|_*$, using the above argument we can use localization
\begin{align}\label{eq:datadep1}
\F_{r(k;x_1,\ldots,x_t)} = \left\{ f \in \F ~:~ \sum_{i=1}^t \ell(f,x_i) -\sum_{i=1}^t \ell(\hat{f}_{t},x_i) \le k \norm{f - \hat{f}_t}  \right\} \ .
\end{align}
Further, using Taylor approximation we can pass to the localization
\begin{align}\label{eq:datadep2}
\F_{r(k;x_1,\ldots,x_t)} = \left\{ f \in \F ~:~ \frac{1}{2}\norm{f - \hat{f}_t}^2_{x_1,\ldots,x_T} \le k \norm{f - \hat{f}_t}  \right\}
\end{align}
where $\norm{f}^2_{x_1,\ldots,x_T} = f^\top H_t f$, and $H_t$ is the Hessian of the function $g(f) = \sum_{i=1}^t \ell(f,x_i)$. Notice that the earlier example where we adapt to strong convexity of the loss is a special case of the above localization where we lower bound the data-dependent norm (Hessian-based norm) by the $\ell_2$ norm times the smallest eigenvalue. If for instance we are faced with $\eta$-exp-concave losses, such as the squared loss,  the data-dependent norm can be again lower bounded by 
$$\norm{f}^2_{x_1,\ldots,x_T} \ge \eta f^\top \left(\sum_{i=1}^t \nabla_i \right) \left(\sum_{i=1}^t \nabla_i \right)^\top f$$ 
and so we can use localization based on outer products of sum of gradients. We then do not ``pay'' for those directions in which the adversary has not played, thus adapting to the \emph{effective dimension} of the sequence of plays. 

In general, for online convex optimization problems one can use localizations given in Equations \eqref{eq:datadep1} or \eqref{eq:datadep2}. The localization in Equation \eqref{eq:datadep1} is applicable even in the linear setting, and if it so happens that the adversary mainly plays in a one dimensional sub-space, then the algorithm automatically adapts to the adversary and yields faster rates for regret. As already mentioned, the example of adaptive gradient descent is a special case of localization in Equation \eqref{eq:datadep2}. Of course, one needs to provide also an appropriate blocking strategy. A possible general blocking strategy could be :
$$
\mbf{Block}(\{x_1,\ldots,x_t\} , k) = (j,k-j), ~~~\mbox{where}~~~ j = \argmin{j \in \{0,\ldots,k\}}\left\{ \Rel{j}{\F_{r(x_1,\ldots,x_t)}} + \sup_{x_{t+1},\ldots,x_{t+j}}\Rel{k-j}{\F_{r(x_1,\ldots,x_{t+k})}} \right\} \ .$$

\vspace{1cm}
\begin{center}
\line(1,0){470}
\end{center}

In the remainder of the paper, we develop new algorithms to show the versatility of our approach. One could try to argue that the introduction of the notion of a relaxation has not alleviated the burden of algorithm development, as we simply pushed the work into magically coming up with a relaxation. We would like to stress that this is not so. A key observation is that a relaxation does not appear out of thin air, but rather as an upper bound on the sequential Rademacher complexity. Thus, a general recipe is to start with a problem at hand and develop a sequence of upper bounds until one obtains a computationally feasible one, or until other desired properties are satisfied. Exactly for this purpose, the proofs in the appendix \emph{derive} the relaxations rather than just present them as something given. Since one would follow the same upper bounding steps to prove an upper bound on the value of the game, \emph{the derivation of the relaxation and the proof of the regret bound go hand-in-hand}. For this reason, we sometimes omit the explicit mention of a regret bound for the sake of conciseness: the algorithms enjoy the same regret bound as that obtained by the corresponding non-constructive proof of the upper bound. 
\begin{center}
\line(1,0){470}
\end{center}

\section{Classification}
\label{sec:binary}

We start by considering the problem of supervised learning, where $\X$ is the space of instances and $\Y$ the space of responses (labels). There are two closely related protocols for the online interaction between the learner and Nature, so let us outline them. The ``proper'' version of supervised learning follows the protocol presented in Section~\ref{sec:value}: at time $t$, the learner selects $f_t\in\F$, Nature simultaneously selects $(x_t,y_t)\in\X\times\Y$, and the learner suffers the loss $\ell(f(x_t),y_t)$. The ``improper'' version is as follows: at time $t$, Nature chooses $x_t\in\X$ and presents it to the learner as ``side information'', the learner then picks $\hat{y}_t\in\Y$ and Nature simultaneously chooses $y_t\in\Y$. In the improper version, the loss of the learner is $\ell(\hat{y}_t,y_t)$, and it is easy to see that we may equivalently state this protocol as the learner choosing any function $f_t\in \Y^\X$ (not necessarily in $\F$), and Nature simultaneously choosing $(x_t,y_t)$. We mostly focus on the ``improper'' version of supervised learning, as the distinction does not make any difference in any of the bounds.

For the improper version of supervised learning, we may write the value in \eqref{eq:val} as
\begin{align*}
	\Val_T(\F) = \sup_{x_1\in\X} \inf_{q_1 \in \Delta(\Y)} \sup_{y_1 \in \X} \Eunderone{\hat{y}_1 \sim q_1} \ldots \sup_{x_T\in\X}\inf_{q_T \in \Delta(\Y)} \sup_{y_T \in \X} \Eunderone{\hat{y}_T \sim q_T}\left[\sum_{t=1}^T \ell(\hat{y}_t,y_t) - \inf_{f \in \F} \sum_{t=1}^T \ell(f(x_t),y_t) \right]
\end{align*}
and a relaxation $\Rel{}{}$ is admissible if for any $(x_1,y_1)\ldots,(x_T,y_T) \in \X\times \Y$, 
\begin{align}
	\label{eq:sup_relax}
 \sup_{x\in\X}\inf_{q \in \Delta(\Y)} \sup_{y \in \Y} \left\{ \underset{\hat{y} \sim q}{\En} \ell(\hat{y},y)  + \Relax{T}{\F}{\{(x_{i},y_i)\}_{i=1}^t,(x,y)}\right\} \leq \Relax{T}{\F}{\{(x_{i},y_i)\}_{i=1}^t} 
\end{align}
and 
$$ \Relax{T}{\F}{\{(x_{i},y_i)\}_{i=1}^T} \ge - \inf_{f \in \F} \sum_{t=1}^T \ell(f(x_t),y_t) .$$

Let us now focus on binary label prediction, that is $\Y=\{\pm1\}$. In this case, the supremum over $y$ in \eqref{eq:sup_relax} becomes a maximum over two values. Let us now take the absolute loss $\ell(\hat{y},y)=|\hat{y}-y| = 1-\hat{y}y$. We can see that the optimal randomized strategy, given the side information $x$, is given by \eqref{eq:sup_relax} as 
$$\argmin{q \in \Delta(\Y)} \max \left\{ 1-q  + \Relax{T}{\F}{\{(x_{i},y_i)\}_{i=1}^t,(x,1)}, 1+q+\Relax{T}{\F}{\{(x_{i},y_i)\}_{i=1}^t,(x,-1)} \right\}$$
which is achieved by setting the two expressions equal to each other:
\begin{align} 
	\label{eq:opt_distribution}
	q = \frac{1}{2}\left\{\Relax{T}{\F}{\{(x_{i},y_i)\}_{i=1}^t,(x,1)}-\Relax{T}{\F}{\{(x_{i},y_i)\}_{i=1}^t,(x,-1)} \right\}
\end{align}

This result will be specialized in the latter sections for particular relaxations $\Rel{}{}$ and extended beyond absolute loss. We remark that the extension to $k$-class prediction is immediate and involves taking a maximum over $k$ terms in \eqref{eq:sup_relax}.

\subsection{Algorithms Based on the Littlestone's Dimension}

Consider the problem of binary prediction, as described above. Further, assume that $\F$ has a finite Littlestone's dimension $\ldim(\F)$ \citep{Lit88,BenPalSSS09agnostic}. Suppose the loss function is $\ell(\hat{y},y) = |\hat{y}-y|$, and consider the ``mixed''  conditional Rademacher complexity 
\begin{align}
	\label{eq:littlestone_relaxation}
	\sup_{\x}  \En_{\epsilon}  \sup_{f \in \F} \left\{ 2 \sum_{i=1}^{T-t} \epsilon_i f(\x_i(\epsilon))  - \sum_{i=1}^{t} |f(x_i)-y_i|\right\}
\end{align}
as a possible relaxation. Observe that the above complexity is defined with the loss function removed (in a contraction-style argument  \cite{RakSriTew10nips}) in the terms involving the ``future'', in contrast with the definition \eqref{eq:def_cond_seq_rad}. The latter is defined with loss functions on both the ``future'' and the ``past'' terms. In general, if we can pass from the sequential Rademacher complexity over the loss class $\ell(\F)$ to the sequential Rademacher complexity of the base class $\F$, we may attempt to do so step-by-step by using the ``mixed'' type of sequential Rademacher complexity as in \eqref{eq:littlestone_relaxation}. This idea shall be used several times later in this paper. 

The admissibility condition \eqref{eq:sup_relax} with the conditional sequential Rademacher \eqref{eq:littlestone_relaxation} as a relaxation would require us to upper bound
\begin{align}
	\label{eq:littlestone_admissibility}
\sup_{x_t}\inf_{q_t \in [-1,1]}  \max_{y_t \in \{\pm1\}} \left\{ \Eunderone{\hat{y}_t \sim q_t} |\hat{y}_t-y_t| + \sup_{\x}  \En_{\epsilon}  \sup_{f \in \F} \left\{ 2 \sum_{i=1}^{T-t} \epsilon_i f(\x_i(\epsilon))  - \sum_{i=1}^{t} |f(x_i)-y_i|\right\}  \right\}
\end{align}
We observe that the supremum over $\x$ is preventing us from obtaining a concise algorithm. We need to further ``relax'' this supremum, and the idea is to pass to a finite cover of $\F$ on the given tree $\x$ and then proceed as in the Exponential Weights example for a finite collection of experts. This leads to an upper bound on \eqref{eq:littlestone_relaxation} and gives rise to algorithms similar in spirit to those developed in \cite{BenPalSSS09agnostic}, but with more attractive computational properties and defined more concisely. 

Define the function $g(d,t) = \sum_{i=0}^d {t \choose i}$, which is shown in \citep{RakSriTew10nips} to be the maximum size of an exact (zero) cover for a function class with the Littlestone's dimension $\ldim=d$. Given $\{(x_1,y_t),\ldots,(x_t,y_t)\}$ and $\sigma=(\sigma_1,\ldots,\sigma_t)\in\{\pm1\}^t$, let 
$$\F_t(\sigma) = \{f\in\F: f(x_i)=\sigma_i ~~\forall i\leq t\},$$
the subset of functions that agree with the signs given by $\sigma$ on the ``past'' data and let $$\F|_{x_1,\ldots,x_t}\deq\F|_{x^t}\deq\{(f(x_1),\ldots,f(x_t)): f\in\F\}$$
be the projection of $\F$ onto $x_1,\ldots,x_t$. Denote $L_t(f) = \sum_{i=1}^{t} |f(x_i)-y_i|$ and $L_t(\sigma) = \sum_{i=1}^{t} |\sigma_i-y_i|$ for $\sigma\in\{\pm1\}^{t}$. The following proposition gives a relaxation and two algorithms, both of which achieve the $O(\sqrt{\ldim(\F)T\log T})$ regret bound proved in \cite{BenPalSSS09agnostic}, yet both different from the algorithm in that paper.
\begin{proposition}
	\label{prop:relax_littlestone1}
	The relaxation
	\begin{align*}
		\Relax{T}{\F}{(x^t,y^t)} = \frac{1}{\lambda}\log \left( \sum_{\sigma \in \F|_{x^t}} g(\ldim(\F_t(\sigma)),T-t) \exp\left\{- \lambda L_{t}(\sigma) \right\}   \right) + 2\lambda (T-t) \ .
	\end{align*}
	is admissible and leads to an admissible algorithm
	\begin{align}
		\label{eq:littlestone_algo1}
		q_t (+1)= \frac{ \sum_{(\sigma,+1) \in \F|_{x^t}} g(\ldim(\F_t(\sigma,+1)),T-t) \exp\left\{- \lambda L_{t-1}(\sigma) \right\} }{\sum_{(\sigma,\sigma_t) \in \F|_{x^t}} g(\ldim(\F_t(\sigma,\sigma_t)),T-t) \exp\left\{- \lambda L_{t-1}(\sigma) \right\}},
	\end{align}
	with $q_t(-1) = 1- q_t (+1)$.
	An alternative method for the same relaxation and the same regret guarantee is to predict the label $y_t$ according to a distribution with mean
	\begin{align}
		\label{eq:littlestone_algo2}
		q_t = \frac{1}{2\lambda}\log \frac{   \sum_{(\sigma,\sigma_t) \in \F|_{x^t}}  g(\ldim(\F_t(\sigma,\sigma_t)),T-t) \exp\left\{- \lambda L_{t-1}(\sigma) \right\} \exp\left\{- \lambda (1-\sigma_t) \right\} }{ \sum_{(\sigma,\sigma_t) \in \F|_{x^t}}  g(\ldim(\F_t(\sigma,\sigma_t)),T-t) \exp\left\{- \lambda L_{t-1}(\sigma) \right\} \exp\left\{- \lambda (1+\sigma_t) \right\}}
	\end{align}
\end{proposition}

There is a very close correspondence between the proof of Proposition~\ref{prop:relax_littlestone1} and the proof of the combinatorial lemma of \citep{RakSriTew10nips}, the analogue of the Vapnik-Chervonenkis-Sauer-Shelah result.

The two algorithms presented above show two alternatives: one through employing the properties of exponential weights, and the other is through the solution in \eqref{eq:opt_distribution}. The merits of the two approaches remain to be explored. In particular, it appears that the method based on \eqref{eq:opt_distribution} can lead to some non-trivial new algorithms, distinct from the more common exponential weighting technique.

\section{Randomized Algorithms and Follow the Perturbed Leader}
\label{sec:randomized_algos}

We now develop a class of admissible randomized methods that arise through sampling. Consider the objective 
$$\inf_{q \in \Delta(\F)} \sup_{x \in \X} \left\{ \Es{f \sim q}{\ell(f,x)}  + \Relax{T}{\F}{ x_1,\ldots,x_{t-1},x}\right\}$$
given by a relaxation $\Rel{}{}$. If $\Rel{}{}$ is the sequential (or classical) Rademacher complexity, it involves an expectation over sequences of coin flips, and this computation (coupled with optimization for each sequence) can be prohibitively expensive. More generally, $\Rel{}{}$ might involve an expectation over possible ways in which the future might be realized. In such cases, we may consider a rather simple ``random playout'' strategy: draw the random sequence and solve only one optimization problem for that random sequence. The ideas of random playout have been discussed previously in the literature for estimating the utility of a move in a game  (see also \cite{AbeWarYel08}). In this section we show that, in fact, the random playout strategy has a solid basis: for the examples we consider, it satisfies admissibility. Furthermore, we show that Follow the Perturbed Leader is an example of such a randomized strategy. 

Let us informally describe the general idea, as the key steps might be hard to trace in the proofs. Suppose our objective is of the form
$$ S(q) = \sup_x \left( \En_{f\sim q} \Psi(f,x) + \En_{w\sim p} \Phi(w,x)\right)$$
for some functions $\Psi$ and $\Phi$, and $q$ a mixed strategy. We have in mind the situation where the first term is the instantaneous loss at the present round, and the second term is the expected cost for the future. Consider a randomized strategy $\tilde{q}$ which is defined by first randomly drawing $w\sim p$ and then computing 
$$f(w) \deq \argmin{f} \sup_x(\Psi(f,x) + \Phi(w,x)) $$
for the random draw $w$. We then verify that
\begin{align*}
	S(\tilde{q}) &= \sup_x\left( \En_{f\sim \tilde{q}} \Psi(f,x) + \En_{w\sim p} \Phi(w,x)\right) 
	= \sup_x\left( \En_{w\sim p} \Psi(f(w),x) + \En_{w\sim p} \Phi(w,x)\right) \\
	&~~~~~\leq \En_{w\sim p} \sup_x\left( \Psi(f(w),x) +  \Phi(w,x)\right) 
	= \En_{w\sim p} \inf_{f} \sup_x\left( \Psi(f,x) + \Phi(w,x)\right) \ .
\end{align*}
What makes the proof of admissibility possible is that the infimum in the last expression is inside the expectation over $w$ rather than outside. We can then appeal to the minimax theorem to prove admissibility. 

In our examples, $\Psi$ is the loss at round $t$ and $\Phi$ is the relaxation term, such as the sequential Rademacher complexity. In Section~\ref{sec:random_walks_trees} we show that, if we can compute the ``worst'' tree $\x$, we can randomly draw a path and use it for our randomized strategy. Note that the worst-case trees are closely related to random walks of maximal variation, and our method thus points to an intriguing connection between regret minimization and random walks (see also \cite{AbeWarYel08,HarRak10} for related ideas).
 
Interestingly, in many learning problems it turns out that the sequential Rademacher complexity and the classical Rademacher complexity are within a constant factor of each other. In such cases, the function $\Phi$ does not involve the supremum over a tree, and the randomized method only needs to draw a sequence of coin flips and compute a solution to an optimization problem slightly more complicated than ERM. 

In particular, the sequential and classical Rademacher complexities can be related for linear classes in finite-dimensional spaces. Online linear optimization is then a natural application of the randomized method we propose. Indeed, we show that Follow the Perturbed Leader (FPL) algorithm \cite{KalVem05} arises in this way. We note that FPL has been previously  considered as a rather unorthodox algorithm providing some kind of regularization via randomization. Our analysis shows that it arises through a natural relaxation based on the sequential (and thus the classical) Rademacher complexity, coupled with the random playout idea. As a new algorithmic contribution, we provide a version of the FPL algorithm for the case of the decision sets being $\ell_2$ balls, with a regret bound that is \emph{independent of the dimension}. We also provide an FPL-style method for the combination of $\ell_1$ and $\ell_\infty$ balls. To the best of our knowledge, these results are novel.

In the later sections, we provide a novel randomized method for the Trace Norm Completion problem, and a novel randomized method for the setting of static experts and transductive learning. In general, the techniques we develop might in future provide computationally feasible randomized algorithms where deterministic ones are too computationally demanding.

\subsection{When Sequential and Classical Rademacher Complexities are Related}

The assumption below implies that the sequential Rademacher complexity and the classical Rademacher complexity are within constant factor $C$ of each other. We will later verify that this assumption holds in the examples we consider.

\begin{assumption}
	\label{asm:fpl}
 There exists a distribution $D \in \Delta(\X)$ and constant $C \ge 2$ such that for any $t \in [T]$ and given any $x_1,\ldots,x_{t-1}, x_{t+1},\ldots,x_T \in \X$ and any $\epsilon_{t+1},\ldots,\epsilon_{T} \in \{\pm 1\}$,
\begin{align}
	\label{eq:assumption_fpl}
	\sup_{p \in \Delta(\X)} & \Eunderone{x_t \sim p}\sup_{f \in \F} \left[ \ C \sum_{i=t+1}^{T} \epsilon_i \ell(f,x_i) - L_{t-1}(f) + \Es{x \sim p}{\ell(f,x)} - \ell(f,x_t) \right] \notag\\
& ~~~~~~~\le \Eunderone{\epsilon_{t}, x_{t} \sim D }\sup_{f \in \F} \left[ \ C \sum_{i=t}^{T} \epsilon_i \ell(f,x_i) - L_{t-1}(f) \right]
\end{align}
where $\epsilon_t$ is an independent Rademacher random variable and $L_{t-1}(f)=\sum_{i=1}^{t-1} \ell(f,x_i)$.
\end{assumption}

Under the above assumption one can use the following relaxation 
\begin{align}
	\label{eq:fplrel}
	\Relax{T}{\F}{x_1,\ldots,x_t} = \Eunderone{x_{t+1},\ldots x_T \sim D}\En_{\epsilon}  \sup_{f \in \F} \left[C \sum_{i=t+1}^T \epsilon_i \ell(f,x_i) - \sum_{i=1}^t \ell(f,x_i) \right]
\end{align}
which is a partially symmetrized version of the classical Rademacher averages.

The proof of admissibility for the randomized methods based on this relaxation is quite curious -- the forecaster can be seen as mimicking the sequential Rademacher complexity by sampling from the ``equivalently bad'' classical Rademacher complexity under the specific distribution $D$ given by the above assumption.

\begin{lemma}
	\label{lem:fpl}
	Under the Assumption~\ref{asm:fpl}, the relaxation in Eq.~\eqref{eq:fplrel} is admissible and a randomized strategy that ensures admissibility is given by: at time $t$, draw $x_{t+1},\ldots,x_{T} \sim D$ and  Rademacher random variables $\epsilon=(\epsilon_{t+1},\ldots,\epsilon_T)$ and then :
\begin{enumerate}
\item In the case the loss $\ell$ is convex in its first argument and the set $\F$ is convex and compact, define
\begin{align}
	\label{eq:def_general_fpl}
	f_t = \argmin{g \in \F} \sup_{x \in \X} \left\{\ell(g,x) + \sup_{f \in \F} \left\{ C \sum_{i=t+1}^T \epsilon_i \ell(f,x_i)  - \sum_{i=1}^{t-1} \ell(f,x_i) - \ell(f,x) \right\} \right\}
\end{align}
\item In the case of non-convex loss, sample $f_t$ from the distribution 
\begin{align}
	\label{eq:def_general_rand_fpl}
	\hat{q}_t = \argmin{\hat{q} \in \Delta(\F)} \sup_{x \in \X} \left\{\Es{f \sim \hat{q}}{\ell(f,x)} + \sup_{f \in \F} \left\{ C \sum_{i=t+1}^T \epsilon_i \ell(f,x_i)  - \sum_{i=1}^{t-1} \ell(f,x_i) - \ell(f,x) \right\} \right\}
\end{align}
\end{enumerate}
The expected regret for the method is bounded by the classical Rademacher complexity:
$$
\E{\Reg_T} \le C\ \En_{x_{1:T} \sim D} \Es{\epsilon}{\sup_{f \in \F} \sum_{t=1}^T \epsilon_t \ell(f,x_t)} ,
$$
\end{lemma}

Of particular interest are the settings of static experts and transductive learning, which we consider in Section~\ref{sec:static}. In the transductive case, the $x_t$'s are pre-specified before the game, and in the static expert case -- effectively absent. In these cases, as we show below, there is no explicit distribution $D$ and we only need to sample the random signs $\epsilon$'s. We easily see that in these cases, the expected regret bound is simply two times the transductive Rademacher complexity.

\subsection{Linear Loss}

The idea of sampling from a fixed distribution is particularly appealing in the case of linear loss, $\ell(f,x) = \inner{f,x}$. Suppose $\X$ is a unit ball in some norm $\|\cdot\|$ in a vector space $B$, and $\F$ is a unit ball in the dual norm $\|\cdot\|_*$. Assumption~\ref{asm:fpl} then becomes 
\begin{assumption}
	\label{asm:fpl-linear}
	 There exists a distribution $D \in \Delta(\X)$ and constant $C \ge 2$ such that for any $t \in [T]$ and given any $x_1,\ldots,x_{t-1}, x_{t+1},\ldots,x_T \in \X$ and any $\epsilon_{t+1},\ldots,\epsilon_{T} \in \{\pm 1\}$,
	\begin{align}
		\label{eq:assumption_fpl_linear}
		\sup_{p \in \Delta(\X)} \Eunderone{x_t \sim p} \norm{ C \sum_{i=t+1}^{T} \epsilon_i x_i - \sum_{i=1}^{t-1} x_i + \Eunderone{x \sim p} [x] - x_t } \le \Eunderone{\epsilon_{t}, x_{t} \sim D }\norm{ C \sum_{i=t}^{T} \epsilon_i x_i - \sum_{i=1}^{t-1} x_i  }
	\end{align}
	For \eqref{eq:assumption_fpl_linear} to hold it is enough to ensure that
	\begin{align}
		\label{eq:assumption_fpl_linear_simpler}
		\sup_{p \in \Delta(\X)} \Eunderone{x_t \sim p} \norm{ w + \Eunderone{x \sim p} [x] - x_t } \le \Eunderone{\epsilon_{t}, x_{t} \sim D }\norm{ w+  C\epsilon_t x_t  }
	\end{align}	
	for any $w\in B$.
\end{assumption}

At round $t$, the generic algorithm specified by Lemma~\ref{eq:def_general_rand_fpl} draws fresh Rademacher random variables $\epsilon$ and $x_{t+1},\ldots,x_T \sim D$ and picks 
\begin{align}\label{eq:fplgenup}
f_t = \argmin{f \in \F} \sup_{x \in \X} \left\{\ip{f}{x} +  \norm{ C \sum_{i=t+1}^T \epsilon_i x_i  - \sum_{i=1}^{t-1} x_i - x } \right\}
\end{align}
We now look at specific examples of $\ell_2/\ell_2$ and $\ell_1/\ell_\infty$ cases and provide closed form solution of the randomized algorithms.\\

\noindent {\bf Example :  $\ell_1/\ell_\infty$ Follow the Perturbed Leader: } 

Here, we consider the setting similar to that in \cite{KalVem05}. Let $\F\subset\reals^N$ be the $\ell_1$ unit ball and $\X$ the (dual) $\ell_{\infty}$ unit ball in $\reals^N$. In \cite{KalVem05}, $\F$ is the probability simplex and $\X = [0,1]^N$ but these are subsumed by the $\ell_1/\ell_\infty$ case. We claim that:
\begin{lemma}
	\label{lem:fpl-l1linfty}
	Assumption~\ref{asm:fpl-linear} is satisfied with a distribution $D$ that is uniform on the vertices of the cube $\{\pm1\}^N$ and $C=6$.
\end{lemma}

In fact, one can pick any symmetric distribution $D$ on the real line and use $D^N$ for the perturbation. Assumption~\ref{asm:fpl-linear} is then satisfied, as we show in the following lemma.
\begin{lemma}\label{lem:genlinffpl}
If $D$ is any symmetric distribution over the real line, then Assumption~\ref{asm:fpl-linear} is satisfied by using the product distribution $D^N$. The constant $C$ required is any $ C \ge 6/\En_{x \sim D} |x|$.
\end{lemma}

The above lemma is especially attractive when used with standard normal distribution because in that case as sum of normal random variables is again normal. Hence, instead of drawing $x_{t+1},\ldots,x_{T} \sim N(0,1)$ on round $t$, one can simply draw just one vector $X_t \sim N(0,\sqrt{T-t})$ and use it for perturbation. In this case constant $C$ is bounded by $8$. 

While we have provided simple distributions to use for perturbation, the form of update in Equation \eqref{eq:fplgenup} is not in a convenient form. The following lemma shows a simple Follow the Perturbed Leader type algorithm with the associated regret bound.

\begin{lemma}\label{lem:l1fplmain}
Suppose $\F$ is the $\ell^N_1$ unit ball and $\X$ is the dual $\ell_\infty^N$ unit ball, and let $D$ be any symmetric distribution. Consider the randomized algorithm that at each round  $t$ freshly draws Rademacher random variables $\epsilon_{t+1} , \ldots, \epsilon_T$ and freshly draws $x_{t+1},\ldots,x_T \sim D^N$ (each co-ordinate drawn independently from $D$) and picks 
$$
f_t = \argmin{f \in \F}\ip{f}{\sum_{i=1}^{t-1} x_i - C \sum_{i=t+1}^T \epsilon_i x_i } 
$$ 
where $C = 6/\Es{x \sim D}{|x|}$. The randomized algorithm enjoys a bound on the expected regret given by
$$
\E{\Reg_T} \le C\ \Eunderone{x_{1:T} \sim D^N}\En_{\epsilon} \norm{\sum_{t=1}^T \epsilon_t x_t}_\infty +  4 \sum_{t=1}^T  \P_{y_{t+1:T} \sim D}\left( C \left| \sum_{i=t+1}^T  y_i\right|  \le 4 \right)
$$
\end{lemma}

Notice that for $D$ being the $\{\pm1\}$ coin flips or standard normal distribution, the probability $$\P_{y_{t+1},\ldots,y_T \sim D}\left( C \left| \sum_{i=t+1}^T  y_i\right|  \le 4 \right)$$ is exponentially small in $T-t$ and so $\sum_{t=1}^T  \P_{y_{t+1},\ldots,y_T \sim D}\left( C \left| \sum_{i=t+1}^T  y_i\right|  \le 4 \right)$ is bounded by a constant. For these cases, we have
$$
\E{\Reg_T} \le O\left( \Eunderone{x_{1:T} \sim D^N}\En_{\epsilon} \norm{\sum_{t=1}^T \epsilon_t x_t}_\infty \right) = O\left(\sqrt{T \log N}\right)
$$
This yields the logarithmic dependence on the dimension, matching that of the Exponential Weights algorithm.

\noindent {\bf Example :  $\ell_2/\ell_2$ Follow the Perturbed Leader: } 

We now consider the case when $\F$ and $\X$ are both the unit $\ell_2$ ball. We can use as perturbation the uniform distribution on the surface of unit sphere, as the following lemma shows. This result was already hinted at in \cite{AbeBarRakTew08colt}, as the random draw from the unit sphere is likely to produce an orthogonal direction, yielding a strategy close to optimal. However, we do not require dimensionality to be high for the result to hold.

\begin{lemma}
	\label{lem:fpl-l2l2}
Let $\X$ and $\F$ be unit balls in Euclidean norm.  Then Assumption~\ref{asm:fpl-linear} is satisfied with a uniform distribution $D$ on the surface of the unit sphere with constant $C=4\sqrt{2}$.
\end{lemma}

Again as in the previous example the form of update in Equation \eqref{eq:fplgenup} is not in a convenient form and this is addressed in the following lemma. 

\begin{lemma}\label{lem:l2fplmain}
Let $\X$ and $\F$ be unit balls in Euclidean norm, and $D$ be the uniform distribution on the surface of the unit sphere. Consider the randomized algorithm that at each round (say round $t$) freshly draws  $x_{t+1},\ldots,x_T \sim D$ and picks 
$$
f_t = \frac{- \sum_{i=1}^{t-1} x_i + C \sum_{i=t+1}^T x_i }{\sqrt{\norm{- \sum_{i=1}^{t-1} x_i + C \sum_{i=t+1}^T \epsilon_i x_i }_2^2 + 1}}
$$ 
where $C=4 \sqrt{2}$.
The randomized algorithm enjoys a bound on the expected regret given by
$$
\E{\Reg_T} \le C\ \En_{x_1,\ldots,x_T \sim D} \norm{\sum_{t=1}^T x_t}_2 \le 4 \sqrt{2 T} 
$$
\end{lemma}

Importantly, the bound does not depend on the dimensionality of the space. To the best of our knowledge, this is the first such result for Follow the Perturbed Leader style algorithms.

\begin{remark}
	The FPL methods developed in \cite{KalVem05,PLG} assume that the adversary is oblivious. With this simplification, the algorithms can reuse the same random perturbation drawn at the beginning of the game. It is then argued in \cite{PLG} that the methods also work for non-oblivious opponents since the FPL strategy is fully determined by the outcomes played by the adversary \cite[Remark 4.2]{PLG}. In contrast, our proofs directly deal with the adaptive adversary.
\end{remark}

\subsection{Supervised Learning}

For completeness, let us state a version of Assumption~\ref{asm:fpl} for the case of supervised  learning. That is, the side information $x_t$ is presented to the learner, who then picks $\hat{y}_t$ and observes the outcome $y_t$.

\begin{assumption}\label{asm:fpl2}
There exists a distribution $D \in \Delta(\X\times \Y)$ and constant $C \ge 2$ such that for any $t \in [T]$ and given any $(x_1,y_1),\ldots,(x_{t-1},y_{t-1}), (x_{t+1},y_{t+1}),\ldots,(x_T,y_T) \in \X \times \Y$ and any $\epsilon_{t+1},\ldots,\epsilon_{T} \in \{\pm 1\}$,
\begin{align*}
\sup_{x_t \in \X}\sup_{p_t \in \Delta(\Y)} & \Eunderone{y_t \sim p_t} \sup_{f \in \F} \left[ \ C \sum_{i=t+1}^{T} \epsilon_i \ell(f(x_i),y_i) -L_{t-1}(f) + \Es{y \sim p_t}{\ell(f(x_t),y)} - \ell(f(x_t),y_t) \right] \\
& \le \Eunderone{\epsilon_{t}, (x_{t},y_t) \sim D }\sup_{f \in \F} \left[ \ C \sum_{i=t}^{T} \epsilon_i \ell(f(x_i),y_i) - L_{t-1}(f) \right],
\end{align*}
where $\epsilon_t$ is an independent Rademacher random variable and   $L_{t-1}(f)=\sum_{i=1}^{t-1} \ell(f(x_i),y_i)$.
\end{assumption}

Under the Assumption~\ref{asm:fpl2}, we can use the following relaxation:
\begin{align}\label{eq:fplrel2}
\Relax{T}{\F}{(x_1,y_1),\ldots,(x_t,y_t)} = \Eunder{(x_{t+1},y_{t+1}),\ldots (x_T,y_T) \sim D}{\epsilon_{t+1:T}}  \sup_{f \in \F} \left[ C \sum_{i=t+1}^T \epsilon_i \ell(f(x_i),y_i) - \sum_{i=1}^t \ell(f(x_i),y_i) \right]
\end{align}

\begin{lemma}\label{lem:fpl2}
	Under the Assumption \ref{asm:fpl2}, the relaxation in Eq.~\eqref{eq:fplrel2} is admissible and a randomized strategy that ensures admissibility is given by: at time $t$, draw $(x_{t+1},y_{t+1}),\ldots,(x_{T},y_T) \sim D$ and Rademacher random variables $\epsilon_{t+1},\ldots,\epsilon_T$ and then :
\begin{enumerate}
\item In the case the loss $\ell$ is convex in its first argument, define 
\begin{align}
	\label{eq:def_general_fpl2}
	\hat{y}_t = \argmin{\hat{y} \in [-B,B]} \sup_{y_t \in \Y} \left\{ \ell(\hat{y},y_t)+ \sup_{f \in \F} \left[ C \sum_{i=t+1}^T \epsilon_i \ell(f(x_i),y_i)  - \sum_{i=1}^{t} \ell(f(x_i),y_i) \right] \right\}
\end{align}
and
\item In the case of non-convex loss, pick $\hat{y}_t$ from the distribution 
\begin{align}
	\label{eq:def_general_rand_fpl2}
	\hat{q}_t = \argmin{\hat{q} \in \Delta([-B,B])} \sup_{y_t \in \Y} \left\{ \Es{\hat{y} \sim \hat{q}}{\ell(\hat{y},y_t)}+ \sup_{f \in \F} \left[ C \sum_{i=t+1}^T \epsilon_i \ell(f(x_i),y_i)  - \sum_{i=1}^{t} \ell(f(x_i),y_i) \right] \right\}
\end{align}
\end{enumerate}
The expected regret bound of the method (in both cases) is
$$
\E{\Reg_T} \le C\ \Eunderone{(x_1,y_1),\ldots,(x_T,y_T) \sim D} \Es{\epsilon}{\sup_{f \in \F} \sum_{t=1}^T \epsilon_t \ell(f(x_t),y_t)}
$$
\end{lemma}

\subsection{Random Walks with Trees}
\label{sec:random_walks_trees}

We can also define randomized algorithms without the assumption that the classical and the sequential Rademacher complexities are close. Instead, we assume that we have a black-box access to a procedure that on round $t$ returns the ``worst-case'' tree $\x^t$ of depth $T-t$. 

\begin{lemma} \label{lem:treewalk}
Given any $x_1,\ldots,x_{t-1}$ let 
\begin{align}
	\label{eq:worst_case_tree_fpl}
	\x^t \deq \argmax{\x}\En_{\epsilon} \sup_{f \in \F} \left[ 2 \sum_{i=t+1}^T \epsilon_i \ell(f,\x_i(\epsilon)) - \sum_{i=1}^t \ell(f,x_i) \right] \ .
\end{align} 
Consider the randomized strategy where at round $t$ we first draw $\epsilon_{t+1},\ldots,\epsilon_T$ uniformly at random and then further draw our move $f_t$ according to the distribution
\begin{align}
	\label{eq:treewalk_mixed}
q_t(\epsilon) = \argmin{q \in \Delta(\F)} \sup_{x_t}\left\{ \Es{f_t \sim q}{\ell(f_t,x_t)} +  \sup_{f \in \F} \left[ 2 \sum_{i=t+1}^T \epsilon_i \ell(f,\x^t_i(\epsilon)) - \sum_{i=1}^t \ell(f,x_i) \right] \right\}
\end{align}
The expected regret of this randomized strategy is bounded by sequential Rademacher complexity:
\begin{align*}
\E{\Reg_T} \le \Rad_T(\F) \ .
\end{align*}
 \end{lemma}

Thus, if for any given history $x_1,\ldots,x_{t-1}$ we can compute $\x^t$ in \eqref{eq:worst_case_tree_fpl}, or even just draw directly a random path $\x^t_1(\epsilon),\ldots,\x^t_{T-t}(\epsilon)$ on each round, then we obtain a randomized strategy that in expectation can guarantee a regret bound equal to sequential Rademacher complexity. Also notice that whenever the optimal strategy in \eqref{eq:treewalk_mixed} is deterministic (e.g. in the online convex optimization scenario), one does not need the double randomization. Instead, in such situations one can directly draw $\epsilon_1,\ldots,\epsilon_{T-t}$ and use
$$
f_t(\epsilon) = \argmin{f_t \in \F} \sup_{x_t}\left\{ \ell(f_t,x_t) +  \sup_{f \in \F}\left\{ 2 \sum_{i=t+1}^T \epsilon_i \ell(f,\x^t_i(\epsilon)) - \sum_{i=1}^t \ell(f,x_i) \right\} \right\}
$$

\section{Static Experts with Convex Losses and Transductive Online Learning}
\label{sec:static}

We show how to recover a variant of the $R^2$ forecaster of \cite{CBShamir11efficient}, for static experts and transductive online learning. At each round, the learner makes a prediction $q_t\in[-1,1]$, observes the outcome $y_t\in[-1,1]$, and suffers convex $L$-Lipschitz loss $\ell(q_t,y_t)$. Regret is defined as the difference between learner's cumulative loss and 
$\inf_{f\in F} \sum_{t=1}^T \ell(f[t],y_t)$, where $F\subset [-1,1]^T$ can be seen as a set of static experts. The transductive setting is equivalent to this: the sequence of $x_t$'s is known before the game starts, and hence the effective function class is once again a subset of $[-1,1]^T$.


It turns out that in the static experts case, sequential Rademacher complexity boils down to the classical Rademacher complexity (see \cite{RakSriTew11nips}), and thus the relaxation in \eqref{eq:opt_distribution} can be taken to be the classical, rather than sequential, Rademacher averages. This is also the reason that an efficient implementation by sampling is possible. Furthermore, for the absolute loss, the factor of $2$ that appears in the sequential Rademacher complexity is not needed. For general convex loss, one possible relaxation is just a conditional version of the classical Rademacher averages:
\begin{align}
	\label{eq:classical_rad_relax_static}
	\Relax{T}{\F}{y_{1},\ldots,y_t} = \En_{\epsilon_{t+1:T}} \sup_{f\in F} \left[ 2L\sum_{s=t+1}^T \epsilon_s f[s] - L_t(f) \right]
\end{align}
where $L_t(f)=\sum_{s=1}^t \ell(f[s],y_s)$. This relaxation can be shown to be admissible.

First, consider the case of absolute loss $\ell(q_t,y_t)=|q_t-y_t|$ and binary-valued outcomes $y_t\in\{\pm1\}$. In this case, 
the solution in \eqref{eq:opt_distribution} yields the algorithm
\begin{align*}
	q_t &= \frac{1}{2}\En_{\epsilon_{t+1:T}}\left[ \sup_{f\in F} \left( \sum_{s=t+1}^T \epsilon_s f[s] - L_{t-1}(f) + f[t] \right) - \sup_{f\in F} \left( \sum_{s=t+1}^T \epsilon_s f[s] - L_{t-1}(f) - f[t] \right)\right]
\end{align*}
which corresponds to the well-known minimax optimal forecaster for static experts with absolute loss \cite{PLG}. Plugging in this value of $q_t$ into Eq.~\eqref{eq:sup_relax} proves admissibility, and thus the regret guarantee of this method is equal to the classical Rademacher complexity.

We now derive two variants of the $R^2$ forecaster for the more general case of $L$-Lipschitz loss and $y_t\in[-1,1]$. 

\paragraph{First Alternative : }
If \eqref{eq:classical_rad_relax_static} is used as a relaxation, the calculation of prediction $\hat{y}_t$ involves a supremum over $f \in F$ with (potentially nonlinear) loss functions of instances seen so far. In some cases this optimization might be hard and it might be  preferable if the supremum only involves terms \emph{linear} in $f$. This is the idea behind he first method we present. To this end we start by noting that by convexity
$$
\sum_{t=1}^T \ell(\hat{y}_t,y_t)  - \inf_{f \in \F} \sum_{t=1}^T \ell(f(x_t),y_t) \le \sum_{t=1}^T \partial \ell(\hat{y}_t,y_t) \cdot \hat{y}_t  - \inf_{f \in \F} \sum_{t=1}^T \partial \ell(\hat{y}_t,y_t) \cdot  f[t]
$$
Now given the above, one can consider an alternative online learning problem which, if we solve, also solves the original problem. That is, consider the online learning problem with the new loss
$$
\ell'(\hat{y},r) = r \cdot \hat{y}
$$
In this alternative game, we first pick prediction $\hat{y}_t$ (deterministically), next the adversary picks $r_t$ (corresponding to $r_t = \partial \ell(\hat{y}_t,y_t)$ for choice of $y_t$ picked by adversary). Now note that $\ell'$ is indeed convex in its first argument and is $L$ Lipschitz because $|\partial \ell(\hat{y}_t,y_t)| \le L$. 
This is a one dimensional convex learning game where we pick $\hat{y}_t$ and regret is given by
\begin{align*}
\Reg_T = \sum_{t=1}^T \partial \ell(\hat{y}_t,y_t) \cdot \hat{y}_t  - \inf_{f \in F} \sum_{t=1}^T \partial \ell(\hat{y}_t,y_t) \cdot f[t]
\end{align*}

One can consider the relaxation 
\begin{align}\label{eq:trans2rel}
\Relax{T}{\F}{  \partial \ell(\hat{y}_1,y_1) , \ldots, \partial \ell(\hat{y}_t,y_t)} = \En_{\epsilon_{t+1:T}} \sup_{f \in F} \left[ 2 L \sum_{i=t+1}^T \epsilon_i f[t] - \sum_{i=1}^t \partial \ell(\hat{y}_i,y_i) \cdot f[i] \right]
\end{align}
as a linearized form of \eqref{eq:classical_rad_relax_static}. At round $t$, the prediction of the algorithm is then 
\begin{align}\label{eq:trans2pred}
\hat{y}_t =   \Es{\epsilon}{\sup_{f \in F} \left\{ \sum_{i=t+1}^{T} \epsilon_i f[i] - \frac{1}{2L} \sum_{i=1}^{t-1} \partial\ell(\hat{y}_i,y_i) f[i] + \frac{1}{2} f[t] \right\} - \sup_{f \in F} \left\{  \sum_{i=t+1}^{T} \epsilon_i f[i] -  \frac{1}{2L} \sum_{i=1}^{t-1} \partial\ell(\hat{y}_i,y_i) f[i] -  \frac{1}{2} f[t] \right\}}
\end{align}

\begin{lemma}
\label{lem:ontrvar2}
	The relaxation in Equation \eqref{eq:trans2rel} is admissible with respect to the prediction strategy specified in Equation \eqref{eq:trans2pred}. Further the regret of the strategy is bounded as
$$
\Reg_T \le 2 L\ \Es{\epsilon}{\sup_{f \in F} \sum_{t=1}^T \epsilon_t f[t]}
$$
\end{lemma}

The presented algorithm is similar in principle to $R^2$, with the main difference that $R^2$ computes the infima over a sum of absolute losses, while here we have a more manageable linearized objective. Note that while we need to evaluate the expectation over $\epsilon$'s on each round, we can estimate $\hat{y}_t$ by sampling $\epsilon$'s and using McDiarmid's inequality to argue that, with enough draws, our estimate is close to $\hat{y}_t$ with high probability. What is interesting, we can develop a randomized method that only draws one sequence of $\epsilon$'s per step, as shown next.

\paragraph{Second Alternative : }
Consider the non-linearized relaxation 
\begin{align}\label{eq:trans1rel}
\Relax{T}{\F}{y_1,\ldots,y_t} =  \Es{\epsilon}{\sup_{f \in F}  \ 2 L \sum_{i=t+1}^{T} \epsilon_i f[i] - \sum_{i=1}^{t} \ell(f[i],y_i) } 
\end{align}
already given in \eqref{eq:classical_rad_relax_static}. We now present a randomized method based on the ideas of Section~\ref{sec:randomized_algos}: at round $t$ we first draw $\epsilon_{t+1},\ldots,\epsilon_T$ and predict 
\begin{align}\label{eq:trans1pred}
\hat{y}_t(\epsilon) & =  \left(\inf_{f \in F} \left\{  - \sum_{i=t+1}^{T} \epsilon_i f[i] + \frac{1}{2 L} \sum_{i=1}^{t-1} \ell(f[i],y_i) + \frac{1}{2} f[t] \right\} - \inf_{f \in F} \left\{ - \sum_{i=t+1}^{T} \epsilon_i f[i] + \frac{1}{2 L}\sum_{i=1}^{t-1} \ell(f[i],y_i) - \frac{1}{2} f[t] \right\} \right) 
\end{align}
We show that this predictor in expectation enjoys regret bound of the transductive Rademacher complexity. More specifically we have the following lemma.

\begin{lemma}
	\label{lem:transd-assm2-satisfied}
	The relaxation specified in Equation \eqref{eq:trans1rel} is admissible w.r.t. the randomized prediction strategy specified in Equation \eqref{eq:trans1pred}. Further the expected regret of the randomized strategy is bounded as
	$$
	\E{\Reg_T} \le 2 L\ \Es{\epsilon}{\sup_{f \in F} \sum_{t=1}^T \epsilon_t f[t]}
	$$
\end{lemma}

In the next section, we employ both alternatives to develop novel algorithms for matrix completion.

\section{Matrix Completion}

Consider the problem of predicting unknown entries in a matrix (as in collaborative filtering). We focus here on an online formulation, where at each round $t$ the adversary picks an entry in an $m\times n$ matrix and a value $y_t$ for that entry (we shall assume without loss of generality that $n\geq m$). The learner then chooses a predicted value $\hat{y}_t$, and suffers loss $\ell(y_t,\hat{y}_t)$, which we shall assume to be $\rho$-Lipschitz. We define our regret with respect to the class $\F$ which we will take to be the set of all matrices whose trace-norm is at most $B$ (namely, we can use any such matrix to predict just by returning its relevant entry at each round). Usually, one sets $B$ to be on the order of $\sqrt{mn}$.

We consider here a transductive version, where the sequence of entry locations is known in advance, and only the entry values are unknown. We show how to develop an algorithm whose regret is bounded by the (transductive) Rademacher complexity of $\F$. We note that in Theorem 6 of \cite{ShalSahm09}, this complexity was shown to be at most order $B \sqrt{n}$ independent of $T$. Moreover, in \cite{CBShamir11efficient}, it was shown that for algorithms with such guarantees, and whose play each round does not depend on the order of future entries, under mild conditions on the loss function one can get the same regret even in the ``fully'' online case where the set of entry locations is unknown in advance. Algorithmically, all we need to do is pretend we are in a transductive game where the sequence of entries is all $m\times n$ entries, in some arbitrary order. In this section we use the two alternatives provided for transductive learning problem in the previous subsection and provide two alternatives for the matrix completion problem.

We note that both variants proposed here improve on the one provided by the $R^2$ forecaster in \cite{CBShamir11efficient}, since that algorithm
competes against the smaller class $\F'$ of matrices with bounded trace-norm \emph{and} bounded individual entries. In contrast, our algorithm provides similar regret guarantees against the larger class of matrices only whose trace-norm is bounded. Moreover, the variants are also computationally more efficient.

\paragraph{First Alternative : }
The algorithm we now present is obtained by using the first method for online tranductive learning proposed in the previous section. The relaxation in Equation \eqref{eq:trans2rel} for the specific problem at hand is given by, 
\begin{align} \label{eq:tracerel}
\Rel{\F}{y_1,\ldots,y_t} =  B\ \Es{\epsilon}{\norm{ 2 \rho \sum_{i=t+1}^{T} \epsilon_i x_i - \sum_{i=1}^{t} \partial  \ell(\hat{y}_i,y_i) x_i }_{\mrm{\sigma}}}
\end{align}
In the above $\norm{\cdot}_{\sigma}$ stands for the spectral norm and  each $x_i$ is a matrix with a $1$ at some specific position and $0$ elsewhere. That is $x_i$ at round $i$ can be seen as the entry of the matrix which we are asked to fill in at round $i$. 
The prediction at round $t$ returned by the algorithm is given by Equation \eqref{eq:trans2pred} which for this problem is given by 
\begin{align*}\label{eq:tracestra}
\hat{y}_t = B\ \Es{\epsilon}{\left( \norm{ \sum_{i=t+1}^{T} \epsilon_i x_i - \frac{1}{2 \rho} \sum_{i=1}^{t-1} \partial\ell(\hat{y}_i,y_i) x_i  + \frac{1}{2} x_t}_{\sigma} - 
\norm{\sum_{i=t+1}^{T} \epsilon_i x_i - \frac{1}{2\rho}  \sum_{i=1}^{t-1} \partial\ell(\hat{y}_i,y_i) x_i  - \frac{1}{2} x_t}_{\sigma} \right)}
\end{align*}
\noindent Notice that the algorithm only involves calculation of spectral norms on each round which can be done efficiently. Again as mentioned in previous subsection, one can evaluate the expectation over random signs by sampling $\epsilon$'s on each round.

\paragraph{Second Alternative : }
The second algorithm is obtained from the second alternative for online transductive learning with convex losses in the previous section. The relaxation given in Equation \eqref{eq:trans1rel} for the case of matrix completion problem with trace norm constraint is given by:
$$
\Relax{T}{\F}{y_1,\ldots,y_t} =  \Es{\epsilon}{\sup_{f : \norm{f}_{\Sigma} \le B}  \ 2 \rho \sum_{i=t+1}^{T} \epsilon_i \ip{f}{x_i} - \sum_{i=1}^{t} \ell(\ip{f}{x_i},y_i) } 
$$
where $\norm{\cdot}_{\Sigma}$ stands for the race norm of the $m \times n$ matrix $f$ and each $x_i$ is a matrix with a $1$ at some specific position and $0$ elsewhere. That is $x_i$ at round $i$ can be seen as the entry of the matrix which we are asked to fill in at round $i$. We use $\ip{f}{x}$ to represent the generalized inner product of the two matrices. Since we only take inner products with respect to the matrices $x_i$, each $\ip{f}{x_i}$ is simply the value of matrix $f$ at the position specified by $x_i$'s. The prediction at a matrix entry corresponding to position $x_t$ is given by first drawing random $\{\pm 1\}$ valued $\epsilon$'s and then applying Equation \eqref{eq:trans1pred} to the problem at hand, yielding 
{\small \begin{align*}
\hat{y}_t(\epsilon) & =  \inf_{\norm{f}_{\Sigma} \le B} \left\{  - \sum_{i=t+1}^{T} \epsilon_i \ip{f}{x_i} + \frac{1}{2 \rho} \sum_{i=1}^{t-1} \ell(\ip{f}{x_i},y_i) + \frac{1}{2} \ip{f}{x_t} \right\} - \inf_{\norm{f}_{\Sigma} \le B} \left\{ - \sum_{i=t+1}^{T} \epsilon_i \ip{f}{x_i} + \frac{1}{2 \rho}\sum_{i=1}^{t-1} \ell(\ip{f}{x_i},y_i) - \frac{1}{2} \ip{f}{x_t} \right\} 
\end{align*}}
Notice that the above involves solving two trace norm constrained convex optimization problems per round.
As a simple corollary of Lemma \ref{lem:transd-assm2-satisfied} we get the following bound on expected regret of the algorithm. 

\begin{corollary} \label{cor:trace1}
For the randomized prediction strategy specified above, the expected regret is bounded as
\begin{align*}
\E{\Reg_{T}} \le 2 B\ \rho\   \Es{\epsilon}{\norm{ \sum_{t=1}^T \epsilon_t x_t}_{\sigma}}  \le O\left( B\ \rho\ (\sqrt{m} + \sqrt{n}) \right)
\end{align*}
\end{corollary}
The last inequality in the above corollary is using Theorem 6 in \cite{ShalSahm09}.

\begin{corollary} \label{cor:trace2}
For the predictions $\hat{y}_t$ specified above, the regret is bounded as
\begin{align*}
\Reg_{T} \le O\left( B\ \rho\ (\sqrt{m} + \sqrt{n}) \right)
\end{align*}
\end{corollary}

\section{More Examples}
\subsection{Constrained Adversaries}

We now show that algorithms can be also developed for situations when the adversary is constrained in the choices per step. Such constrained problems have been treated in a general non-algorithmic way in \cite{RakSriTew11nips}, and we picked the case of variation-constrained adversary for illustration. It is shown in \cite{RakSriTew11nips} that the value of the game where the adversary is constrained to keep the next move $x_t$ within $\sigma_t$ from the average of the past moves $\frac{1}{t-1}\sum_{s=1}^{t-1}x_s$ is upper bounded as
\begin{align}
\Val_T & \le 2 \sup_{(\x,\x') \in \mc{T}} \Es{\epsilon}{\sup_{f \in \F} \sum_{t=1}^T \epsilon_t\left(\inner{f,\x_t(\epsilon)} - \frac{1}{t-1} \sum_{\tau=1}^{t-1} \inner{f,\chi_\tau(\epsilon_\tau)} \right)}
\end{align}
where the supremum is over $\x,\x'$ trees satisfying the above mentioned constraint per step, and the selector $\chi_t(\epsilon_t)$ is defined as $\x_t(\epsilon)$ if $\epsilon_t=-1$ and $\x'_t(\epsilon)$ otherwise. In our algorithmic framework, this leads to the following problem that needs to be solved at each step:
\begin{align*}
	\inf_{f_t}\sup_{x_t} \left\{\inner{f_t,x_t}+2 \sup_{(\x,\x') \in \mc{T}} \Es{\epsilon}{\sup_{f \in \F}\inner{f, \sum_{s=t+1}^T \epsilon_s\left(\x_s(\epsilon) - \frac{1}{s-t} \sum_{\tau=t+1}^{s-1} \chi_\tau(\epsilon_\tau) \right) - \sum_{r=1}^t x_r}} \right\}
\end{align*}
where the supremum is taken over $x_t$ such that the constraint $C(x_1,\ldots,x_t)$ is satisfied and $\mc{T}$ is the set of trees that satisfy the constraints as continuation of the prefix $x_1,\ldots,x_t$. While this expression gives rise to an algorithm, we are aiming for a more computationally feasible method. In fact, passing to an upper bound on the sequential Rademacher complexity yields the following result.
\begin{lemma}
	\label{lem:constrained}
	The following relaxation is admissible and upper bounds the constrained sequential complexity
	\begin{align*}
		\Relax{T}{\F}{ x_{1},\ldots,x_t}=  \frac{2\sqrt{2}R}{\sqrt{\lambda}} \sqrt{ \left\|\sum_{r=1}^t x_r \right\|^2 +  C\sum_{s=t+1}^T\sigma_s^2 }
	\end{align*}
	Furthermore, the an admissible algorithm for this relaxation is Mirror Descent with a step size given at time $t\geq 2$ by 
	$$\frac{\left(1+\frac{1}{t-1}\right)^2}{2\sqrt{\|\tilde{x}_{t-1}\|^2 + C\sum_{s=t}^T\sigma_s^2}}$$
\end{lemma}

\subsection{Universal Mirror Descent}
In \cite{SreSriTew11} it is shown that for the problem of general online convex optimization, the Mirror Descent algorithm is universal and near optimal (up to poly-log factors). Specifically, it is shown that there always exists an appropriate function $\Psi$ such that the Mirror Descent algorithm using this function, along with an appropriate step size, gives the near optimal rate. Moreover, it is shown in \cite{SreSriTew11} that one can use function $\Psi$ whose convex conjugate is given by
\begin{align}
	\Psi^* (x) = \sup_{\x}  \Es{\epsilon}{\left\|x + \sum_{i=1}^{T-t} \epsilon_i \x_i(\epsilon) \right\|^p - C \sum_{i=1}^{T-t} \Es{\epsilon}{\left\|\x_i(\epsilon)\right\|^p}},
\end{align}
as the ``universal regularizer" for the Mirror Descent algorithm. We now show that this function arises rather naturally from the sequential Rademacher relaxation and, moreover, the Mirror Descent  algorithm itself arises from this relaxation.

Let us denote the convex cost functions chosen by the adversary as $\ell_t$, and let $x_t$ be the subgradients $x_t = \nabla \ell_t(f_t)$ of the convex functions.  
\begin{lemma}
	\label{lem:get_mirror}
	The relaxation
\begin{align*}
	\Relax{T}{\F}{x_{1},\ldots,x_t} = \left( \Psi^*\left(\sum_{i=1}^{t-1} x_i \right)  +\ip{\nabla \Psi^*\left(\sum_{i=1}^{t-1} x_i\right)}{x_t} + C (T-t+1) \right)^{1/p}
\end{align*}
is an upper bound on the conditional sequential Rademacher complexity. Further, whenever for some $p' > p$ we have that $\Val_T(\F) \le (C T)^{1/p'}$, then the relaxation is admissible and leads to a form of Mirror Descent algorithm with regret bounded as 
$$
\Reg_T \le (C T)^{1/p}
$$
\end{lemma}

It is remarkable that the universal regularizer and the Mirror Descent algorithm arise naturally, in a few steps of algebra, as upper bounds on the sequential Rademacher complexity.

\newpage
\appendix

\section{PROOFS}

\begin{proof}[\textbf{Proof of Proposition~\ref{prop:main}}]
	By definition, 
	$$\sum_{t=1}^T \En_{f_t\sim q_t}\ell(f_t,x_t) - \inf_{f\in\F} \sum_{t=1}^T \ell(f,x_t) \leq \sum_{t=1}^T \En_{f_t\sim q_t}\ell(f_t,x_t) + \Relax{T}{\F}{x_{1},\ldots,x_{T}} \ .$$
	Peeling off the $T$-th expected loss, we have
	\begin{align*}
		\sum_{t=1}^T \En_{f_t\sim q_t}\ell(f_t,x_t) + \Relax{T}{\F}{x_{1},\ldots,x_{T}} &\leq \sum_{t=1}^{T-1} \En_{f_t\sim q_t}\ell(f_t,x_t) + \left\{\En_{f_t\sim q_t}\ell(f_t,x_t) + \Relax{T}{\F}{x_{1},\ldots,x_{T}} \right\} \\
		&\leq \sum_{t=1}^{T-1} \En_{f_t\sim q_t}\ell(f_t,x_t) +  \Relax{T}{\F}{x_{1},\ldots,x_{T-1}} 
		\end{align*}
		where we used the fact that $q_T$ is an admissible algorithm for this relaxation, and thus the last inequality holds for any choice $x_T$ of the opponent. Repeating the process, we obtain
		$$\sum_{t=1}^T \En_{f_t\sim q_t}\ell(f_t,x_t) - \inf_{f\in\F} \sum_{t=1}^T \ell(f,x_t) \leq \Rel{T}{\F} \ .$$
		We remark that the left-hand side of this inequality is random, while the right-hand side is not. Since the inequality holds for any realization of the process, it also holds in expectation. The inequality $$\Val_T(\F) \le \Rel{T}{\F}$$ holds by unwinding the value recursively and using admissibility of the relaxation. The high-probability bound is an immediate consequences of \eqref{eq:sum_cond_exp_bdd_by_relax} and the Hoeffding-Azuma inequality for bounded martingales. The last statement is immediate.
\end{proof}

\begin{proof}[\textbf{Proof of Proposition~\ref{prop:rad_admissible}}]
	Denote $L_t(f) = \sum_{s=1}^t \ell(f,x_s)$. The first step of the proof is an application of the minimax theorem (we assume the necessary conditions hold):
	\begin{align*}
		&\inf_{q_t \in \Delta(\F)} \sup_{x_t \in \X} \left\{ \Eunderone{f_t \sim q_t}\left[\ell(f_t,x_t)\right] + \sup_{\x} \En_{\epsilon_{t+1:T}} \sup_{f\in\F} \left[ 2\sum_{s=t+1}^T \epsilon_s\ell(f,\x_{s-t}(\epsilon_{t+1:s-1})) - L_t(f) \right]\right\} \\
		&=\sup_{p_t \in \Delta(\X)} \inf_{f_t \in \F} \left\{ \Eunderone{x_t \sim p_t}\left[\ell(f_t,x_t)\right] + \Eunderone{x_t \sim p_t}\sup_{\x} \En_{\epsilon_{t+1:T}} \sup_{f\in\F} \left[ 2\sum_{s=t+1}^T \epsilon_s\ell(f,\x_{s-t}(\epsilon_{t+1:s-1})) - L_t(f) \right]\right\} 
	\end{align*}
	For any $p_t\in \Delta(\X)$, the infimum over $f_t$ of the above expression is equal to
	\begin{align*} 
		&\Eunderone{x_t \sim p_t}\sup_{\x} \En_{\epsilon_{t+1:T}} \sup_{f\in\F} \left[ 2\sum_{s=t+1}^T \epsilon_s\ell(f,\x_{s-t}(\epsilon_{t+1:s-1})) - L_{t-1}(f) + \inf_{f_t \in \F}\Eunderone{x_t \sim p_t}\left[\ell(f_t,x_t)\right] - \ell(f,x_t) \right] \\
		&\leq  \Eunderone{x_t \sim p_t}\sup_{\x} \En_{\epsilon_{t+1:T}} \sup_{f\in\F} \left[ 2\sum_{s=t+1}^T \epsilon_s\ell(f,\x_{s-t}(\epsilon_{t+1:s-1})) - L_{t-1}(f) + \Eunderone{x_t \sim p_t}\left[\ell(f,x_t)\right] - \ell(f,x_t) \right] \\
		&\leq  \Eunderone{x_t,x'_t \sim p_t}\sup_{\x} \En_{\epsilon_{t+1:T}} \sup_{f\in\F} \left[ 2\sum_{s=t+1}^T \epsilon_s\ell(f,\x_{s-t}(\epsilon_{t+1:s-1})) - L_{t-1}(f) + \ell(f,x'_t) - \ell(f,x_t) \right] 
	\end{align*}
	We now argue that the independent $x_t$ and $x'_t$ have the same distribution $p_t$, and thus we can introduce a random sign $\epsilon_t$. The above expression then equals to
	\begin{align*}
		&\Eunderone{x_t,x'_t \sim p_t}\Eunderone{\epsilon_t} \sup_{\x} \En_{\epsilon_{t+1:T}} \sup_{f\in\F} \left[ 2\sum_{s=t+1}^T \epsilon_s\ell(f,\x_{s-t}(\epsilon_{t+1:s-1})) - L_{t-1}(f) + \epsilon_t(\ell(f,x'_t) - \ell(f,x_t)) \right] \\
		&\leq \sup_{x_t,x'_t \in\X}\Eunderone{\epsilon_t} \sup_{\x} \En_{\epsilon_{t+1:T}} \sup_{f\in\F} \left[ 2\sum_{s=t+1}^T \epsilon_s\ell(f,\x_{s-t}(\epsilon_{t+1:s-1})) - L_{t-1}(f) + \epsilon_t(\ell(f,x'_t) - \ell(f,x_t)) \right] 
		\end{align*}
		where we upper bounded the expectation by the supremum. Splitting the resulting expression into two parts, we arrive at the upper bound of 
		\begin{align*}
		&2\sup_{x_t\in\X}\Eunderone{\epsilon_t} \sup_{\x} \En_{\epsilon_{t+1:T}} \sup_{f\in\F} \left[ \sum_{s=t+1}^T \epsilon_s\ell(f,\x_{s-t}(\epsilon_{t+1:s-1})) - \frac{1}{2}L_{t-1}(f) + \epsilon_t \ell(f,x_t) \right] 
		= \Rad_T (\F | x_1,\ldots,x_{t-1}) \ .
	\end{align*}
	The last equality is easy to verify, as we are effectively adding a root $x_t$ to the two subtrees, for $\epsilon_t=+1$ and $\epsilon_t=-1$, respectively. 
	
	One can see that the proof of admissibility corresponds to one step minimax swap and symmetrization in the proof of \cite{RakSriTew10nips}. In contrast, in the latter paper, all $T$ minimax swaps are performed at once, followed by $T$ symmetrization steps.
\end{proof}

\begin{proof}[\textbf{Proof of Proposition~\ref{prop:exp_weights_relax}}]
	Let us first prove that the relaxation is admissible with the Exponential Weights algorithm as an admissible algorithm. Let $L_t(f) = \sum_{i=1}^t \ell(f,x_i)$. Let $\lambda^*$ be the optimal value in the definition of $\Relax{T}{\F}{x_{1},\ldots,x_{t-1}}$. Then
\begin{align*}
	&\inf_{q_t \in \Delta(\F)} \sup_{x_t \in \X} \left\{ \underset{f \sim q_t}{\En}\left[\ell(f,x_t)\right] + \Relax{T}{\F}{x_{1},\ldots,x_t}\right\} \\
	&\leq\inf_{q_t \in \Delta(\F)} \sup_{x_t \in \X} \left\{ \underset{f \sim q_t}{\En}\left[\ell(f,x_t)\right] + \frac{1}{\lambda^*}\log\left( \sum_{f \in \F}  \exp\left(   - \lambda^* L_t(f)   \right) \right)  +  2\lambda^* (T-t) \right\}
\end{align*}
Let us upper bound the infimum by a particular choice of $q$ which is the exponential weights distribution 
$$q_t(f) = \exp(-\lambda^* L_{t-1}(f))/Z_{t-1}$$ 
where $Z_{t-1} = \sum_{f \in \F}  \exp\left(   - \lambda^* L_{t-1}(f)   \right) $. By \cite[Lemma A.1]{PLG}, 
\begin{align*}
	\frac{1}{\lambda^*}\log\left( \sum_{f \in \F}  \exp\left(   - \lambda^* L_t(f)   \right) \right) &= \frac{1}{\lambda^*}\log\left( \En_{f\sim q_t} \exp\left(   - \lambda^* \ell(f,x_t)   \right) \right) + \frac{1}{\lambda^*}\log Z_{t-1} \\
	&\leq -\En_{f\sim q_t} \ell(f,x_t) + \frac{\lambda^*}{2} + \frac{1}{\lambda^*}\log Z_{t-1}
\end{align*}
Hence,
\begin{align*}
	\inf_{q_t \in \Delta(\F)} \sup_{x_t \in \X} \left\{ \underset{f \sim q_t}{\En}\left[\ell(f,x_t)\right] + \Relax{T}{\F}{x_{1},\ldots,x_t}\right\} 
	&\leq \frac{1}{\lambda^*}\log\left( \sum_{f \in \F}  \exp\left(   - \lambda^* L_{t-1}(f)   \right) \right)  +  2\lambda^* (T-t+1) \\
	&~~~ = \Relax{T}{\F}{x_{1},\ldots,x_{t-1}}
\end{align*}
by the optimality of $\lambda^*$. The bound can be improved by a factor of $2$ for some loss functions, since it will disappear from the definition of sequential Rademacher complexity.

We conclude that the Exponential Weights algorithm is an admissible strategy for the relaxation \eqref{rel:expweights}.

\paragraph{Arriving at the relaxation}
We now show that the Exponential Weights relaxation arises naturally as an upper bound on sequential Rademacher complexity of a finite class. For any $\lambda>0$, 
	\begin{align*}
\Es{\epsilon}{ \sup_{f \in \F} \left\{ 2 \sum_{i=1}^{T-t} \epsilon_i \ell(f,\x_i(\epsilon))  - L_t(f)\right\}   } 
 &\le \frac{1}{\lambda}\log\left(\Es{\epsilon}{ \sup_{f \in \F}  \exp\left(    2\lambda \sum_{i=1}^{T-t} \epsilon_i \ell(f,\x_i(\epsilon))  - \lambda L_t(f)    \right)  } \right)\\
& \le \frac{1}{\lambda}\log\left(\Es{\epsilon}{ \sum_{f \in \F}  \exp\left(    2\lambda \sum_{i=1}^{T-t} \epsilon_i \ell(f,\x_i(\epsilon))  - \lambda L_t(f)    \right)  } \right)\\
&= \frac{1}{\lambda}\log\left( \sum_{f \in \F}  \exp\left(   - \lambda L_t(f)    \right)  \Es{\epsilon}{\prod_{i=1}^{T-t} \exp\left( 2 \lambda  \epsilon_i \ell(f,\x_i(\epsilon)) \right)  } \right)
\end{align*}
We now upper bound the expectation over the ``future'' tree by the worst-case path, resulting in the upper bound
\begin{align*}
& \frac{1}{\lambda}\log\left( \sum_{f \in \F}  \exp\left(   - \lambda L_t(f)    \right) \times  \exp\left( 2 \lambda^2  \max_{\epsilon_1,\ldots\epsilon_{T-t} \in \{\pm1\}}\sum_{i=1}^{T-t}  \ell(f,\x_i(\epsilon))^2 \right)   \right)\\
&\le \frac{1}{\lambda}\log\left( \sum_{f \in \F}  \exp\left(   - \lambda L_t(f)    + 2 \lambda^2  \max_{\epsilon_1,\ldots\epsilon_{T-t} \in \{\pm1\}}\sum_{i=1}^{T-t} \ell(f,\x_i(\epsilon))^2 \right)   \right)\\
& \le \frac{1}{\lambda}\log\left( \sum_{f \in \F}  \exp\left(   - \lambda L_t(f)   \right) \right)  +  2 \lambda \sup_{\x} \sup_{f \in \F} \max_{\epsilon_1,\ldots\epsilon_{T-t} \in \{\pm1\}}\sum_{i=1}^{T-t}  \ell(f,\x_i(\epsilon))^2  
\end{align*}
The last term, representing the ``worst future'', is upper bounded by $2\lambda(T-t)$. This removes the $\x$ tree and leads to the relaxation \eqref{rel:expweights} and a computationally tractable algorithm.
\end{proof}

\begin{proof}[\textbf{Proof of Proposition~\ref{prop:mirror_relax}}]

The argument can be seen as a generalization of the Euclidean proof in \cite{AbeBarRakTew08colt} to general smooth norms. The proof below not only shows that the Mirror Descent algorithm is admissible for the relaxation \eqref{rel:mirror}, but in fact shows that it coincides with the optimal algorithm for the relaxation, i.e. the one that attains the infimum over strategies.

Let $\tilde{x}_{t-1} = \sum_{i=1}^{t-1} x_i$. The optimal algorithm for the relaxation \eqref{rel:mirror} is 
$$
f_t = \argmin{f \in \F}\left\{ \sup_{x_t \in \X} \left\{ \ip{f}{x_t} +   \sqrt{ \norm{ \tilde{x}_{t-1}}^2 + \ip{\nabla \norm{\tilde{x}_{t-1}}^2}{x_t} + C (T - t + 1) } \right\} \right\}
$$
Now write any $f_t$ as $f_t = -\alpha \nabla \|\tilde{x}_{t-1}\|^2 +  g$ for some $g \in \mrm{Kernel}(\nabla \|\tilde{x}_{t-1}\|^2) \deq \left\{h: \ip{\nabla \|\tilde{x}_{t-1}\|^2}{h}=0\right\}$, and any $x_t$ as 
$x_t = \beta \tilde{x}_{t-1} + \gamma y$
for some $y \in \mrm{Kernel}(\nabla \|\tilde{x}_{t-1}\|^2)$.
Hence we can write:
\begin{align}
	\label{eq:grad_desc_expansion}
	& \ip{f_t}{x_t} +   \left( \|\tilde{x}_{t-1} \|^2  +\ip{\nabla \|\tilde{x}_{t-1}\|^2}{x_t} + C (T-t+1) \right)^{1/2}    \notag\\
	& = -\alpha \beta \|\tilde{x}_{t-1}\|^2 + \gamma \ip{g}{y} +   \left( \|\tilde{x}_{t-1}\|^2 + \beta  \|\tilde{x}_{t-1}\|^2 + C (T - t + 1) \right)^{1/2}  
\end{align}
Given any $f_t = -\alpha \nabla \norm{\tilde{x}_{t-1}}^2 +  g$, $x$ can be picked with $y \in \mrm{Kernel}(\nabla \norm{\tilde{x}_{t-1}}^2)$ that satisfies $\ip{g}{y} \ge 0$. One can always do this because if for some $y'$, $\ip{g}{y'} < 0$ by picking $y = -y'$ we can ensure that $\ip{g}{y} \ge 0$. Hence the minimizer $f_t$ must be once such that $f_t = -\alpha \nabla \norm{\tilde{x}_{t-1}}^2$ and thus $\ip{g}{y}=0$. Now, it must be that $\alpha \ge 0$ so that $x_t$ either increases the first term or second term but not both. Hence we have that  
$f_t = - \alpha \nabla \norm{\tilde{x}_{t-1}}^2$ for some $\alpha \ge 0$.
Now given such an $f_t$, the sup over $x_t$ can be written as supremum over $\beta$ of a concave function, which gives rise to the derivative condition 
$$
-\alpha  \norm{\tilde{x}_{t-1}}^2  +   \frac{\norm{\tilde{x}_{t-1}}^2}{2 \sqrt{ \norm{ \tilde{x}_{t-1}}^2 + \beta  \norm{\tilde{x}_{t-1}}^2 + C (T - t + 1) }} = 0
$$ 
At this point it is clear that the value of 
\begin{align}
	\label{eq:optalpha}
	\alpha = \frac{1}{2 \sqrt{\norm{ \tilde{x}_{t-1}}^2 +  C (T - t + 1) }}
\end{align}
forces $\beta=0$. Let us in fact show that this value is optimal. We have
$$
\frac{1}{4 \alpha^2  } = \norm{ \tilde{x}_{t-1}}^2 + \beta  \norm{\tilde{x}_{t-1}}^2 + C (T - t + 1)
$$ 
Plugging this value of $\beta$ back, we now aim to optimize
$$
\frac{1}{4 \alpha} + \alpha \norm{ \tilde{x}_{t-1}}^2 + \alpha C (T - t + 1)
$$
over $\alpha$. We then obtain the value given in \eqref{eq:optalpha}. With this value, we have the familiar update
\begin{align}
	\label{eq:gradient_descent}
	f_t = -\frac{\nabla \norm{\tilde{x}_{t-1}}^2}{2 \sqrt{\norm{ \tilde{x}_{t-1}}^2 +  C (T - t + 1)}} \ .
\end{align}
Plugging back the value of $\alpha$, we find that $\beta = 0$. With these values, 
\begin{align*}
\inf_{f \in \F}&\left\{ \sup_{x \in \X} \left\{ \ip{f}{x} +   \left( \|\tilde{x}_{t-1} \|^2  +\ip{\nabla \|\tilde{x}_{t-1}\|^2}{x} + C (T-t+1) \right)^{1/2}  \right\} \right\} = \left( \|\tilde{x}_{t-1}\|^2 + C (T - t + 1) \right)^{1/2}  \\
&~~~~~~~~~~~~~\leq \left( \|\bar{x}_{t-2}\|^2 + \ip{\nabla\|\bar{x}_{t-2}\|^2}{x_{t-1}} + C (T - t + 2) \right)^{1/2} =\Relax{T}{\F}{x_{1},\ldots,x_{t-1}}
\end{align*}
We have shown that \eqref{eq:gradient_descent} is an optimal algorithm for the relaxation, and it is admissible.

\paragraph{Arriving at the Relaxation} The derivation of the relaxation is immediate:
\begin{align}
	\Rad_T (\F | x_1,\ldots,x_t) &= \sup_{\x} \En_{\epsilon_{t+1:T}} \left\| \sum_{s=t+1}^T \epsilon_s \x_{s-t}(\epsilon_{t+1:s-1}) - \sum_{s=1}^t x_s \right\| \\
	&\leq \sup_{\x} \sqrt{\En_{\epsilon_{t+1:T}} \left\| \sum_{s=t+1}^T \epsilon_s \x_{s-t}(\epsilon_{t+1:s-1}) - \sum_{s=1}^t x_s \right\|^2 }\\
	&\leq \sup_{\x} \sqrt{ \left\| \sum_{s=1}^t x_s \right\|^2 + C \En_{\epsilon_{t+1:T}}\sum_{s=t+1}^T\left\| \epsilon_s \x_{s-t}(\epsilon_{t+1:s-1})\right\|^2}
\end{align}
where the last step is due to the smoothness of the norm and the fact that the first-order terms disappear under the expectation. The sum of norms is now upper bounded by $T-t$, thus removing the dependence on the ``future'', and we arrive at
\begin{align*}
 \sqrt{ \left\| \sum_{s=1}^t x_s \right\|^2 + C(T-t)} \leq \sqrt{ \left\| \sum_{s=1}^{t-1} x_s \right\|^2 + \ip{\nabla \norm{\sum_{s=1}^{t-1} x_s}^2}{x_t} +  C(T-t+1)}
\end{align*}
as a relaxation on the sequential Rademacher complexity.
\end{proof}

\begin{proof}[\textbf{Proof of Lemma~\ref{lem:adaptive_gd}}]
We shall first establish the admissibility of the relaxation specified. To show admissibility, let us first check the initial condition:
\begin{align*}
\Relax{k}{\F_{r(k;x_1,\ldots,x_t)}}{y_1,\ldots,y_k} & = - \ip{\hat{f}_{t}}{\tilde{y}_k} + 2\min\left\{1, \frac{k}{\sigma_{1:t}}\right\} \sqrt{ \norm{\sum_{j=1}^{k-1} y_j}^2 +  \ip{\nabla \norm{\sum_{j=1}^{k-1}y_j}^2}{y_k}+ C }\\
& \ge - \ip{\hat{f}_{t}}{\tilde{y}_k} + 2\min\left\{1, \frac{k}{\sigma_{1:t}}\right\} \sqrt{ \norm{\tilde{y}_k}^2}\\
& \ge - \ip{\hat{f}_{t}}{\tilde{y}_k} + \sup_{f : \norm{f - \hat{f}_t} \le 2\min\{1 , \frac{k}{\sigma_{1:t}} \}} \ip{f - \hat{f}_t}{- \tilde{y}_k}\\
& \ge - \inf_{f : \norm{f - \hat{f}_t} \le 2\min\{1 , \frac{k}{\sigma_{1:t}} \}} \sum_{j=1}^{k} \ip{f }{ y_j}
\end{align*}
Now, for the recurrence, we have
\begin{align*}
& \ip{f_i}{y_i} + \sup_{\y} \Es{\epsilon}{\sup_{f : \|f -\hat{f}_{t}\| \le 2\min\left\{1,\frac{k}{\sigma_{1:t}}\right\} }  \ip{f}{\sum_{j=1}^{k-i} \epsilon_j \y_j(\epsilon) - \sum_{j=1}^{i} y_j} } \\
& = \ip{f_i}{y_i} - \ip{\hat{f}_{t}}{\sum_{j=1}^i y_j}+ \sup_{\y} \Es{\epsilon}{\sup_{f : \|f -\hat{f}_{t}\| \le 2\min\left\{1, \frac{k}{\sigma_{1:t}}\right\}}  \ip{f - \hat{f}_t}{\sum_{j=1}^{k-i} \epsilon_j \y_j(\epsilon) - \sum_{j=1}^{i} y_j} } \\
& \le \ip{f_i}{y_i} - \ip{\hat{f}_{t}}{\sum_{j=1}^i y_j} + 2\min\left\{ 1, \frac{k}{\sigma_{1:t}}\right\} \sup_{\y} \Es{\epsilon}{\norm{\sum_{j=1}^{k-i} \epsilon_j \y_j(\epsilon)- \sum_{j=1}^{i} y_j} } \\
\intertext{}
& \le \ip{f_i}{y_i} - \ip{\hat{f}_{t}}{\sum_{j=1}^{i} y_j} + 2\min\left\{1, \frac{k}{\sigma_{1:t}}\right\} \sqrt{ \norm{\tilde{y}_i}^2 + C (k-i) }\\
& \le \ip{f_i}{y_i} - \ip{\hat{f}_{t}}{\tilde{y}_{i}} + 2\min\left\{1, \frac{k}{\sigma_{1:t}}\right\} \sqrt{ \norm{\tilde{y}_i}^2 +  \ip{\nabla \norm{\tilde{y}_{i-1}}^2}{\y_i}+ C (k-i+1) }\\
& = \ip{f_i - \hat{f}_{t}}{y_i} - \ip{\hat{f}_{t}}{\tilde{y}_{i-1}} + 2\min\left\{1, \frac{k}{\sigma_{1:t}}\right\} \sqrt{ \norm{\tilde{y}_{i-1}}^2 +  \ip{\nabla \norm{\tilde{y}_{i-1}}^2}{\x_i}+ C (k-i+1) }
\end{align*}

and we start block at $\hat{f}_t$. For the first block, this value is $0$ but later on it is the empirical risk minimizer. We therefore get a mixture of Follow the Leader (FTL) and Gradient Descent (GD) algorithms. If block size is $1$, we get FTL only, and when the block size is $T$ we get GD only. In general, however, the resulting method is an interesting mixture of the two. 
Using the arguments of Proposition~\ref{rel:mirror}, the update in the block is given by
\begin{align*}
f_{t+i} = \hat{f}_t - \max\left\{ 1 ,\frac{k}{ \sigma_{1:t} }\right\}  \frac{- \nabla\norm{\tilde{y}_{i-1}^2}}{\sqrt{ \norm{\tilde{y}_{i-1}}^2 + C (k-i+1)  }}
\end{align*}

Now that we have shown the admissibility of the relaxation and the form of update obtained by the relaxation we turn to the bounds on the regret specified in the lemma. We shall provide these bounds using Lemma \ref{lem:adaptive}. We will split the analysis to two cases, one when $\alpha > 1/2$ and other when $\alpha \le 1/2$.

{\bf Case $\alpha > \frac{1}{2}$ :}\\
To start note that since we initialize the block lengths with the doubling trick, that is initialize block lengths as $1, 2, 4, \ldots$ hence, after $t$ rounds the maximum length of current block say $k$ can be at most $2t$ and so $\sqrt{k} \le \sqrt{2t}$. Now let us first consider the case when $\alpha > \frac{1}{2}$. In this case, since $\sigma_{1:t} = B t^{\alpha}$, we can conclude that the condition $\sigma_{1:t} \ge \sqrt{k}$ is satisfied as long as $t^{\alpha - \frac{1}{2}} \ge \frac{\sqrt{2}}{B}$. Since we are considering the case when $\alpha > \frac{1}{2}$ we can conclude that for all rounds larger than $\sqrt{2}/B$, the blocking strategy always picks block size of $1$. Hence applying Lemma \ref{lem:adaptive} we conclude that in the case when $1 > \alpha > 1/2$ (or when $\alpha = 1/2$ and $B \ge \sqrt{2}$),
\begin{align*}
\Reg_T \le \sum_{t=1}^T \frac{1}{\sigma_{1:t}} = \sum_{t=1}^T \frac{1}{B t^\alpha} = O(T^{1-\alpha}/B)
\end{align*}
Also note that for the case when $\alpha = 1$, the summation is bounded by $O(\log T)$ and so 
\begin{align*}
\Reg_T \le \sum_{t=1}^T \frac{1}{\sigma_{1:t}} = \sum_{t=1}^T \frac{1}{B t^\alpha} = O(\log T / B)
\end{align*}

{\bf Case $\alpha \le \frac{1}{2}$ :}\\
Now we consider the case when $\alpha < 1/2$. Say we are at start of some block $t = 2^m$. The initial block length then is $2 t$ by the doubling trick initialization. Now within this block, the adaptive algorithm continues with this current block until the point when the square-root of the remaining number of rounds in the block say $k$ becomes smaller than $\sigma_{1 : t + (2t - k)}$. That is until
\begin{align}\label{eq:int1}
\sqrt{k} \le B (3t - k)^{\alpha} 
\end{align}
The regret on this block can be bounded using Lemma \ref{lem:adaptive}  (notice that here we use the lemma for the algorithm within a sub-block initialized by the doubling trick rather than on the entire $T$ rounds). The regret on this block is bounded as :
\begin{align*}
\Rel{2t - k}{\F_{r(x_1,\ldots,x_t)}} + \sum_{i=2t - k+1}^{2t} \Rel{1}{\F_{r(x_1,\ldots,x_{i})}} & \le \sqrt{2t -k} +  \sum_{j=2t -k +1}^{2 t} \frac{1}{B j^{\alpha}} \\
& \le \sqrt{2t} +  \sum_{j=2t -k +1}^{2 t} \frac{1}{B j^{\alpha}} \\
& \le  \sqrt{2t} +  \frac{1}{B}\left( (2t+1)^{1 - \alpha} - (2t - k +1)^{1 - \alpha} \right)\\
& \le  \sqrt{2t} +  \frac{k^{1 - \alpha}}{B}\\
& \le \sqrt{2t} +  \frac{B^{2(1 - \alpha)} (3t)^{2\alpha(1 - \alpha)} }{B} & \textrm{(using Eq. \eqref{eq:int1})}\\
& \le  \sqrt{2t} + B^{2(1 - \alpha)- 1} \sqrt{3t}   \\
& \le  \sqrt{12\ t}  
\end{align*}
Hence overall regret is bounded as 
\begin{align*}
\Reg_T \le \sum_{i=1}^{\lceil \log_2 T \rceil + 1} \sqrt{12\ \times 2^{i-1}}  \le \sqrt{12} \sum_{i=1}^{\lceil \log_2 T \rceil + 1} 2^{(i-1)/2} \le O(\sqrt{T})
\end{align*}
This concludes the proof.
\end{proof}

\begin{proof}[\textbf{Proof of Lemma~\ref{lem:adaptivehedge}. }]
	
Notice that by doubling trick for at most first $2 \tau$ rounds we simply play the experts algorithm, thus suffering a maximum regret that is minimum of $\tau$ and $4 \sqrt{\tau \log|\F|}$. After these initial number of rounds, consider any round $t$ at which we start a new block with the blocking strategy described above. The first sub-block given by the blocking strategy is of length at most $k$, thanks to our assumption about the gap between the leader and the second-best action. Clearly the minimizer of cumulative loss up to $t$ rounds already played, $\argmin{f \in \F} \sum_{i=1}^t \ell(f,x_i)$, is going to be the leader at least for the next $k$ rounds. Hence for this block we suffer no regret. Now when we use the same blocking strategy repeatedly, due to the same reasoning, we end up playing the same leader for the rest of the game only in chunks of size $k$, and thus suffer no regret for the rest of the game.
\end{proof}

\begin{proof}[\textbf{Proof of Proposition~\ref{prop:relax_littlestone1}}]
We would like to show that, with the distribution $q^*_t$ defined in \eqref{eq:littlestone_algo1},
\begin{align*}
	&\max_{y_t \in \{\pm1\}} \left\{ \underset{\hat{y}_t \sim q^*_t}{\En}|\hat{y}_t-y_t| + \Relax{T}{\F}{(x^t,y^t)}\right\} \leq \Relax{T}{\F}{(x^{t-1},y^{t-1})}
\end{align*}
for any $x_t\in\X$. Let $\sigma\in\{\pm1\}^{t-1}$ and $\sigma_t\in\{\pm1\}$. We have
\begin{align*}
	&\Relax{T}{\F}{(x^t,y^t)} - 2\lambda (T-t)  \\
	&=\frac{1}{\lambda}\log \left(   \sum_{(\sigma,\sigma_t) \in \F|_{x^t}}  g(\ldim(\F_t(\sigma,\sigma_t)),T-t) \exp\left\{- \lambda L_{t-1}(\sigma) \right\} \exp\left\{- \lambda |\sigma_t-y_t| \right\} \right) \\
	&\leq \frac{1}{\lambda}\log \left(  \sum_{\sigma_t\in\{\pm1\}} \exp\left\{- \lambda |\sigma_t-y_t| \right\}  \sum_{\sigma: (\sigma,\sigma_t) \in \F|_{x^t}} g(\ldim(\F_t(\sigma,\sigma_t)),T-t) \exp\left\{- \lambda L_{t-1}(\sigma) \right\}  \right) 
\end{align*}
Just as in the proof of Proposition~\ref{prop:exp_weights_relax}, we may think of the two choices $\sigma_t$ as the two experts whose weighting $q_t^*$ is given by the sum involving the Littlestone's dimension of subsets of $\F$. Introducing the normalization term, we arrive at the upper bound
\begin{align*}
	&\frac{1}{\lambda}\log \left(  \En_{\sigma_t\sim q^*_t} \exp\left\{- \lambda |\sigma_t-y_t| \right\}  \right) + \frac{1}{\lambda}\log \left( \sum_{\sigma_t\in\{\pm1\}}\sum_{\sigma: (\sigma,\sigma_t) \in \F|_{x^t}} g(\ldim(\F_t(\sigma,\sigma_t)),T-t) \exp\left\{- \lambda L_{t-1}(\sigma) \right\} \right) \\
	&\leq -  \En_{\sigma_t\sim q^*_t} |\sigma_t-y_t| + 2\lambda + \frac{1}{\lambda}\log \left( \sum_{\sigma_t\in\{\pm1\}}\sum_{\sigma: (\sigma,\sigma_t) \in \F|_{x^t}} g(\ldim(\F_t(\sigma,\sigma_t)),T-t) \exp\left\{- \lambda L_{t-1}(\sigma) \right\} \right) 
\end{align*}
The last step is due to Lemma A.1 in \cite{PLG}. It remains to show that the log normalization term is upper bounded by the relaxation at the previous step:
\begin{align*}
	& \frac{1}{\lambda}\log \left( \sum_{\sigma_t\in\{\pm1\}} \sum_{\sigma: (\sigma,\sigma_t) \in \F|_{x^t}} g(\ldim(\F_t(\sigma,\sigma_t)),T-t) \exp\left\{- \lambda L_{t-1}(\sigma) \right\}  \right)\\
	&\leq \frac{1}{\lambda}\log \left( \sum_{\sigma \in \F|_{x^{t-1}}} \exp\left\{- \lambda L_{t-1}(\sigma) \right\}  \sum_{\sigma_t\in\{\pm1\}} g(\ldim(\F_t(\sigma,\sigma_t)),T-t)  \right)  \\
	&\leq \frac{1}{\lambda}\log \left(  \sum_{\sigma \in \F|_{x^{t-1}}} \exp\left\{- \lambda L_{t-1}(\sigma) \right\}  g(\ldim(\F_{t-1}(\sigma)),T-t+1) \right) \\
	&= \Relax{T}{\F}{(x^{t-1},y^{t-1})}
\end{align*}
To justify the last inequality, note that $\F_{t-1}(\sigma)=\F_{t}(\sigma,+1)\cup\F_t(\sigma,-1)$ and at most one of $\F_t(\sigma,+1)$ or $\F_t(\sigma,-1)$ can have Littlestone's dimension $\ldim(\F_{t-1}(\sigma))$. We now appeal to the recursion $$g(d,T-t)+g(d-1,T-t) \leq g(d, T-t+1)$$ 
where $g(d,T-t)$ is the size of the zero cover for a class with Littlestone's dimension $d$ on the worst-case tree of depth $T-t$  (see \citep{RakSriTew10nips}). This completes the proof of admissibility.


\paragraph{Alternative Method} Let us now derive the algorithm given in \eqref{eq:littlestone_algo2} and prove its admissibility. Once again, consider the optimization problem
\begin{align*}
	&\max_{y_t \in \{\pm1\}} \left\{ \underset{\hat{y}_t \sim q^*_t}{\En}|\hat{y}_t-y_t| + \Relax{T}{\F}{(x^t,y^t)}\right\}
\end{align*}
with the relaxation
\begin{align*}
	\Relax{T}{\F}{(x^t,y^t)} = \frac{1}{\lambda}\log \left( \sum_{\sigma \in \F|_{x^t}} g(\ldim(\F_t(\sigma)),T-t) \exp\left\{- \lambda L_{t}(\sigma) \right\}   \right) + \frac{\lambda}{2} (T-t)
\end{align*}

The maximum can be written explicitly, as in Section~\ref{sec:binary}:
\begin{align*}
	\max&\left\{1-q^*_t + \frac{1}{\lambda}\log \left(   \sum_{(\sigma,\sigma_t) \in \F|_{x^t}}  g(\ldim(\F_t(\sigma,\sigma_t)),T-t) \exp\left\{- \lambda L_{t-1}(\sigma) \right\} \exp\left\{- \lambda (1-\sigma_t) \right\} \right) , \right.\\
&\left. 1+q^*_t + \frac{1}{\lambda}\log \left(   \sum_{(\sigma,\sigma_t) \in \F|_{x^t}}  g(\ldim(\F_t(\sigma,\sigma_t)),T-t) \exp\left\{- \lambda L_{t-1}(\sigma) \right\} \exp\left\{- \lambda (1+\sigma_t) \right\} \right) \right\}
\end{align*}
where we have dropped the $\frac{\lambda}{2}(T-t)$ term from both sides. Equating the two values, we obtain
\begin{align*}
	2q^*_t = \frac{1}{\lambda}\log \frac{   \sum_{(\sigma,\sigma_t) \in \F|_{x^t}}  g(\ldim(\F_t(\sigma,\sigma_t)),T-t) \exp\left\{- \lambda L_{t-1}(\sigma) \right\} \exp\left\{- \lambda (1-\sigma_t) \right\} }{ \sum_{(\sigma,\sigma_t) \in \F|_{x^t}}  g(\ldim(\F_t(\sigma,\sigma_t)),T-t) \exp\left\{- \lambda L_{t-1}(\sigma) \right\} \exp\left\{- \lambda (1+\sigma_t) \right\}}
\end{align*}
The resulting value becomes
\begin{align*}
	&1+\frac{\lambda}{2}(T-t)+\frac{1}{2\lambda}\log \left\{   \sum_{(\sigma,\sigma_t) \in \F|_{x^t}}  g(\ldim(\F_t(\sigma,\sigma_t)),T-t) \exp\left\{- \lambda L_{t-1}(\sigma) \right\} \exp\left\{- \lambda (1-\sigma_t) \right\} \right\}  \\
	&~~~~~~~~~~~~~~~~+ \frac{1}{2\lambda}\log \left\{\sum_{(\sigma,\sigma_t) \in \F|_{x^t}}  g(\ldim(\F_t(\sigma,\sigma_t)),T-t) \exp\left\{- \lambda L_{t-1}(\sigma) \right\} \exp\left\{- \lambda (1+\sigma_t) \right\} \right\} \\
	&= 1+\frac{\lambda}{2}(T-t)+\frac{1}{\lambda}\En_\epsilon \log \left\{   \sum_{(\sigma,\sigma_t) \in \F|_{x^t}}  g(\ldim(\F_t(\sigma,\sigma_t)),T-t) \exp\left\{- \lambda L_{t-1}(\sigma) \right\} \exp\left\{- \lambda (1-\epsilon\sigma_t) \right\} \right\}  \\
	&\leq 1+\frac{\lambda}{2}(T-t)+ \frac{1}{\lambda}\log \left\{ \sum_{(\sigma,\sigma_t) \in \F|_{x^t}}  g(\ldim(\F_t(\sigma,\sigma_t)),T-t) \exp\left\{- \lambda L_{t-1}(\sigma) \right\} \En_\epsilon  \exp\left\{- \lambda (1-\epsilon\sigma_t) \right\} \right\}
\end{align*}
for a Rademacher random variable $\epsilon\in\{\pm1\}$.
Now,
$$\En_\epsilon  \exp\left\{- \lambda (1-\epsilon\sigma_t) \right\} = e^{-\lambda}\En_\epsilon e^{\lambda\epsilon\sigma_t} \leq e^{-\lambda} e^{\lambda^2/2}$$
Substituting this into the above expression, we obtain an upper bound of
\begin{align*}
	\frac{\lambda}{2}(T-t+1)+ \frac{1}{\lambda}\log \left\{ \sum_{(\sigma,\sigma_t) \in \F|_{x^t}}  g(\ldim(\F_t(\sigma,\sigma_t)),T-t) \exp\left\{- \lambda L_{t-1}(\sigma) \right\}  \right\}
\end{align*}
which completes the proof of admissibility using the same combinatorial argument as in the earlier part of the proof.

\paragraph{Arriving at the Relaxation}

Finally, we show that the relaxation we use arises naturally as an upper bound on the sequential Rademacher complexity. Fix a tree $\x$. Let $\sigma \in \{\pm1\}^{t-1}$ be a sequence of signs. Observe that given history $x^t=(x_1,\ldots,x_t)$, the signs $\epsilon\in\{\pm1\}^{T-t}$, and a tree $\x$, the function class $\F$ takes on only a finite number of possible values $(\sigma, \sigma_t, \omega)$ on $(x^t,\x(\epsilon))$. Here, $\x(\epsilon)$ denotes the sequences of values along the path $\epsilon$. We have,
\begin{align*}
	\sup_{\x} \En_{\epsilon}  \sup_{f \in \F} \left\{ 2 \sum_{i=1}^{T-t} \epsilon_i f(\x_i(\epsilon))  - \sum_{i=1}^{t} |f(x_i)-y_i|\right\} 
	&= \sup_{\x} \En_{\epsilon} \max_{\sigma_t\in\{\pm1\}} \max_{(\sigma,\omega): (\sigma,\sigma_t,\omega)\in \F|_{(x^t, \x(\epsilon))}} \left\{ 2 \sum_{i=1}^{T-t} \epsilon_i \omega_i  - \sum_{i=1}^{t} |\sigma_i-y_i|\right\}\\
	&\leq\sup_{\x} \En_{\epsilon}  \max_{\sigma_t\in\{\pm1\}} \max_{\sigma: (\sigma,\sigma_t) \in \F|_{x^t}}\max_{\v\in V(\F(\sigma,\sigma_t),\x)} \left\{ 2 \sum_{i=1}^{T-t} \epsilon_i \v_i(\epsilon)  - \sum_{i=1}^{t} |\sigma_i-y_i|\right\}
\end{align*}
where $\F|_{(x^t,\x(\epsilon))}$ is the projection of $\F$ onto $(x^t, \x(\epsilon))$, $\F(\sigma,\sigma_t)=\{f\in\F: f(x^t)=(\sigma,\sigma_t)\}$, and $V(\F(\sigma,\sigma_t),\x)$ is the zero-cover of the set $\F(\sigma,\sigma_t)$ on the tree $\x$. We then have the following relaxation:
\begin{align*}
	\frac{1}{\lambda}\log \left( \sup_{\x} \En_{\epsilon}  \sum_{\sigma_t\in\{\pm1\}} \sum_{\sigma: (\sigma,\sigma_t) \in \F|_{x^t}}\sum_{\v\in V(\F(\sigma,\sigma_t),\x)} \exp\left\{ 2\lambda \sum_{i=1}^{T-t} \epsilon_i \v_i(\epsilon)  - \lambda L_{t}(\sigma,\sigma_t) \right\}\right)
\end{align*}
where $L_{t}(\sigma,\sigma_t) = \sum_{i=1}^{t} |\sigma_i-y_i|$.
The latter quantity can be factorized:
\begin{align*}
	&\frac{1}{\lambda}\log \left( \sup_{\x}  \sum_{\sigma_t\in\{\pm1\}} \sum_{\sigma: (\sigma,\sigma_t) \in \F|_{x^t}} \exp\left\{ - \lambda L_{t}(\sigma,\sigma_t) \right\} \En_{\epsilon} \sum_{\v\in V(\F(\sigma,\sigma_t),\x)} \exp\left\{ 2\lambda \sum_{i=1}^{T-t} \epsilon_i \v_i(\epsilon)  \right\}\right) \\
	&\leq \frac{1}{\lambda}\log \left(  \sup_{\x} \sum_{\sigma_t\in\{\pm1\}} \sum_{\sigma: (\sigma,\sigma_t) \in \F|_{x^t}} \exp\left\{ - \lambda L_{t}(\sigma,\sigma_t) \right\} \text{card}(V(\F(\sigma,\sigma_t),\x)) \exp\left\{ 2\lambda^2 (T-t) \right\}\right) \\
	&\leq \frac{1}{\lambda}\log \left(  \sum_{\sigma_t\in\{\pm1\}} \exp\left\{- \lambda |\sigma_t-y_t| \right\}  \sum_{\sigma: (\sigma,\sigma_t) \in \F|_{x^t}} g(\ldim(\F(\sigma,\sigma_t)),T-t) \exp\left\{- \lambda L_{t-1}(\sigma) \right\}  \right) + 2\lambda (T-t) \ .
\end{align*}
This concludes the derivation of the relaxation.

\end{proof}

\begin{proof}[\textbf{Proof of Lemma \ref{lem:fpl}}]
We first exhibit the proof for the convex loss case.  To show admissibility using the particular randomized strategy $q_t$ given in the lemma, we need to show that 
\begin{align*}
\sup_{x_t} \left\{ \Es{f \sim q_t}{\ell(f,x_t)} + \Relax{T}{\F}{x_1,\ldots,x_t}\right\} \le \Relax{T}{\F}{x_1,\ldots,x_{t-1}}
\end{align*}
The strategy $q_t$ proposed by the lemma is such that we first draw $x_{t+1},\ldots,x_T \sim D$ and $\epsilon_{t+1},\ldots \epsilon_T$ Rademacher random variables, and then based on this sample pick $f_t=f_t(x_{t+1:T},\epsilon_{t+1:T})$ as in \eqref{eq:def_general_fpl}. Hence,
\begin{align*}
\sup_{x_t} &\left\{ \Es{f \sim q_t}{\ell(f,x_t)} + \Relax{T}{\F}{x_1,\ldots,x_t}\right\} \\
& = \sup_{x_t} \left\{ \Eunder{x_{t+1:T}}{\epsilon_{t+1:T}} \ell(f_t,x) + \Eunder{x_{t+1:T}}{\epsilon_{t+1:T}} \sup_{f \in \F} \left[ C \sum_{i=t+1}^T \epsilon_i \ell(f,x_i) - L_t(f) \right] \right\}\\
& \le \Eunder{x_{t+1:T}}{\epsilon_{t+1:T}} \sup_{x_t} \left\{ \ell(f_t,x) +  \sup_{f \in \F} \left[ C \sum_{i=t+1}^T \epsilon_i \ell(f,x_i) - L_t(f) \right]\right\}
\end{align*}
where $L_t(f) = \sum_{i=1}^t \ell(f,x_i)$.
Observe that our strategy ``matched the randomness'' arising from the relaxation! Now, with $f_t$ defined as 
\begin{align*}
f_t = \argmin{g \in \F} \sup_{x_t\in\X} \left\{ \ell(g,x_t) + \sup_{f \in \F} \left[ C \sum_{i=t+1}^T \epsilon_i \ell(f,x_i) - L_t(f)  \right] \right\}
\end{align*}
for any given $x_{t+1:T},\epsilon_{t+1:T}$, we have
\begin{align*}
\sup_{x_t} & \left\{ \ell(f_t,x_t) + \sup_{f \in \F} \left[ C \sum_{i=t+1}^T \epsilon_i \ell(f,x_i) - L_t(f)  \right] \right\}  = \inf_{g \in \F} \sup_{x_t} \left\{ \ell(g,x_t) + \sup_{f \in \F} \left[ C \sum_{i=t+1}^T \epsilon_i \ell(f,x_i) - L_t(f)  \right] \right\}
\end{align*}
We can conclude that for this choice of $q_t$, 
\begin{align*}
\sup_{x_t} &\left\{ \Es{f \sim q_t}{\ell(f,x_t)} + \Relax{T}{\F}{x_1,\ldots,x_t}\right\}  \le \Eunder{x_{t+1:T}}{\epsilon_{t+1:T}} \inf_{g \in \F} \sup_{x_t} \left\{\ell(g,x_t) +  \sup_{f \in \F} \left[ C \sum_{i=t+1}^T \epsilon_i \ell(f,x_i) - L_t(f) \right]\right\}\\
& =  \Eunder{x_{t+1:T}}{\epsilon_{t+1:T}} \inf_{g \in \F} \sup_{p_t \in \Delta(\X)} \Es{x_t \sim p_t}{\ell(g,x_t) +  \sup_{f \in \F} \left[ C \sum_{i=t+1}^T \epsilon_i \ell(f,x_i) - L_t(f) \right]}\\
& = \Eunder{x_{t+1:T}}{\epsilon_{t+1:T}}\sup_{p \in \Delta(\X) } \inf_{g \in \F}  \left\{\Es{x_t \sim p}{\ell(g,x_t)} +  \Es{x_t \sim p}{\sup_{f \in \F} C \sum_{i=t+1}^T \epsilon_i \ell(f,x_i) - L_t(f)} \right\}
\end{align*}
In the last step we appealed to the minimax theorem which holds as loss is convex in $g$ and $\F$ is a compact convex set and the term in the expectation is linear in $p_t$, as it is an expectation. The last expression can be written as
\begin{align*}
&\Eunder{x_{t+1:T}}{\epsilon_{t+1:T}} \sup_{p \in \Delta(\X) }\En_{x_t \sim p}   \sup_{f \in \F} \left[ C \sum_{i=t+1}^T \epsilon_i \ell(f,x_i) - L_{t-1}(f) + \inf_{g \in \F}  \Es{x_t \sim p}{\ell(g,x_t)} - \ell(f,x_t) \right] \\
& \le \Eunder{x_{t+1:T}}{\epsilon_{t+1:T}} \sup_{p \in \Delta(\X) }\En_{x_t \sim p}   \sup_{f \in \F} \left[ C \sum_{i=t+1}^T \epsilon_i \ell(f,x_i) - L_{t-1}(f) +   \Es{x_t \sim p}{\ell(f,x_t)} - \ell(f,x_t) \right] \\
& \le\Eunder{x_{t+1:T}}{\epsilon_{t+1:T}} \En_{x_t \sim D}\En_{\epsilon_t}   \sup_{f \in \F} \left[ C \sum_{i=t+1}^T \epsilon_i \ell(f,x_i) - L_{t-1}(f) +   C \epsilon_t \ell(f,x_t) \right]\\
& = \Relax{T}{\F}{x_1,\ldots,x_{t-1}}
\end{align*}
Last inequality is by Assumption \ref{asm:fpl}, using which we can replace a draw from supremum over distributions by a draw from the ``equivalently bad'' fixed distribution $D$ by suffering an extra factor of $C$ multiplied to that random instance.

The key step where we needed convexity was to use minimax theorem to swap infimum and supremum inside the expectation. In general the minimax theorem need not hold. In the non-convex scenario this is the reason we add the extra randomization through $\hat{q}_t$. The non-convex case has a similar proof except that we have expectation w.r.t. $\hat{q}_t$ extra on each round which essentially convexifies our loss and thus allows us to appeal to the minimax theorem.
\end{proof}

\begin{proof}[\textbf{Proof of Lemma~\ref{lem:fpl-l1linfty}}]
	Let $w\in\reals^N$ be arbitrary. Throughout this proof, let $\epsilon\in\{\pm1\}$ be a single Rademacher random variable, rather than a vector. To prove \eqref{eq:assumption_fpl_linear_simpler}, observe that
	\begin{align*}
		\sup_{p \in \Delta(\X)} \Eunderone{x_t \sim p} \norm{ w + \Eunderone{x \sim p} [x] - x_t }_\infty &\le \sup_{p \in \Delta(\X)} \Eunderone{x,x' \sim p} \norm{ w + x' - x }_\infty\\
		&= \sup_{p \in \Delta(\X)} \Eunderone{x,x' \sim p}\En_\epsilon \norm{ w + \epsilon(x' - x) }_\infty\\
		&\leq \sup_{x,x'\in \X} \En_\epsilon \norm{ w + \epsilon(x' - x) }_\infty\\
		&\leq \sup_{x'\in \X} \En_\epsilon \norm{ w/2 + \epsilon x'}_\infty+ \sup_{x\in \X} \En_\epsilon \norm{ w/2 - \epsilon x}_\infty\\
		&= \sup_{x\in \X} \En_\epsilon \max_{i\in[N]} \left| w_i + 2\epsilon x_i\right|
	\end{align*}	
	The supremum over $x\in\X$ is achieved at the vertices of $\X$ since the expected maximum is a convex function. It remains to prove the identity
	\begin{align}
		\label{eq:coord_max}
		\max_{x\in \{\pm1\}^N} \En_\epsilon \max_{i\in[N]} \left| w_i + 2\epsilon x_i\right| \leq \Eunderone{x \sim D }\Eunderone{\epsilon}\max_{i\in[N]} \left| w_i+  6\epsilon x_i  \right|
	\end{align}
	Let $i^*=\argmax{i} |w_i|$ and $j^*=\argmax{i\neq i^*} |w_i|$ be the coordinates with largest and second-largest magnitude. If $|w_{i^*}|-|w_{j^*}| \geq 4$, the statement follows since, for any $x\in\{\pm1\}^N$ and $\epsilon\in\{\pm1\}$, 
	$$ \max_{i\neq i^*} \left| w_i + 2\epsilon x_i\right|\leq  \max_{i\neq i^*}\left| w_{i} \right| + 2 \leq   \left| w_{i^*} \right| -2 \leq |w_{i^*}+2\epsilon x_{i^*}|,$$ and thus
	$$\max_{x\in \{\pm1\}^N} \En_\epsilon \max_{i\in[N]} \left| w_i + 2\epsilon x_i\right| = \max_{x\in \{\pm1\}^N} \En_\epsilon  \left| w_{i^*} + 2\epsilon x_{i^*}\right| = |w_{i^*}| = \En_{x,\epsilon} |w_{i^*}+6\epsilon x_{i^*}| \leq \En_{x,\epsilon} \max_{i}|w_{i}+6\epsilon x_{i}|.$$
	 It remains to consider the case when $|w_{i^*}|- |w_{j^*}| < 4$. We have that
	\begin{align}
		\En_{x ,\epsilon}\max_{i\in[N]} \left| w_i+  6\epsilon  x_i  \right| \geq \En_{x ,\epsilon}\max_{i\in\{i^*,j^*\}} \left| w_i+  6\epsilon  x_i  \right| &\geq \frac{1}{2} (|w_{i^*}|+6)+ \frac{1}{4} (|w_{i^*}|-6) + \frac{1}{4} (|w_{j^*}|+6) \geq |w_{i^*}|+2 \\
		&\geq \max_{x\in \{\pm1\}^N} \En_\epsilon \max_{i\in[N]} \left| w_i + 2\epsilon  x_i\right|,
	\end{align}
	where $1/2$ is the probability that $\epsilon x_{i^*} = sign(w_{i^*})$, the second event of probability $1/4$ is the event that $\epsilon x_{i^*} \neq sign(w_{i^*})$ and $\epsilon x_{j^*} \neq sign(w_{j^*})$, while the third event of probability $1/4$ is that $\epsilon x_{i^*} \neq sign(w_{i^*})$ and $\epsilon x_{j^*} = sign(w_{j^*})$. 
\end{proof}

\begin{proof}[\textbf{Proof of Lemma \ref{lem:genlinffpl}}]
	Let $w\in\reals^N$ be arbitrary. Just as in the proof of Lemma~\ref{lem:fpl-l1linfty}, we need to show
	\begin{align}
		\label{eq:coord_max2}
		\max_{x\in \{\pm1\}^N} \En_\epsilon \max_{i\in[N]} \left| w_i + 2\epsilon  x_i\right| \leq \Eunderone{x \sim D }\Eunderone{\epsilon}\max_{i\in[N]} \left| w_i+  C \epsilon  x_i  \right|
	\end{align}
	Let $i^*=\argmax{i} |w_i|$ and $j^*=\argmax{i\neq i^*} |w_i|$ be the coordinates with largest and second-largest magnitude. If $|w_{i^*}|-|w_{j^*}| \geq 4$, the statement follows
	exactly as in Lemma~\ref{lem:fpl-l1linfty}.
It remains to consider the case when $|w_{i^*}|- |w_{j^*}| < 4$. In this case first note that,
$$
\max_{x\in \{\pm1\}^N} \En_\epsilon \max_{i\in[N]} \left| w_i + 2\epsilon  x_i\right| \le |w_{i^*}|+2 
$$
On the other hand, since the distribution we consider is symmetric, with probability $1/2$ its sign is negative and with remaining probability positive. Define $\sigma_{i^*} = \sign(x_{i^*})$, $\sigma_{j^*} = \sign(x_{j^*})$, $\tau_{i^*} = \sign(w_{i^*})$, and $\tau_{j^*} = \sign(w_{j^*})$. Since each coordinate is drawn i.i.d., using conditional expectations we have,

\begin{align*}
 &\En_{x,\epsilon} \max_{i}|w_{i}+C \epsilon x_{i}|  =  \En_{x} \max_{i}|w_{i}+C  x_{i}| \\
 & \ge \frac{\Es{x}{|w_{i^*} + C x_{i^*}|\ \middle|\ \sigma_{i^*} =\tau_{i^*}}}{2}  + \frac{\Es{x}{|w_{j^*} + C x_{j^*}|\ \middle|\ \sigma_{i^*} \ne \tau_{i^*}, \sigma_{j^*} = \tau_{j^*}}}{4} + \frac{\E{|w_{i^*} + C x_{i^*}|\ \middle|\ \sigma_{i^*} \ne \tau_{i^*}, \sigma_{j^*} \ne \tau_{j^*}}}{4}\\
  &\ge \frac{\Es{x}{|w_{i^*}| + C |x_{i^*}|\ \middle|\ \sigma_{i^*} =\tau_{i^*}}}{2} + \frac{\Es{x}{|w_{j^*}| + C |x_{j^*}|\ \middle|\ \sigma_{i^*} \ne \tau_{i^*}, \sigma_{j^*} = \tau_{j^*}}}{4} + \frac{\E{|w_{i^*}| - C |x_{i^*}|\ \middle|\ \sigma_{i^*} \ne \tau_{i^*}, \sigma_{j^*} \ne \tau_{j^*}}}{4}\\
& = \frac{\E{|w_{i^*}| + C |x_{i^*}|\ \middle|\ \sigma_{i^*} =\tau_{i^*}}}{2}  + \frac{\E{|w_{j^*}| + C |x_{j^*}| \  \middle| \  \sigma_{j^*} = \tau_{j^*}}}{4} + \frac{\E{|w_{i^*}| - C |x_{i^*}| \ \middle| \ \sigma_{i^*} \ne \tau_{i^*}}}{4}\\
& = \frac{|w_{i^*}| + C \E{|x_{i^*}| \ \middle|\  \sigma_{i^*} =\tau_{i^*}}}{2}  + \frac{|w_{j^*}| + C \E{|x_{j^*}| \ ~\middle|~ \  \sigma_{j^*} = \tau_{j^*}}}{4}  + \frac{|w_{i^*}| - C \E{|x_{i^*}| \ ~\middle|~\  \sigma_{i^*} \ne \tau_{i^*}}}{4}\\
& = \frac{2 |w_{i^*}| + |w_{j^*}| + 3 C \E{|x_{i^*}| \ \middle|\  \sigma_{i^*} =\tau_{i^*}}}{4}   + \frac{|w_{i^*}| - C \E{|x_{i^*}| \ \middle|\  \sigma_{i^*} \ne \tau_{i^*}}}{4}\\
& = \frac{3 |w_{i^*}| + |w_{j^*}| + 2 C \E{|x_{i^*}| \ \middle|\  \sigma_{i^*} =\tau_{i^*}}}{4} 
\end{align*}
Now since we are in the case when $|w_{i^*}|- |w_{j^*}| < 4$ we see that
$$
 \En_{x,\epsilon} \max_{i}|w_{i}+C \epsilon x_{i}| \ge \frac{3 |w_{i^*}| + |w_{j^*}| + 2 C \E{|x_{i^*}| ~\middle|~ \sigma_{i^*} =\tau_{i^*}}}{4}  \ge \frac{4 |w_{i^*}| + 2 C \E{|x_{i^*}| ~\middle|~ \sigma_{i^*} =\tau_{i^*}} - 4}{4} 
$$
On the other hand, as we already argued,
\begin{align*}
\max_{x\in \{\pm1\}^N} \En_\epsilon \max_{i\in[N]} \left| w_i + 2\epsilon  x_i\right| \le |w_{i^*}|+2 
\end{align*}
Hence, as long as 
\begin{align*}
\frac{ C\ \E{|x_{i^*}|\ \middle|\ \sigma_{i^*} =\tau_{i^*}} - 2}{2} \ge 2 
\end{align*}
or, in other words, as  long as
$$
C   \ge 6 / \E{|x_i|\ \middle|\ \sign(x_i) =\sign(w_i)} = 6/ \Es{x}{|x|}~,
$$ 
we have that
$$
\max_{x\in \{\pm1\}^N} \En_\epsilon \max_{i\in[N]} \left| w_i + 2\epsilon  x_i\right|  \le  \En_{x,\epsilon} \max_{i}|w_{i}+C \epsilon x_{i}| \ .
$$
This concludes the proof.
\end{proof}

\begin{lemma}\label{lem:l1linffpl}
Consider the case when $\X$ is the $\ell_\infty^N$ ball and $\F$ is the $\ell_1^N$ unit ball. Let $f^* =  \argmin{f \in \F} \ip{f}{R}$, then for any random vector $R$, 
\begin{align*}
\Es{R}{ \sup_{x \in \X} \left\{\ip{f^*}{x} + \norm{R + x}_\infty \right\}} & \le  \Es{R}{\inf_{f \in \F} \sup_{x} \left\{\ip{f}{x} + \norm{R + x}_\infty \right\}} + 4\ \P\left( \norm{R}_{\infty}  \le 4 \right)
\end{align*} 
\end{lemma}
\begin{proof}
Let $f^* =  \argmin{f \in \F} \ip{f}{R}$. We start by noting that for any $f' \in \F$,
\begin{align*}
\sup_{x \in \X} \left\{\ip{f'}{x} + \norm{R + x}_\infty \right\} & = \sup_{x \in \X} \left\{\ip{f'}{x} + \sup_{f \in \F}\ip{f}{R + x}\right\}\\
& = \sup_{f \in \F} \sup_{x \in \X} \left\{\ip{f'}{x} + \ip{f}{R + x}\right\}\\
& = \sup_{f \in \F}  \left\{\sup_{x \in \X}\ip{f' + f}{x} + \ip{f}{R }\right\}\\
& = \sup_{f \in \F}  \left\{\norm{f' + f}_{1} + \ip{f}{R }\right\}
\end{align*}
Hence note that 
\begin{align} \label{eq:lowR}
\inf_{f' \in \F}\sup_{x \in \X} \left\{\ip{f'}{x} + \norm{R + x}_\infty \right\} & =  \inf_{f' \in \F}\sup_{f \in \F}  \left\{\norm{f' + f}_1 + \ip{f}{R }\right\} \\
&\ge  \inf_{f' \in \F}\left\{\norm{f' - f^*}_1 - \ip{f^*}{R }\right\} \ge  \inf_{f' \in \F} \left\{\norm{f' - f^*}_1 + \norm{R }_{\infty}\right\} = \norm{R}_\infty
\end{align}
On  the other hand note that, $f^*$ is the vertex of the $\ell_1$ ball (any one which given by  $\argmin{i \in [d]} |R[i]|$ with sign opposite as sign of $R[i]$ on that vertex). Since the $\ell_1$ ball is the convex hull of the $2d$ vertices, any vector $f \in \F$ can be written as $f = \alpha h-\beta f^*$ 
some $h \in \F$ such that $\norm{h}_1 = 1$ and 
$\ip{h}{R} = 0$ (which means that $h$ is $0$ on the maximal co-ordinate of $R$ specified by $f^*$) and for some $\beta \in [-1,1]$, $\alpha \in [0,1]$ s.t. $\norm{\alpha h - \beta f^*}_1 \le 1$. Further note that the constraint on $\alpha, \beta$ imposed by requiring that $\norm{\alpha h - \beta f^*}_1 \le 1$ can be written as $\alpha + |\beta| \le 1$. Hence,
\begin{align*}
 \sup_{x \in \X} \left\{\ip{f^*}{x} + \norm{R + x}_\infty \right\} & = \sup_{f \in \F}  \left\{\norm{f^* + f}_1 + \ip{f}{R}\right\}\\
& = \sup_{\alpha \in [0,1]}\sup_{h \perp f^*, \norm{h}_1 = 1} \sup_{\beta \in [-1,1] , \norm{\alpha h - \beta f^*}_1 \le 1}  \left\{\norm{ (1 - \beta)f^* + \alpha h}_1 + \beta \ip{f^*}{R } + \alpha \ip{h}{R} \right\}\\
& = \sup_{\alpha \in [0,1]}\sup_{h \perp f^*, \norm{h}_1 = 1} \sup_{\beta \in [-1,1] , \norm{\alpha h - \beta f^*}_1 \le 1} \left\{ |1 - \beta| \norm{ f^*}_1 + \alpha \norm{h}_1 + \beta \norm{R}_{\infty}  \right\}\\
& = \sup_{\alpha \in [0,1]} \sup_{\beta \in [-1,1] : |\beta| + \alpha \le 1}\left\{ |1 - \beta|  + \alpha  + \beta \norm{R}_{\infty}  \right\}\\
& \le \sup_{\beta \in [-1,1] }\left\{ |1 - \beta|  + 1 - |\beta|  + \beta \norm{R}_{\infty}  \right\}\\
& \le \sup_{\beta \in [-1,1] }\left\{ 2|1 - \beta|  + \beta \norm{R}_{\infty}  \right\}\\
& = \sup_{\beta \in \{-1,1\} }\left\{ 2|1 - \beta|  + \beta \norm{R}_{\infty}  \right\}\\
& = \max\left\{  \norm{R}_{\infty}  , 4 - \norm{R}_{\infty}  \right\}\\
& \le \norm{R}_{\infty}  +  4\ \ind{ \norm{R}_{\infty}  \le 4}
\end{align*}
Hence combining with equation \ref{eq:lowR} we can conclude that
\begin{align*}
\Es{R}{ \sup_{x} \left\{\ip{f^*}{x} + \norm{R + x}_\infty \right\}} & \le \Es{R}{\inf_{f \in \F} \sup_{x} \left\{\ip{f}{x} + \norm{R + x}_\infty \right\}} + 4\ \Es{R}{\ind{ \norm{R}_{\infty}  \le 4}}\\
& = \Es{R}{\inf_{f \in \F} \sup_{x} \left\{\ip{f}{x} + \norm{R + x}_\infty\right\}} + 4\ \P\left( \norm{R}_{\infty}  \le 4 \right)
\end{align*}
\end{proof}

\begin{proof}[\textbf{Proof of Lemma \ref{lem:l1fplmain}}]
On any round $t$, the algorithm draws $\epsilon_{t+1} , \ldots, \epsilon_T$ and $x_{t+1},\ldots,x_T \sim D^N$ and plays
$$
f_t = \argmin{f \in \F}\ip{f}{\sum_{i=1}^{t-1} x_i - C \sum_{i=t+1}^T  x_i } 
$$
We shall show that this randomized algorithm is (almost) admissible w.r.t. the relaxation (with some small additional term at each step). We define the relaxation as
$$
	\Relax{T}{\F}{x_1,\ldots,x_t} = \Es{x_{t+1},\ldots x_T \sim D}{ \norm{\sum_{i=1}^t x_i  - C \sum_{i=t+1}^T x_i}_\infty}
$$
Proceeding just as in the proof of Lemma \ref{lem:fpl} note that, for our randomized strategy, 
\begin{align}
\sup_{x} &\left\{ \Es{f \sim q_t}{\ip{f}{x}} + \Relax{T}{\F}{x_1,\ldots,x_t}\right\} \notag\\
& = \sup_{x} \left\{ \Es{x_{t+1:T} \sim D^N}{ \ip{f_t}{x}} + \Es{x_{t+1:T} \sim D^N}{\norm{ \sum_{i=1}^{t-1} x_i + x - C \sum_{i=t+1}^T x_i}_\infty } \right\}\notag\\
& \le \Es{x_{t+1:T} \sim D^N}{ \sup_{x} \left\{ \ip{f_t}{x} +  \norm{  \sum_{i=1}^{t-1} x_i  + x - C \sum_{i=t+1}^T x_i  }_\infty \right\}}\label{eq:fp1}
\end{align}

In view of Lemma \ref{lem:l1linffpl} (with $R = \sum_{i=1}^{t-1} x_i - C \sum_{i=t+1}^T \epsilon_i x_i$) we conclude that
\begin{align*}
& \Es{x_{t+1},\ldots,x_{T}}{ \sup_{x \in \X} \left\{\ip{f_t}{x} + \norm{\sum_{i=1}^{t-1} x_i - C \sum_{i=t+1}^T  x_i + x}_\infty \right\}} \\
&~~~~~~~~~~ \le  \Es{x_{t+1},\ldots,x_{T}}{\inf_{f \in \F} \sup_{x} \left\{\ip{f}{x} + \norm{\sum_{i=1}^{t-1} x_i - C \sum_{i=t+1}^T x_i + x}_\infty \right\}} + 4\ \P\left( \norm{\sum_{i=1}^{t-1} x_i - C \sum_{i=t+1}^T  x_i}_{\infty}  \le 4 \right)\\
&~~~~~~~~~~ =  \Es{x_{t+1},\ldots,x_{T}}{\sup_{x} \left\{\ip{f^*_t}{x} + \norm{\sum_{i=1}^{t-1} x_i - C \sum_{i=t+1}^T x_i + x}_\infty \right\}} + 4\ \P\left( \norm{\sum_{i=1}^{t-1} x_i - C \sum_{i=t+1}^T x_i}_{\infty}  \le 4 \right)
\end{align*}
where
$$
f^*_t = \argmin{f \in \F}\sup_{x} \left\{\ip{f}{x} + \norm{\sum_{i=1}^{t-1} x_i - C \sum_{i=t+1}^T x_i + x}_\infty \right\}
$$ 
Combining with Equation \eqref{eq:fp1} we conclude that 
\begin{align*}
\sup_{x} &\left\{ \Es{f \sim q_t}{\ip{f}{x}} + \Relax{T}{\F}{x_1,\ldots,x_t}\right\} \\
& \le \Es{x_{t+1},\ldots,x_{T}}{\sup_{x} \left\{\ip{f^*_t}{x} + \norm{\sum_{i=1}^{t-1} x_i - C \sum_{i=t+1}^T x_i + x}_\infty \right\}} + 4\ \P\left( \norm{\sum_{i=1}^{t-1} x_i - C \sum_{i=t+1}^T x_i}_{\infty}  \le 4 \right)
\end{align*}

Now, since
\begin{align*}
4\ \P\left( \norm{\sum_{i=1}^{t-1} x_i - C \sum_{i=t+1}^T  x_i}_{\infty}  \le 4 \right) \leq 4\ \P\left( C \norm{ \sum_{i=t+1}^T  x_i}_{\infty}  \le 4 \right)
\le 4\ \P_{y_{t+1},\ldots,y_T \sim D}\left( C \left| \sum_{i=t+1}^T  y_i\right|  \le 4 \right)
\end{align*}
we have
\begin{align}\label{eq:fplrelaxation2}
\sup_{x} &\left\{ \Es{f \sim q_t}{\ip{f}{x}} + \Relax{T}{\F}{x_1,\ldots,x_t}\right\} \\
& \le \Es{x_{t+1},\ldots,x_{T}}{\sup_{x} \left\{\ip{f^*_t}{x} + \norm{\sum_{i=1}^{t-1} x_i - C \sum_{i=t+1}^T x_i + x}_\infty \right\}} + 4\ \P_{y_{t+1},\ldots,y_T \sim D}\left( C \left| \sum_{i=t+1}^T  y_i\right|  \le 4 \right)
\end{align}
In view of Lemma \ref{lem:genlinffpl},  Assumption~\ref{asm:fpl-linear}  is satisfied by $D^N$ with constant $C$. Further in the proof of Lemma \ref{lem:fpl}  we already showed that whenever Assumption~\ref{asm:fpl-linear} is satisfied, the randomized strategy specified by  $f^*_t$ is admissible. More specifically we showed that 
$$
\Es{x_{t+1},\ldots,x_{T}}{\sup_{x} \left\{\ip{f^*_t}{x} + \norm{\sum_{i=1}^{t-1} x_i - C \sum_{i=t+1}^T x_i + x}_\infty \right\}} \le \Relax{T}{F}{x_1,\ldots,x_{t-1}}
$$
and so using this in Equation \eqref{eq:fplrelaxation2} we conclude that for the randomized strategy in the statement of the lemma,
\begin{align*}
\sup_{x} &\left\{ \Es{f \sim q_t}{\ip{f}{x}} + \Relax{T}{\F}{x_1,\ldots,x_t}\right\} \\
& \le \Relax{T}{F}{x_1,\ldots,x_{t-1}} + 4\ \P_{y_{t+1},\ldots,y_T \sim D}\left( C \left| \sum_{i=t+1}^T  y_i\right|  \le 4 \right)
\end{align*}
Or in other words the randomized strategy proposed is admissible with an additional additive factor of $4\ \P_{y_{t+1},\ldots,y_T \sim D}\left( C \left| \sum_{i=t+1}^T  y_i\right|  \le 4 \right)$ at each time step $t$.
Hence by Proposition \ref{prop:main} we have that for the randomized algorithm specified in the lemma,
\begin{align*}
\E{\Reg_T} & \le \Rel{T}{F} + 4 \sum_{t=1}^T  \P_{y_{t+1},\ldots,y_T \sim D}\left( C \left| \sum_{i=t+1}^T  y_i\right|  \le 4 \right) \\
& = C\ \Es{x_1,\ldots,x_T \sim D^N}{\norm{\sum_{t=1}^T  x_t}_\infty}  + 4 \sum_{t=1}^T  \P_{y_{t+1},\ldots,y_T \sim D}\left( C \left| \sum_{i=t+1}^T  y_i\right|  \le 4 \right)
\end{align*}
This concludes the proof.
\end{proof}

\begin{proof}[\textbf{Proof of Lemma \ref{lem:fpl-l2l2}}]
Instead of using $C = 4 \sqrt{2}$ and drawing uniformly from surface of unit sphere we can equivalently think of the constant as being $1$ and drawing uniformly from surface of sphere of 	radius $4 \sqrt{2}$. Let $\norm{\cdot}$ stand for the Euclidean norm. To prove \eqref{eq:assumption_fpl_linear_simpler}, first observe that
	\begin{align}
		\sup_{p \in \Delta(\X)} \Eunderone{x_t \sim p} \norm{ w + \Eunderone{x \sim p} [x] - x_t } \le \sup_{x \in \X} \Eunderone{\epsilon} \norm{ w + 2\epsilon x }
	\end{align}	
	for any $w\in B$. Further, using Jensen's inequality
	$$\sup_{x \in \X} \Eunderone{\epsilon} \norm{ w + 2\epsilon x }\leq \sup_{x \in \X} \sqrt{\Eunderone{\epsilon} \norm{ w + 2\epsilon x }^2}\leq \sup_{x \in \X} \sqrt{ \norm{ w}^2 + \Eunderone{\epsilon}\norm{2\epsilon x }^2}= \sqrt{ \norm{ w}^2 +4}$$
	To prove the lemma, it is then enough to show that for $r=4\sqrt{2}$
	\begin{align}
		\label{eq:wanted_ineq}
		\En_{x\sim D}\norm{w+rx} \geq \sqrt{ \norm{ w}^2 +4}
	\end{align}
	for any $w$, where we omitted $\epsilon$ since $D$ is symmetric. This fact can be proved with the following geometric argument. 

	We define quadruplets $(w+z_1,w+z_2,w-z_1,w-z_2)$ of points on the sphere of radius $r$. Each quadruplets will have the property that 
	\begin{align}
		\label{eq:wanted_ineq2}
		\frac{\norm{w+z_1}+\norm{w+z_2}+\norm{w-z_1}+\norm{w-z_2}}{4}\geq \sqrt{ \norm{ w}^2 +4}
	\end{align}
	for any $w$.
	We then argue that the uniform distribution can be decomposed into these quadruplets such that each point on the sphere occurs in only one quadruplet (except for a measure zero set when $z_1$ is aligned with $-w$), thus concluding that \eqref{eq:wanted_ineq} holds true.

	\begin{figure}[htbp]
		\centering
			\includegraphics[height=1.3in]{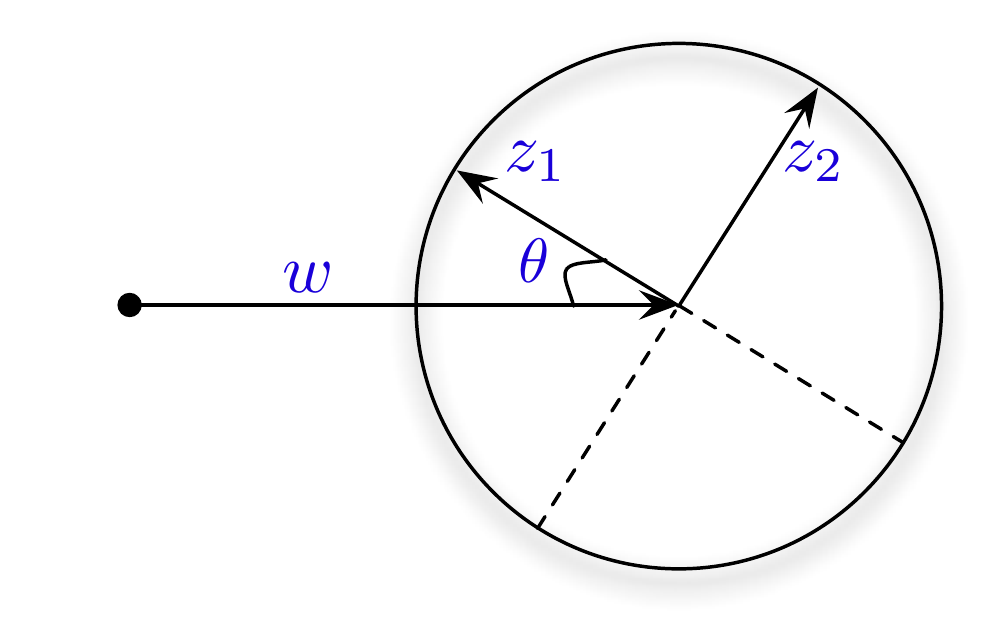}
		\caption{The two-dimensional construction for the proof of Lemma~\ref{lem:fpl-l2l2}. }
		\label{fig:graphics_fpl-l2-l2}
	\end{figure}

	Pick any direction $w^\perp$ perpendicular to $w$. A quadruplet is defined by perpendicular vectors $z_1$ and $z_2$ which have length $r$ and which lie in the plane spanned by $w,w^\perp$. Let $\theta$ be the angle between $-w$ and $z_1$. Since we are now dealing with a two dimensional plane spanned by $w$ and $w^\perp$, we may as well assume that $w$ is aligned with the positive $x$-axis, as in Figure~\ref{fig:graphics_fpl-l2-l2}. We write $w$ for $\|w\|$. The coordinates of the quadruplet are  
	$$(w-r\cos(\theta), r\sin(\theta)),~~ (w+r\cos(\theta), -r\sin(\theta)),~~ (w+r\sin(\theta), r\cos(\theta)),~~ (w-r\sin(\theta), -r\cos(\theta))$$
	For brevity, let $s=\sin(\theta), c=\cos(\theta)$. The desired inequality \eqref{eq:wanted_ineq2} then reads 
	$$\sqrt{w^2-8wc+r^2}+\sqrt{w^2+8wc+r^2}+\sqrt{w^2+8ws+r^2}+\sqrt{w^2-8ws+r^2}\geq 4\sqrt{w^2+4}$$
	To prove that this inequality holds, we square both sides, keeping in mind that the terms are non-negative. The sum of four squares on the left hand side gives $4w^2+4r^2$. For the six cross terms, we can pass to a lower bound by replacing $r^2$ in each square root by $r^2c^2$ or $r^2s^2$, whichever completes the square. Then observe that
	$$ |w+rs|\cdot|w-rs|+|w+rc|\cdot|w-rc| = 2w^2-r^2$$
	while the other four cross terms
	\begin{align*}
			(|w+rs|\cdot|w-rc|+|w+rs|\cdot|w+rc|)+(|w-rs|\cdot|w+rc|+|w-rs|\cdot|w-rc|)  \geq|w+rs|\cdot 2w+|w-rs|\cdot 2w \geq 4w^2
		\end{align*}
	Doubling the cross terms gives a contribution of $2(6w^2-r^2)$, while the sum of squares yielded  $4w^2+4r^2$. The desired inequality is satisfied as long as $16w^2+2r^2 \geq 16(w^2+4)$, or $r\geq 4\sqrt{2}$.

\end{proof}

\begin{proof}[\textbf{Proof of Lemma \ref{lem:l2fplmain}}]
By Lemma \ref{lem:fpl-l2l2}, Assumption~\ref{asm:fpl-linear} is satisfied by distribution $D$ with constant $C = 4 \sqrt{2}$. Hence by Lemma \ref{lem:genlinffpl} we can conclude that for the randomized algorithm which at round $t$ freshly draws  $x_{t+1},\ldots,x_T \sim D$ and picks 
$$
f^*_t = \argmin{f \in \F}\sup_{x \in \X} \left\{\ip{f}{x} + \norm{- \sum_{i=1}^{t-1} x_i + 4 \sqrt{2} \sum_{i=t+1}^T x_i  - x}_2 \right\}
$$
(we dropped the $\epsilon$'s as the distribution is symmetric to start with) the expected regret is bounded as
$$
\E{\Reg_T} \le 4 \sqrt{2}\ \Es{x_1,\ldots,x_T \sim D}{\norm{\sum_{t=1}^T x_t}_2} \le 4 \sqrt{2 T} 
$$
We claim that the strategy specified in the lemma that chooses 
$$
f_t = \frac{- \sum_{i=1}^{t-1} x_i + 4 \sqrt{2} \sum_{i=t+1}^T x_i }{\sqrt{\norm{- \sum_{i=1}^{t-1} x_i + 4 \sqrt{2} \sum_{i=t+1}^T \epsilon_i x_i }_2^2 + 1}}
$$ 
is the same as choosing $f^*_t$. To see this let us start by defining 
$$
\bar{x}_t = - \sum_{i=1}^{t-1} x_i + 4 \sqrt{2} \sum_{i=t+1}^T x_i  
$$
Now note that 
\begin{align*}
f^*_t = \argmin{f \in \F}\sup_{x \in \X} \left\{\ip{f}{x} + \norm{- \sum_{i=1}^{t-1} x_i + 4 \sqrt{2} \sum_{i=t+1}^T x_i  - x}_2 \right\} & = \argmin{f \in \F}\sup_{x \in \X} \left\{\ip{f}{x} + \norm{\bar{x}_t  - x}_2 \right\}\\
& = \argmin{f \in \F}\sup_{x \in \X} \left\{\ip{f}{x} + \sqrt{\norm{\bar{x}_t  - x}_2^2} \right\}\\
& = \argmin{f \in \F}\sup_{x : \norm{x}_2 \le 1} \left\{\ip{f}{x} + \sqrt{\norm{\bar{x}_t}^2  - 2 \ip{\bar{x}_t}{x} +  \norm{x}_2^2} \right\}\\
& = \argmin{f \in \F}\sup_{x : \norm{x}_2 = 1} \left\{\ip{f}{x} + \sqrt{\norm{\bar{x}_t}^2  - 2 \ip{\bar{x}_t}{x} +  1} \right\}
\end{align*}
However this argmin calculation is identical to the one in the proof of Proposition~\ref{prop:mirror_relax}  (with $C = 1$ and $T-t = 0$) and the solution is given by
$$
f^*_t = f_t = \frac{- \sum_{i=1}^{t-1} x_i + 4 \sqrt{2} \sum_{i=t+1}^T x_i }{\sqrt{\norm{- \sum_{i=1}^{t-1} x_i + 4 \sqrt{2} \sum_{i=t+1}^T \epsilon_i x_i }_2^2 + 1}}
$$
Thus we conclude the proof.
\end{proof}

\begin{proof}[\textbf{Proof of Lemma~\ref{lem:fpl2}}]
We first prove the statement for the convex case. 
To show admissibility using the particular randomized strategy given in the lemma, we need to show that for the randomized strategy specified by $q_t$, 
\begin{align*}
\sup_{y_t} \left\{ \Es{\hat{y}_t \sim q_t}{\ell(\hat{y}_t,y_t)} + \Relax{T}{\F}{(x_1,y_1),\ldots,(x_t,y_t)}\right\} \le \Relax{T}{\F}{(x_1,y_1),\ldots,(x_{t-1},y_{t-1})}
\end{align*}
for any $x_t$. The strategy $q_t$ proposed by the lemma is such that we first draw $(x_{t+1},y_{t+1}),\ldots,(x_T,y_T) \sim D$ and $\epsilon_{t+1},\ldots \epsilon_T$ Rademacher random variables, and then based on this sample pick $\hat{y}_t=\hat{y}_t(x_{t+1:T},y_{t+1:T},\epsilon_{t+1:T})$ as in \eqref{eq:def_general_fpl2}. Hence,
\begin{align*}
\sup_{y_t} &\left\{ \Es{\hat{y}_t \sim q_t}{\ell(\hat{y}_t,y_t)} + \Relax{T}{\F}{(x_1,y_1),\ldots,(x_t,y_t)}\right\} \\
& = \sup_{y_t} \left\{ \Eunder{(x_{t+1:T},y_{t+1:T})}{\epsilon_{t+1:T}} \ell(\hat{y}_t,y_t) + \Eunder{(x_{t+1:T},y_{t+1:T})}{\epsilon_{t+1:T}} \sup_{f \in \F} \left[ C \sum_{i=t+1}^T \epsilon_i \ell(f(x_i),y_i) - L_t(f) \right] \right\}\\
& \le \Eunder{(x_{t+1:T},y_{t+1:T})}{\epsilon_{t+1:T}} \sup_{y_t} \left\{ \ell(\hat{y}_t,y_t) +  \sup_{f \in \F}\left[ C \sum_{i=t+1}^T \epsilon_i \ell(f(x_i),y_i) - L_t(f) \right]\right\} \ .
\end{align*}
Now, with $\hat{y}_t$ in \eqref{eq:def_general_fpl2},
\begin{align*}
\sup_{y_t} & \left\{ \ell(\hat{y}_t,y_t) + \sup_{f \in \F} \left[ C \sum_{i=t+1}^T \epsilon_i \ell(f(x_i),y_i) - L_t(f)  \right] \right\} \\
& = \inf_{\hat{y} \in [-B,B]} \sup_{y_t} \left\{ \ell(\hat{y},y_t) + \sup_{f \in \F} \left[ C \sum_{i=t+1}^T \epsilon_i \ell(f(x_i),y_i) - L_t(f)  \right] \right\}\\
& = \inf_{\hat{y} \in [-B,B]} \sup_{p_t} \Es{y_t \sim p_t}{ \ell(\hat{y},y_t) + \sup_{f \in \F} \left[ C \sum_{i=t+1}^T \epsilon_i \ell(f(x_i),y_i) - L_t(f)  \right] }
\end{align*}
Now we assume that the loss $\ell(\hat{y},y)$ is convex in  the first argument (and bounded). Note that the term
$$
 \Es{y_t \sim p_t}{ \ell(\hat{y},y_t) + \sup_{f \in \F} \left\{ C \sum_{i=t+1}^T \epsilon_i \ell(f(x_i),y_i) - L_{t}(f)  \right\} }
$$
is linear in $p_t$ and, due to convexity of loss, is convex in $\hat{y}_t$. Hence by the minimax theorem, for this choice of $q_t$, we conclude that
\begin{align*}
\sup_{y_t} &\left\{ \Es{\hat{y}_t \sim q_t}{\ell(\hat{y}_t,y_t)} + \Relax{T}{\F}{(x_1,y_1),\ldots,(x_t,y_t)}\right\} \\
& \le \Eunder{(x_{t+1},y_{t+1}),\ldots, (x_T,y_T)}{\epsilon_{t+1:T}} \inf_{\hat{y}_t \in [-B,B]} \sup_{p_t} \Es{y_t \sim p_t}{\ell(\hat{y}_t,y_t) +  \sup_{f \in \F} \left\{ C \sum_{i=t+1}^T \epsilon_i \ell(f(x_i),y_i) - L_t(f) \right\} }\\
& = \Eunder{(x_{t+1},y_{t+1}),\ldots, (x_T,y_T)}{\epsilon_{t+1:T}} \sup_{p_t }\inf_{\hat{y}_t \in [-B,B]}  \Es{y_t \sim p_t}{\ell(\hat{y}_t,y_t) +  \sup_{f \in \F} \left\{ C \sum_{i=t+1}^T \epsilon_i \ell(f(x_i),y_i) - L_t(f) \right\}} 
\end{align*}
The last step above is due to the minimax theorem as the loss is convex in $\hat{y}_t$, the set $[-B,B]$ is compact, and the term is linear in $p_t$. The above expression is equal to
\begin{align*}
& = \En \sup_{p_t }\En_{y_t \sim p_t}  \sup_{f \in \F} \left\{ C \sum_{i=t+1}^T \epsilon_i \ell(f(x_i),y_i) - \sum_{i=1}^{t-1} \ell(f(x_i),y_i) + \inf_{\hat{y_t} \in [-B,B]}  \Es{y_t \sim p_t}{\ell(\hat{y}_t,y_t)} - \ell(f(x_t),y_t)\right\} \\
& \le \En \sup_{p_t }\En_{y_t \sim p_t}   \sup_{f \in \F} \left\{ C \sum_{i=t+1}^T \epsilon_i \ell(f(x_i),y_i) - L_{t-1}(f) +   \inf_{g \in \F} \Es{y_t \sim p_t}{\ell(g(x_t),y_t)} - \ell(f(x_t),y_t)\right\} \\
& \le \En \sup_{p_t }\En_{y_t \sim p_t}  \sup_{f \in \F} \left\{ C \sum_{i=t+1}^T \epsilon_i \ell(f(x_i),y_i) - L_{t-1}(f)  +   \Es{y_t \sim p_t}{\ell(f(x_t),y_t)} - \ell(f(x_t),y_t)\right\} \\
& \le \Eunder{(x_{t+1},y_{t+1}),\ldots, (x_T,y_T)}{\epsilon_{t+1:T}} \En_{(x_t,y_t) \sim D}\En_{\epsilon_t}   \sup_{f \in \F} \left\{ C \sum_{i=t+1}^T \epsilon_i \ell(f(x_i),y_i) - L_{t-1}(f)  +   C \epsilon_t \ell(f(x_t),y_t) \right\}\\
& = \Relax{T}{\F}{(x_1,y_1),\ldots,(x_{t-1},y_{t-1})}
\end{align*}
The second part of the Lemma is proved analogously.

\end{proof}

\begin{proof}[\textbf{Proof of Lemma \ref{lem:treewalk}}]
Now let $q_t$ be the randomized strategy where we draw $\epsilon_{t+1},\ldots,\epsilon_T$ uniformly at random and pick 
\begin{align}\label{eq:fpltree}
q_t(\epsilon) = \argmin{q \in \Delta(\F)} \sup_{x_t}\left\{ \Es{f_t \in q}{\ell(f_t,x_t)} +  \sup_{f \in \F} \left[ 2 \sum_{i=t+1}^T \epsilon_i \ell(f,\x^t_i(\epsilon)) - \sum_{i=1}^t \ell(f,x_i) \right] \right\}
\end{align}
With the definition of $\x^t$ in \eqref{eq:worst_case_tree_fpl}, and with the notation $L_t(f)=\sum_{i=1}^t \ell(f,x_i) $
\begin{align*}
\sup_{x_t}& \left\{ \Es{f_t \sim q_t}{\ell(f_t,x_t)} + \sup_{\x} \En_{\epsilon} \sup_{f \in \F} \left[ 2 \sum_{i=t+1}^T \epsilon_i \ell(f,\x_i(\epsilon)) - L_{t}(f)\right] \right\} \\
& = \sup_{x_t}\left\{ \Es{\epsilon}{\Es{f_t \sim q_t(\epsilon)}{\ell(f_t,x_t)}} +  \En_{\epsilon}\sup_{f \in \F} \left[ 2 \sum_{i=t+1}^T \epsilon_i \ell(f,\x^t_i(\epsilon)) - L_{t}(f)\right] \right\}\\
& \le \Es{\epsilon}{ \sup_{x_t}\left\{ \Es{f_t \sim q_t(\epsilon)}{\ell(f_t,x_t)} +   \sup_{f \in \F} \left[ 2 \sum_{i=t+1}^T \epsilon_i \ell(f,\x^t_i(\epsilon)) - L_{t}(f)\right] \right\}}\\
& = \Es{\epsilon}{ \inf_{q_t \in \Delta(\F)} \sup_{x_t}\left\{ \Es{f_t \sim q_t}{\ell(f_t,x_t)} +   \sup_{f \in \F} \left[2 \sum_{i=t+1}^T \epsilon_i \ell(f,\x^t_i(\epsilon)) - L_{t}(f)\right] \right\}}
\end{align*}
where the last step is due to the way we pick our predictor $f_t(\epsilon)$ given random draw of $\epsilon$'s in Equation \eqref{eq:fpltree}. We now apply the minimax theorem, yielding the following upper bound on the term above:
\begin{align*}
&\Es{\epsilon}{  \sup_{p_t \in \Delta(\X)} \inf_{f_t \in \F}\left\{ \Es{x_t \sim p_t}{\ell(f_t,x_t)} +   \En_{x_t \sim p_t} \sup_{f \in \F} \left[ 2 \sum_{i=t+1}^T \epsilon_i \ell(f,\x^t_i(\epsilon)) - L_{t}(f)\right] \right\}}
\end{align*}
This expression can be re-written as
\begin{align*}
&\Es{\epsilon}{  \sup_{p_t \in \Delta(\X)} \left\{  \En_{x_t \sim p_t} \sup_{f \in \F} \left[ 2 \sum_{i=t+1}^T \epsilon_i \ell(f,\x^t_i(\epsilon)) - L_{t-1}(f) + \Es{x_t \sim p_t}{\ell(f ,x_t)} - \ell(f,x_t) \right] \right\}}\\
&\le \Es{\epsilon}{  \sup_{p_t \in \Delta(\X)} \left\{  \En_{x_t, x'_t \sim p_t}\En_{\epsilon_t}\sup_{f \in \F} \left[ 2 \sum_{i=t+1}^T \epsilon_i \ell(f,\x^t_i(\epsilon)) - L_{t-1}(f) + \epsilon_t\left(\ell(f ,x_t) - \ell(f,x_t) \right) \right] \right\}}
\end{align*}
By passing to the supremum over $x_t,x'_t$, we get an upper bound
\begin{align*}
&\Es{\epsilon}{  \sup_{x_t , x'_t \in \X} \left\{  \En_{\epsilon_t} \sup_{f \in \F} \left[ 2 \sum_{i=t+1}^T \epsilon_i \ell(f,\x^t_i(\epsilon)) - L_{t-1}(f) + \epsilon_t\left(\ell(f ,x_t) - \ell(f,x_t) \right) \right] \right\}}\\
&\le \Es{\epsilon}{  \sup_{x_t  \in \X} \left\{  \En_{\epsilon_t} \sup_{f \in \F} \left[ 2 \sum_{i=t+1}^T \epsilon_i \ell(f,\x^t_i(\epsilon)) - L_{t-1}(f) + 2 \epsilon_t \ell(f ,x_t) \right] \right\}}\\
&\le \sup_{\tilde{\x}} \Es{\epsilon}{  \sup_{x_t  \in \X} \left\{  \En_{\epsilon_t}\sup_{f \in \F} \left[ 2 \sum_{i=t+1}^T \epsilon_i \ell(f,\tilde{\x}_i(\epsilon)) - L_{t-1}(f) + 2 \epsilon_t \ell(f ,x_t) \right] \right\}} \\
&\le \sup_{\x }\En_{\epsilon}  \sup_{f \in \F}\left[ 2 \sum_{i=t}^T \epsilon_i \ell(f,\x_i(\epsilon)) - L_{t-1}(f) \right]
\end{align*}
\end{proof}

\begin{proof}[\textbf{Proof of Lemma \ref{lem:ontrvar2}}]
We shall start by showing that the relaxation is admissible for the game where we pick prediction $\hat{y}_t$ and the adversary then directly picks the gradient $\partial \ell(\hat{y}_t,y_t)$. To this end note that 
\begin{align*}
\inf_{\hat{y}_t} \sup_{\partial \ell(\hat{y}_t,y_t)}
& \left\{\partial \ell(\hat{y}_t,y_t) \cdot \hat{y}_t  + \Relax{T}{\F}{  \partial \ell(\hat{y}_1,y_1) , \ldots, \partial \ell(\hat{y}_t,y_t)} \right\}\\
& = \inf_{\hat{y}_t} \sup_{\partial \ell(\hat{y}_t,y_t)} \left\{\partial \ell(\hat{y}_t,y_t) \cdot \hat{y}_t  + \Es{\epsilon}{ \sup_{f \in \F} 2 L \sum_{i=t+1}^T \epsilon_i f[t] - \sum_{i=1}^t \partial \ell(\hat{y}_i,y_i) \cdot f[i] } \right\}\\
& \le \inf_{\hat{y}_t} \sup_{r_t \in [-L , L]} \left\{ r_t\cdot \hat{y}_t  + \Es{\epsilon}{ \sup_{f \in \F} 2 L \sum_{i=t+1}^T \epsilon_i f[t] - L_{t-1}(f)  - r_t \cdot f[t]} \right\}
\end{align*}
Let us use the notation $L_{t-1}(f) = \sum_{i=1}^{t-1} \partial \ell(\hat{y}_i,y_i) \cdot f[i]$ for the present proof. The supremum over $r_t\in[-L,L]$ is achieved at the endpoints since the expression is convex in $r_t$. Therefore, the last expression is equal to
\begin{align*}
&\inf_{\hat{y}_t} \sup_{r_t \in \{-L , L\}} \left\{ r_t\cdot \hat{y}_t  + \En_{\epsilon} \sup_{f \in \F}\left[ 2 L \sum_{i=t+1}^T \epsilon_i f[t] - L_{t-1}(f)  - r_t \cdot f[t] \right]\right\}\\
& = \inf_{\hat{y}_t} \sup_{p_t \in \Delta(\{-L , L\})} \Es{r_t \sim p_t}{ r_t\cdot \hat{y}_t  + \En_{\epsilon} \sup_{f \in \F} \left[ 2 L \sum_{i=t+1}^T \epsilon_i f[t] - L_{t-1}(f)  - r_t \cdot f[t]\right] } \\
& =  \sup_{p_t \in \Delta(\{-L , L\})} \inf_{\hat{y}_t} \Es{r_t \sim p_t}{ r_t\cdot \hat{y}_t  + \En_{\epsilon} \sup_{f \in \F} \left[ 2 L \sum_{i=t+1}^T \epsilon_i f[t] - L_{t-1}(f)  - r_t \cdot f[t] \right]} 
\end{align*}
where the last step is due to the minimax theorem. The last quantity is equal to
\begin{align*}
& \sup_{p_t \in \Delta(\{-L , L\})}  \Es{\epsilon}{ \Es{r_t \sim p_t}{ \inf_{\hat{y}_t} \Es{r_t \sim p_t}{r_t} \cdot \hat{y}_t  +  \sup_{f \in \F} \left( 2 L \sum_{i=t+1}^T \epsilon_i f[t] - L_{t-1}(f)  - r_t \cdot f[t] \right) } }\\
& \le  \sup_{p_t \in \Delta(\{-L , L\})} \Es{\epsilon}{ \Es{r_t \sim p_t}{ \sup_{f \in \F} \left( 2 L \sum_{i=t+1}^T \epsilon_i f[t] - L_{t-1}(f)  + (\Es{r_t \sim p_t}{r_t} - r_t) \cdot f[t] \right) }}\\
& \le  \sup_{p_t \in \Delta(\{-L , L\})} \Es{r_t , r'_t \sim p_t}{  \En_{\epsilon} \sup_{f \in \F} \left[ 2 L \sum_{i=t+1}^T \epsilon_i f[t] - L_{t-1}(f)  + (r'_t - r_t) \cdot f[t] \right] }\\
& = \sup_{p_t \in \Delta(\{-L , L\})} \Es{r_t , r'_t \sim p_t}{  \En_{\epsilon} \sup_{f \in \F} \left[ 2 L \sum_{i=t+1}^T \epsilon_i f[t] - L_{t-1}(f)  + \epsilon_t (r'_t - r_t) \cdot f[t] \right] }
\end{align*}
By passing to the worst-case choice of $r_t , r'_t$ (which is achieved at the endpoints because of convexity), we obtain a further upper bound
\begin{align*}
&\sup_{r_t , r'_t \in \{L, -L\}}  \En_{\epsilon} \sup_{f \in \F} \left[ 2 L \sum_{i=t+1}^T \epsilon_i f[t] - L_{t-1}(f)  + \epsilon_t (r'_t - r_t) \cdot f[t] \right]\\
& \le \sup_{r_t \in \{L, -L\}}  \En_{\epsilon} \sup_{f \in \F} \left[2 L \sum_{i=t+1}^T \epsilon_i f[t] - L_{t-1}(f)  + 2 \epsilon_t r_t \cdot f[t] \right]\\
& = \sup_{r_t \in \{L, -L\}}  \En_{\epsilon} \sup_{f \in \F} \left[ 2 L \sum_{i=t}^T \epsilon_i f[t] - L_{t-1}(f)  \right] \\
& = \Relax{T}{\F}{  \partial \ell(\hat{y}_1,y_1) , \ldots, \partial \ell(\hat{y}_{t-1},y_{t-1})}
\end{align*}
Thus we see that the relaxation is admissible. Now the corresponding prediction is given by 
\begin{align*}
\hat{y}_t & = \argmin{\hat{y}} \sup_{r_t \in [-L,L]} \left\{ r_t \hat{y} + \Es{\epsilon}{\sup_{f \in \F} \left\{ 2 L \sum_{i=t+1}^{T} \epsilon_i f[i] - \sum_{i=1}^{t-1} \partial\ell(\hat{y}_i,y_i) f[i] - r_t f[t] \right\} }  \right\}\\
& = \argmin{\hat{y}} \sup_{r_t \in [-L,L]} \left\{ r_t \hat{y} + \Es{\epsilon}{\sup_{f \in \F} \left\{ 2 L \sum_{i=t+1}^{T} \epsilon_i f[i] - \sum_{i=1}^{t-1} \partial\ell(\hat{y}_i,y_i) f[i] - r_t f[t] \right\} }  \right\}\\
& = \argmin{\hat{y}} \sup_{r_t \in \{-L,L\}} \left\{ r_t \hat{y} + \Es{\epsilon}{\sup_{f \in \F} \left\{ 2 L \sum_{i=t+1}^{T} \epsilon_i f[i] - \sum_{i=1}^{t-1} \partial\ell(\hat{y}_i,y_i) f[i] - r_t f[t] \right\}  } \right\}
\end{align*}
The last step holds because of convexity of the term inside the supremum over $r_t$ is convex in $r_t$ and so the suprema is attained at the endpoints of the interval. The $\hat{y}_t$ above is attained when both terms of the supremum are equalized, that is for $\hat{y}_t$ is the prediction that satisfies :
\begin{align*}
\hat{y}_t =   \Es{\epsilon}{\sup_{f \in \F} \left\{ \sum_{i=t+1}^{T} \epsilon_i f[i] - \frac{1}{2L} \sum_{i=1}^{t-1} \partial\ell(\hat{y}_i,y_i) f[i] + \frac{1}{2} f[t] \right\} - \sup_{f \in \F} \left\{  \sum_{i=t+1}^{T} \epsilon_i f[i] -  \frac{1}{2L} \sum_{i=1}^{t-1} \partial\ell(\hat{y}_i,y_i) f[i] -  \frac{1}{2} f[t] \right\}}
\end{align*}
Finally since the relaxation is admissible we can conclude that the regret of the algorithm is bounded as
$$
\Reg_T \le \Rel{T}{\F} = 2\ L\ \Es{\epsilon}{\sup_{f \in \F} \sum_{t=1}^T \epsilon_t f[t]} 
\ .$$
This concludes the proof.
\end{proof}

\begin{proof}[\textbf{Proof of Lemma~\ref{lem:transd-assm2-satisfied}}]

The proof is similar to that of Lemma \ref{lem:ontrvar2}, with a few more twists. We want to establish admissibility of the relaxation given in \eqref{eq:trans1rel} w.r.t. the randomized strategy $q_t$ we provided. To this end note that
\begin{align*}
& \sup_{y_t} \left\{ \Es{\hat{y}_t \sim q_t}{\ell(\hat{y}_t,y_t) } + \Es{\epsilon}{\sup_{f \in \F} \left\{ 2 L \sum_{i=t+1}^{T} \epsilon_i f[i] - L_{t}(f) \right\} } \right\}  \notag \\
& = \sup_{y_t} \left\{ \Es{\epsilon}{\ell(\hat{y}_t(\epsilon),y_t) } + \Es{\epsilon}{\sup_{f \in \F} \left\{ 2 L \sum_{i=t+1}^{T} \epsilon_i f[i] - L_{t}(f) \right\} } \right\}  \notag \\
& \le \Es{\epsilon}{\sup_{y_t} \left\{ \ell(\hat{y}_t(\epsilon),y_t)  + \sup_{f \in \F} \left\{ 2 L \sum_{i=t+1}^{T} \epsilon_i f[i] - L_{t}(f) \right\}  \right\} }  
\end{align*}
by Jensen's inequality, with the usual notation $L_{t}(f)= \sum_{i=1}^{t} \ell(f[i],y_i)$. Further, by convexity of the loss, we may pass to the upper bound
\begin{align}
& \Es{\epsilon}{\sup_{y_t} \left\{ \partial \ell(\hat{y}_t(\epsilon),y_t) \hat{y}_t(\epsilon)  + \sup_{f \in \F} \left\{ 2 L \sum_{i=t+1}^{T} \epsilon_i f[i] - L_{t-1}(f) - \partial \ell(\hat{y}_t(\epsilon),y_t) f[t] \right\}  \right\} } \notag  \\
& \le \Es{\epsilon}{\sup_{y_t} \left\{ \Es{r_t }{r_t \cdot \hat{y}_t(\epsilon)}  + \sup_{f \in \F} \left\{ 2 L \sum_{i=t+1}^{T} \epsilon_i f[i] - L_{t-1}(f) - \Es{r_t}{ r_t \cdot f[t]} \right\}  \right\} } \notag  
\end{align}
where $r_t$ is a $\{\pm L\}$-valued random variable with the mean $\partial \ell(\hat{y}_t(\epsilon),y_t)$. With the help of Jensen's inequality, and passing to the worst-case $r_t$ (observe that this is legal for any given $\epsilon$), we have an upper bound
\begin{align}
& \Es{\epsilon}{\sup_{y_t} \left\{ \Es{r_t \sim \partial \ell(\hat{y}_t(\epsilon),y_t)}{r_t \cdot \hat{y}_t(\epsilon)}  + \Es{r_t \sim \partial \ell(\hat{y}_t(\epsilon),y_t)}{  \sup_{f \in \F} \left\{ 2 L \sum_{i=t+1}^{T} \epsilon_i f[i] - L_{t-1}(f) - r_t \cdot f[t] \right\}}  \right\} }  \notag \\
& \le \Es{\epsilon}{\sup_{r_t \in \{\pm L\}} \left\{ r_t \cdot \hat{y}_t(\epsilon)  +   \sup_{f \in \F} \left\{ 2 L \sum_{i=t+1}^{T} \epsilon_i f[i] - L_{t-1}(f) - r_t \cdot f[t] \right\}  \right\} }   \label{eq:int}
\end{align}

Now the strategy we defined is
$$
\hat{y}_t(\epsilon) = \argmin{\hat{y}_t} \sup_{r_t \in \{\pm L\}} \left\{ r_t \cdot \hat{y}_t(\epsilon)  +   \sup_{f \in \F} \left\{ 2 L \sum_{i=t+1}^{T} \epsilon_i f[i] - \sum_{i=1}^{t-1} \ell(f[i],y_i) - r_t \cdot f[t] \right\}  \right\}
$$
which can be re-written as 
\begin{align*}
\hat{y}_t(\epsilon) = \left( \sup_{f \in \F} \left\{ \sum_{i=t+1}^{T} \epsilon_i f[i] - \frac{1}{2 L}L_{t-1}(f) + \frac{1}{2} f[t] \right\} - \sup_{f \in \F} \left\{  \sum_{i=t+1}^{T} \epsilon_i f[i] - \frac{1}{2 L} L_{t-1}(f) - \frac{1}{2} f[t] \right\} \right)
\end{align*}

By this choice of $\hat{y}_t(\epsilon)$, plugging back in Equation \eqref{eq:int} we see that 
\begin{align*}
 \sup_{y_t} & \left\{ \Es{\hat{y}_t \sim q_t}{\ell(\hat{y}_t,y_t) } + \Es{\epsilon}{\sup_{f \in \F} \left\{ 2 L \sum_{i=t+1}^{T} \epsilon_i f[i] - L_{t}(f) \right\} } \right\}   \\
& \le \Es{\epsilon}{ \sup_{r_t \in \{\pm L\}} \left\{ r_t \cdot \hat{y}_t(\epsilon)  +   \sup_{f \in \F} \left\{ 2 L \sum_{i=t+1}^{T} \epsilon_i f[i] - L_{t-1}(f) - r_t \cdot f[t] \right\}  \right\} }   \\
& = \Es{\epsilon}{ \inf_{\hat{y}_t} \sup_{r_t \in \{\pm L\}} \left\{ r_t \cdot \hat{y}_t  +   \sup_{f \in \F} \left\{ 2 L \sum_{i=t+1}^{T} \epsilon_i f[i] - L_{t-1}(f) - r_t \cdot f[t] \right\}  \right\} }   \\
& = \Es{\epsilon}{ \inf_{\hat{y}_t} \sup_{p_t \in \Delta(\{\pm L\})} \En_{r_t \sim p_t} \left\{ r_t \cdot \hat{y}_t  +   \sup_{f \in \F} \left\{ 2 L \sum_{i=t+1}^{T} \epsilon_i f[i] - L_{t-1}(f) - r_t \cdot f[t] \right\}  \right\} }   
\end{align*}
The expression inside the supremum is linear in $p_t$, as it is an expectation. Also note that the term is convex in $\hat{y}_t$, and the domain $\hat{y}_t \in [ - \sup_{f \in \F} |f[t]| , \sup_{f \in \F} |f[t]|] $ is a bounded interval (hence, compact). We conclude that we can use the minimax theorem, yielding 
\begin{align*}
& \Es{\epsilon}{  \sup_{p_t \in \Delta(\{\pm L\})} \inf_{\hat{y}_t} \Es{r_t \sim p_t} { r_t \cdot \hat{y}_t  +   \sup_{f \in \F} \left\{ 2 L \sum_{i=t+1}^{T} \epsilon_i f[i] - L_{t-1}(f) - r_t \cdot f[t] \right\} } }   \\
& = \Es{\epsilon}{  \sup_{p_t \in \Delta(\{\pm L\})} \left\{ \inf_{\hat{y}_t} \Es{r_t \sim p_t}{ r_t \cdot \hat{y}_t}  +   \Es{r_t \sim p_t}{\sup_{f \in \F} \left\{ 2 L \sum_{i=t+1}^{T} \epsilon_i f[i] - L_{t-1}(f) - r_t \cdot f[t] \right\} }   \right\}}\\
& = \Es{\epsilon}{  \sup_{p_t \in \Delta(\{\pm L\})} \left\{  \Es{r_t \sim p_t}{ \sup_{f \in \F} \left\{ \inf_{\hat{y}_t} \Es{r_t \sim p_t}{ r_t \cdot \hat{y}_t}  +   2 L \sum_{i=t+1}^{T} \epsilon_i f[i] - L_{t-1}(f) - r_t \cdot f[t] \right\} }   \right\}}\\
& \le \Es{\epsilon}{  \sup_{p_t \in \Delta(\{\pm L\})} \left\{  \Es{r_t \sim p_t}{ \sup_{f \in \F} \left\{ \Es{r_t \sim p_t}{ r_t \cdot f[t]}  +   2 L \sum_{i=t+1}^{T} \epsilon_i f[i] - L_{t-1}(f) - r_t \cdot f[t] \right\} }   \right\}}
\end{align*}
In the last step, we replaced the infimum over $\hat{y}_t$ with $f[t]$, only increasing the quantity. Introducing an i.i.d. copy $r'_t$ of $r_t$,
\begin{align*}
& = \Es{\epsilon}{  \sup_{p_t \in \Delta(\{\pm L\})} \left\{   \Es{r_t \sim p_t}{\sup_{f \in \F} \left\{ 2 L \sum_{i=t+1}^{T} \epsilon_i f[i] - L_{t-1}(f) + \left(\Es{r_t \sim p_t}{r_t}- r_t\right) \cdot f[t] \right\} }   \right\}}\\
& \le \Es{\epsilon}{  \sup_{p_t \in \Delta(\{\pm L\})} \left\{   \Es{r_t, r'_t \sim p_t}{\sup_{f \in \F} \left\{ 2 L \sum_{i=t+1}^{T} \epsilon_i f[i] - L_{t-1}(f) + \left(r'_t- r_t\right) \cdot f[t] \right\} }   \right\}}
\end{align*}
Introducing the random sign $\epsilon_t$ and passing to the supremum over $r_t,r_t'$, yields the upper bound
\begin{align*}
&\Es{\epsilon}{  \sup_{p_t \in \Delta(\{\pm L\})} \left\{   \En_{r_t, r'_t \sim p_t}\Es{\epsilon_t}{\sup_{f \in \F} \left\{ 2 L \sum_{i=t+1}^{T} \epsilon_i f[i] - L_{t-1}(f) + \left(r'_t- r_t\right) \cdot f[t] \right\} }   \right\}}\\
& \le \Es{\epsilon}{  \sup_{r_t, r'_t \in \{\pm L\}} \left\{   \Es{\epsilon_t}{\sup_{f \in \F} \left\{ 2 L \sum_{i=t+1}^{T} \epsilon_i f[i] - L_{t-1}(f) + \epsilon_t\left(r'_t- r_t\right) \cdot f[t] \right\} }   \right\}} \\
& \le \Es{\epsilon}{  \sup_{r_t, r'_t \in \{\pm L\}} \left\{   \Es{\epsilon_t}{\sup_{f \in \F} \left\{  L \sum_{i=t+1}^{T} \epsilon_i f[i] - \frac{1}{2}L_{t-1}(f) + \epsilon_t r'_t \cdot f[t] \right\} }   \right\}} \\
& ~~~~~~~~~~~~~ + \Es{\epsilon}{  \sup_{r_t, r'_t \in \{\pm L\}} \left\{   \Es{\epsilon_t}{\sup_{f \in \F} \left\{  L \sum_{i=t+1}^{T} \epsilon_i f[i] - \frac{1}{2}L_{t-1}(f) - \epsilon_t r_t \cdot f[t] \right\} }   \right\}}
\end{align*}
In the above we split the term in the supremum as the sum of two terms one involving $r_t$ and other $r'_t$ (other terms are equally split by dividing by $2$), yielding
\begin{align*}
&\Es{\epsilon}{  \sup_{r_t \in \{\pm L\}} \left\{   \Es{\epsilon_t}{\sup_{f \in \F} \left\{ 2 L \sum_{i=t+1}^{T} \epsilon_i f[i] - L_{t-1}(f) + 2 \ \epsilon_t\ r_t \cdot f[t] \right\} }   \right\}}
\end{align*}
The above step used the fact that the first term only involved $r'_t$ and second only $r_t$ and further $\epsilon_t$ and $- \epsilon_t$ have the same distribution. Now finally noting that irrespective of whether $r_t$ in the above supremum is $L$ or $-L$, since it is multiplied by $\epsilon_t$ we obtain an upper bound
\begin{align*}
&\Es{\epsilon}{\sup_{f \in \F} \left\{ 2 L \sum_{i=t}^{T} \epsilon_i f[i] - L_{t-1}(f)  \right\} }
\end{align*}
We conclude that the relaxation 
$$
\Relax{T}{\F}{y_1,\ldots,y_t} = \Es{\epsilon}{\sup_{f \in \F} \left\{ 2 L \sum_{i=t+1}^{T} \epsilon_i f[i] - L_{t}(f)  \right\} }
$$
is admissible and further the randomized strategy where on each round we first draw $\epsilon$'s and then set
\begin{align*}
\hat{y}_t(\epsilon) & = \left( \sup_{f \in \F} \left\{ \sum_{i=t+1}^{T} \epsilon_i f[i] - \frac{1}{2 L}L_{t-1}(f) + \frac{1}{2} f[t] \right\} - \sup_{f \in \F} \left\{  \sum_{i=t+1}^{T} \epsilon_i f[i] - \frac{1}{2 L} L_{t-1}(f) - \frac{1}{2} f[t] \right\} \right)\\
& = \left(\inf_{f \in \F} \left\{  - \sum_{i=t+1}^{T} \epsilon_i f[i] + \frac{1}{2 L} L_{t-1}(f) + \frac{1}{2} f[t] \right\} - \inf_{f \in \F} \left\{ - \sum_{i=t+1}^{T} \epsilon_i f[i] + \frac{1}{2 L}L_{t-1}(f) - \frac{1}{2} f[t] \right\} \right) 
\end{align*}
is an admissible strategy. Hence, the expected regret under the strategy is bounded as
$$
\E{\Reg_T} \le \Rel{T}{\F} = 2 L\  \Es{\epsilon}{\sup_{f \in \F}  \sum_{i=1}^{T} \epsilon_i f[i]  }
$$
which concludes the proof.
\end{proof}

\begin{proof}[\textbf{Proof of Lemma~\ref{lem:constrained}}]
	The proof is almost identical to the proof of admissibility for the Mirror Descent relaxation, so let us only point out the differences. Let $\tilde{x}_{t-1}=\sum_{i=1}^t x_i$ and $\mu_{t-1}=\frac{1}{t-1}\tilde{x}_{t-1}$. Using the fact that $x_t$ is $\sigma_t$-close to $\mu_{t-1}$, we expand
	\begin{align*}
		\left( \left\|\tilde{x}_{t}\right\|^2 +  C\sum_{s=t+1}^T\sigma_s^2 \right)^{1/2} &\leq \left( \left\|\tilde{x}_{t-1} \left(\frac{t}{t-1}\right) \right\|^2 + \ip{\nabla\left\|\left( \frac{t}{t-1}\right)\tilde{x}_{t-1}\right\|^2}{x_t-\mu_{t-1}}  + C\sum_{s=t+1}^T\sigma_s^2 \right)^{1/2}\\
	\end{align*}
	As before, pick $x_t = \beta \tilde{x}_{t-1} + \gamma y$ for some $y \in \mrm{Kernel}(\nabla\|\tilde{x}_{t-1}\|^2)$. The above expression under the square root then becomes
	\begin{align*}
		 \left\|\tilde{x}_{t-1}  \right\|^2 +\underbrace{\left(\frac{1}{(t-1)^2}+\frac{2}{t-1} + \left( \frac{t}{t-1}\right)^2\left(\beta-\frac{1}{t-1}\right)\right)}_{\beta'}\left\|\tilde{x}_{t-1}  \right\|^2   + C\sum_{s=t+1}^T\sigma_s^2 ,
	\end{align*}
	and the only difference from the expression in \eqref{eq:grad_desc_expansion} is that we have a $\beta'$ instead of $\beta$ under the square root. Taking the derivatives, we see that
	$$ \alpha = \frac{\left(1+\frac{1}{t-1}\right)^2}{2\sqrt{\|\tilde{x}_{t-1}\|^2 + C\sum_{s=t}^T\sigma_s^2}}$$
	forces $\beta'=0$ and we conclude admissibility.

	\paragraph{Arriving at the Relaxation} 
	We upper bound the sequential Rademacher complexity as
	\begin{align}
	&\frac{2}{\alpha} \sup_{(\x,\x') \in \mc{T}} \Es{\epsilon}{\sup_{f \in \F}\inner{  f, \alpha \sum_{s=t+1}^T \epsilon_s\left(\x_s(\epsilon) - \frac{1}{s-t} \sum_{\tau=t+1}^{s-1} \chi_\tau(\epsilon_\tau) \right) - \sum_{r=1}^t x_r }} \notag \\
	& \le \frac{2 R^2}{\alpha} + \frac{\alpha}{\lambda} \sup_{(\x,\x')}  \En_\epsilon\left\|  \sum_{s=t+1}^T\epsilon_s \left(\x_s(\epsilon) - \frac{1}{s-t} \sum_{\tau=t+1}^{s-1} \chi_\tau(\epsilon_\tau) \right) - \sum_{r=1}^t x_r \right\|^2   \\
	&\leq \frac{2\sqrt{2}R}{\sqrt{\lambda}} \sqrt{\sup_{(\x,\x')}  \Es{\epsilon}{\left\|\sum_{s=t+1}^T \epsilon_s\left( \x_s(\epsilon) - \frac{1}{s-t} \sum_{\tau=t+1}^{s-1} \chi_\tau(\epsilon_\tau)\right) - \sum_{r=1}^t x_r \right\|^2}}\\
	&\leq \frac{2\sqrt{2}R}{\sqrt{\lambda}} \sqrt{ \left\|\sum_{r=1}^t x_r \right\|^2 + \sup_{(\x,\x')} C\sum_{s=t+1}^T\left\| \x_s(\epsilon) - \frac{1}{s-t} \sum_{\tau=t+1}^{s-1} \chi_\tau(\epsilon_\tau) \right\|^2 }
	\end{align}
	Since $(\x,\x') \in \mathcal{T}$ are pairs of tree such that for any $\epsilon \in \{\pm 1\}^T$ and any $t \in [T]$.
	$$
	C(x_1,\ldots,x_t,\chi_1(\epsilon_1), \ldots,\chi_{t-1}(\epsilon_{t-1}), \x_{t}(\epsilon)) = 1
	$$
	we can conclude that for any $\epsilon \in \{\pm 1\}^T$ and any $t \in [T]$,
	$$
	\left\|\x_t(\epsilon) - \frac{1}{t-1} \sum_{\tau=1}^{t-1} \chi_{\tau}(\epsilon_\tau) \right\| \le \sigma_t
	$$
\end{proof}

\begin{proof}[\textbf{Proof of Lemma~\ref{lem:get_mirror}}]
Then Sequential Rademacher complexity can be upper bounded as
\begin{align*}
\sup_{\x} \Es{\epsilon}{\left\|\sum_{i=1}^t x_t + \sum_{i=1}^{T-t} \epsilon_i \x_i(\epsilon) \right\|} & \le \sup_{\x} \left( \Es{\epsilon}{\left\|\sum_{i=1}^t x_t + \sum_{i=1}^{T-t} \epsilon_i \x_i(\epsilon) \right\|^p} \right)^{1/p}\\
& \le \sup_{\x} \left( \Es{\epsilon}{\left\|\sum_{i=1}^t x_t + \sum_{i=1}^{T-t} \epsilon_i \x_i(\epsilon) \right\|^p - C \sum_{i=1}^{T-t} \Es{\epsilon}{\left\|\x_i(\epsilon)\right\|^p}} + C (T - t) \right)^{1/p} \\
& =  \left( \Psi^*\left(\sum_{i=1}^t x_i \right)  +  C (T-t) \right)^{1/p} \\
&\leq \left( \Psi^*\left(\sum_{i=1}^{t-1} x_i \right)  +\ip{\nabla \Psi^*\left(\sum_{i=1}^{t-1} x_i\right)}{x_t} + C (T-t+1) \right)^{1/p}
\end{align*}
and admissibility is verified in a similar way to the $2$-smooth case in the Section \ref{sec:rel}. Here we instead use  $p$-smoothness which follows from result in \cite{SreSriTew11}. The form of update specified by the relaxation in this case follows exactly the proof of Proposition~\ref{prop:mirror_relax}, yielding
$$
f_t = - \frac{\nabla \Psi^*(\sum_{j=1}^{t-1} x_i)}{p \left( \Psi^*(\sum_{j=1}^{t-1} x_i) + C(T-t+1)\right)^{1/p}}
$$
\end{proof}

\begin{lemma}
	\label{lem:reg_split}
	The regret upper bound
	\begin{align}
	 \sum_{t=1}^T \ell(f_t,x_t) - \inf_{f \in \F} \sum_{t=1}^T \ell(f,x_t) \le  \sum_{t=1}^T \ell(f_t,x_t) - \sum_{i=1}^m \inf_{f \in \F^{k_i}\left(x_1,\ldots,x_{\tilde{k}_{i-1}}\right)} \sum_{t=\tilde{k}_{i-1}+1}^{\tilde{k}_{i}} \ell(f,x_t) \ . 
	\end{align}
	is valid.
\end{lemma}
\begin{proof}[\textbf{Proof of Lemma~\ref{lem:reg_split}}]
	To prove this inequality, it is enough to show that it holds for subdividing $T$ into two blocks $k_1$ and $k_2$. Observe, that the comparator term becomes only smaller if we pass to two instead of one infima, but we must check that no function $f$ that minimizes the loss over the first block is removed from being a potential minimizer over the second block. This is exactly the definition of $\F^{k_2}(x_1,\ldots,x_{k_1})$.
\end{proof}

\begin{lemma}
	\label{lem:strconv}
	The relaxation $$\Relax{T}{\F}{ x_{1},\ldots,x_{t}} =  - \inf_{f \in \F} \sum_{i=1}^t x_i(f) + (T-t)  \inf_{f \in\F} \sup_{f' \in \F} \|f - f'\|$$
	 is admissible.
\end{lemma}
\begin{proof}[\textbf{Proof of Lemma~\ref{lem:strconv}}]
First, 
	$$\Relax{T}{\F}{ x_{1},\ldots,x_T} =  - \inf_{f \in \F} \sum_{t=1}^T x_t(f). $$
	As for admissibility, 
	\begin{align*}
		 &\inf_{f_t \in \F} \sup_{x}  \left\{ x(f_t) +  \Relax{T}{\F}{x_1,\ldots,x_{t-1},x}\right\}  \\
		 & =  \inf_{f_t \in \F} \sup_{x}  \left\{ x(f_t) - \inf_{f \in \F} \left\{ \sum_{i=1}^{t-1} x_i(f) + x(f) \right\}\right\} + (T-t)  \inf_{f \in\F} \sup_{f' \in \F} \norm{f - f'}\\
		 & \le \inf_{f_t \in \F} \sup_{x}  \left\{ x(f_t) - \inf_{f \in \F}  \sum_{i=1}^{t-1} x_i(f) - \inf_{f \in \F} x(f)\right\} + (T-t)  \inf_{f \in\F} \sup_{f' \in \F} \norm{f - f'}\\
		 & \le \inf_{f_t \in \F} \sup_{x}  \left\{ \sup_{f \in \F} \ip{\nabla x}{ f_t - f} - \inf_{f \in \F}  \sum_{i=1}^{t-1} x_i(f)   \right\} + (T-t)  \inf_{f \in\F} \sup_{f' \in \F} \norm{f - f'}\\
		 & \le \inf_{f_t \in \F}  \left\{ \sup_{f \in \F} \norm{ f_t - f} - \inf_{f \in \F}  \sum_{i=1}^{t-1} x_i(f)   \right\} + (T-t)  \inf_{f \in\F} \sup_{f' \in \F} \norm{f - f'}\\
		 & =  \Relax{T}{\F}{x_1,\ldots,x_{t-1}}
	\end{align*}
\end{proof}

\section*{Acknowledgements}
We gratefully acknowledge the support of NSF under grants CAREER DMS-0954737 and CCF-1116928.

\bibliographystyle{plain}
\bibliography{paper}

\begin{thebibliography}{10}

\bibitem{AbeAgrBarRak09colt}
J.~Abernethy, A.~Agarwal, P.~L. Bartlett, and A.~Rakhlin.
\newblock A stochastic view of optimal regret through minimax duality.
\newblock In {\em COLT '09}, 2009.

\bibitem{AbeBarRakTew08colt}
J.~Abernethy, P.~L. Bartlett, A.~Rakhlin, and A.~Tewari.
\newblock Optimal strategies and minimax lower bounds for online convex games.
\newblock In {\em Proceedings of The Twenty First Annual Conference on Learning
  Theory}, 2008.

\bibitem{AbeWarYel08}
J.~Abernethy, M.K. Warmuth, and J.~Yellin.
\newblock Optimal strategies from random walks.
\newblock In {\em COLT}, pages 437--445, 2008.

\bibitem{bartlett2005local}
P.L. Bartlett, O.~Bousquet, and S.~Mendelson.
\newblock Local rademacher complexities.
\newblock {\em The Annals of Statistics}, 33(4):1497--1537, 2005.

\bibitem{BarHazRak07}
P.L. Bartlett, E.~Hazan, and A.~Rakhlin.
\newblock Adaptive online gradient descent.
\newblock {\em Advances in Neural Information Processing Systems}, 20:65--72,
  2007.

\bibitem{BenPalSSS09agnostic}
S.~Ben-David, D.~P{\'a}l, and S.~Shalev-Shwartz.
\newblock Agnostic online learning.
\newblock In {\em COLT}, 2009.

\bibitem{PLG}
N.~Cesa-Bianchi and G.~Lugosi.
\newblock {\em Prediction, Learning, and Games}.
\newblock Cambridge University Press, 2006.

\bibitem{CBShamir11efficient}
N.~Cesa-Bianchi and O.~Shamir.
\newblock Efficient online learning via randomized rounding.
\newblock In {\em NIPS}, 2011.

\bibitem{chaudhuri2009parameter}
K.~Chaudhuri, Y.~Freund, and D.~Hsu.
\newblock A parameter-free hedging algorithm.
\newblock {\em Arxiv preprint arXiv:0903.2851}, 2009.

\bibitem{KalVem05}
A.~Kalai and S.~Vempala.
\newblock Efficient algorithms for online decision problems.
\newblock {\em J. Comput. Syst. Sci.}, 71(3):291--307, 2005.

\bibitem{koltchinskii2002empirical}
V.~Koltchinskii and D.~Panchenko.
\newblock Empirical margin distributions and bounding the generalization error
  of combined classifiers.
\newblock {\em The Annals of Statistics}, 30(1):1--50, 2002.

\bibitem{Lit88}
N.~Littlestone.
\newblock Learning quickly when irrelevant attributes abound: A new
  linear-threshold algorithm.
\newblock {\em Machine Learning}, 2(4):285--318, 04 1988.

\bibitem{HarRak10}
H.~Narayanan and A.~Rakhlin.
\newblock Random walk approach to regret minimization.
\newblock In {\em NIPS}, 2010.

\bibitem{RakSriTew10nips}
A.~Rakhlin, K.~Sridharan, and A.~Tewari.
\newblock Online learning: Random averages, combinatorial parameters, and
  learnability.
\newblock In {\em NIPS}, 2010.
\newblock Available at http://arxiv.org/abs/1006.1138.

\bibitem{RakSriTew11colt}
A.~Rakhlin, K.~Sridharan, and A.~Tewari.
\newblock Online learning: Beyond regret.
\newblock In {\em COLT}, 2011.
\newblock Available at http://arxiv.org/abs/1011.3168.

\bibitem{RakSriTew11nips}
A.~Rakhlin, K.~Sridharan, and A.~Tewari.
\newblock Online learning: Stochastic, constrained, and smoothed adversaries.
\newblock In {\em NIPS}, 2011.
\newblock Available at http://arxiv.org/abs/1104.5070.

\bibitem{ShalSahm09}
O.~Shamir and S.~Shalev-Shwartz.
\newblock Collaborative filtering with the trace norm: Learning, bounding, and
  transducing.
\newblock In {\em COLT}, 2011.

\bibitem{SreSriTew11}
Nati Srebro, Karthik Sridharan, and Ambuj Tewari.
\newblock On the universality of online mirror descent.
\newblock In J.~Shawe-Taylor, R.S. Zemel, P.~Bartlett, F.C.N. Pereira, and K.Q.
  Weinberger, editors, {\em Advances in Neural Information Processing Systems
  24}, pages 2645--2653. 2011.

\bibitem{sridharan2010convex}
K.~Sridharan and A.~Tewari.
\newblock Convex games in banach spaces.
\newblock In {\em Proceedings of the 23nd Annual Conference on Learning
  Theory}, 2010.

\bibitem{adaptivehedge11}
T.~{van Erven}, P.~{Gr{\"u}nwald}, W.~M. {Koolen}, and S.~{de Rooij}.
\newblock {Adaptive Hedge}.
\newblock {\em ArXiv e-prints}, October 2011.

\end{thebibliography}

\end{document}